\documentclass[11pt]{article}
\pdfoutput=1
\usepackage{enumerate}
\usepackage{smile}
\usepackage[colorlinks,
            linkcolor=red,
            anchorcolor=blue,
            citecolor=blue
            ]{hyperref}
\usepackage{mathrsfs}
\usepackage{color}
\usepackage{colortbl}
\usepackage{fullpage}
\usepackage[protrusion=true,expansion=true]{microtype}
\usepackage{pbox}
\usepackage{setspace}
\usepackage{tabularx}
\usepackage{float}
\usepackage{enumitem}
\usepackage{graphicx}
\usepackage{caption}
\usepackage{subcaption}
\usepackage{pgfplots}
\usepackage{comment}
\usepackage[left=0.9in, right=0.9in, top=1in, bottom=1in]{geometry} 

\usepackage{tikz}
\usetikzlibrary{calc,arrows.meta,decorations.pathreplacing,positioning}

\usepackage{multirow}
\usepackage{array}

\usetikzlibrary{arrows,shapes,snakes,automata,backgrounds,petri}
\usepackage{booktabs} 

\allowdisplaybreaks
\newtheorem*{unnumberedlemma}{Lemma}

\usepackage[utf8]{inputenc} 
\usepackage[T1]{fontenc}    
\usepackage{url}            
\usepackage{booktabs}       
\usepackage{amsfonts}       
\usepackage{nicefrac}       
\usepackage{xcolor}         
\usepackage[capitalize,noabbrev]{cleveref}
\usepackage{bbm}
\usepackage{todonotes}
\crefname{equation}{Eq.}{Eqs.} 
\Crefname{equation}{Eq.}{Eqs.} 
\renewcommand{\textcolor}[2]{#2}

\makeatletter
\newcommand*{\rom}[1]{\expandafter\@slowromancap\romannumeral #1@}
\makeatother
\title{\huge Graph Feedback Bandits on Similar Arms:  With and Without Graph Structures}

\author
{
    Han Qi\thanks{Xi'an Jiaotong University.  Email: \{{\tt qihan19@stu.xjtu.edu.cn,co.fly@stu.xjtu.edu.cn,zhuli@xjtu.edu.cn}\}}
    \thanks{Shanghai Artificial Intelligence Laboratory. Email: {\tt zhangqiaosheng@pjlab.org.cn}}
    \and
    Fei Guo$^{*}$
	\and
	Li Zhu$^{*}$
    \and
    Qiaosheng Zhang$^\dagger$
}
\date{}
\begin{document}

\maketitle

\begin{abstract}
In this paper, we study the stochastic multi-armed bandit problem with graph feedback. Motivated by applications in clinical trials and recommendation systems, we assume that two arms are connected if and only if they are  similar (i.e., their means are close).  
We establish a regret lower bound for this problem under the novel feedback structure and introduce two upper confidence bound (UCB)-based algorithms: Double-UCB and Conservative-UCB. The regret upper bounds of both algorithms match the  lower bound up to a logarithmic factor. 
Leveraging the similarity structure, we further explore a scenario  in which the number of arms increases over time (referred to as the \emph{ballooning setting}). Practical applications of this scenario include Q\&A platforms (e.g., Reddit, Stack Overflow, Quora) and product reviews on platforms like  Amazon and Flipkart, where answers (or reviews) continually appear, and the goal is to display the best ones at the top. 
We extend our two UCB-based algorithms to the ballooning setting. Under mild assumptions, we derive regret upper bounds for both algorithms and discuss their sublinearity.  Furthermore, we propose new variants of the corresponding algorithms that do not rely on prior knowledge of the graph  structural and provide  regret upper bounds. 
Finally, we conduct experiments to validate the theoretical results.  
\end{abstract}

\section{Introduction}
The multi-armed bandit is a classical reinforcement learning problem. At each time step, the learner selects an arm, receiving a reward drawn from an unknown probability distribution. The learner's goal is to maximize the cumulative reward over a period of time steps. This problem has attracted significant attention from the online learning community because of its effective balance between exploration (trying out as many arms as possible) and exploitation (utilizing the arm with the best current performance). A number of applications of multi-armed bandit can be found in online sequential decision problems, such as online recommendation systems \citep{li2011unbiased}, online advertisement campaign \citep{schwartz2017customer} and clinical trials \citep{villar2015multi,aziz2021multi}. 

In the standard multi-armed bandit problem, the learner can only observe the reward of the chosen arm. Meanwhile, existing research \citep{mannor2011bandits,caron2012leveraging,hu2020problem,lykouris2020feedback} has also considered   the bandit problem with \emph{side observations}, wherein the learner can observe information about arms other than the selected one. This observation structure can be encoded as a graph, where each node represents an arm. Node $i$ is linked to node $j$ if selecting $i$ provides information about the reward of $j$. 

In this paper, we study a new feedback model: if two arms are $\epsilon$-similar, i.e., the absolute value of the difference between the means of the two arms does not exceed $\epsilon$, an edge  forms between them. This means that after observing the reward of one arm, the decision-maker simultaneously knows the rewards of arms similar to it. If $\epsilon=0$, this feedback model reduces to the standard multi-armed bandit problem. If $\epsilon$ is greater than the maximum difference between the means, this feedback model becomes equivalent to the full information bandit problem.

As a motivating example, consider the recommendation problem on Spotify or Apple Music. After a recommender suggests a song to a user, it can observe not only whether the user liked or saved the song (reward), but also infer that the user is likely to like or save another song that is very similar. Similarity may be based on factors such as the artist, songwriter, genre, and more.
As another motivating example, consider the problem of medicine clinical trials. Each arm represents a different medical treatment plan, and these plans may have some similarities such as dosage, formulation, etc.  When a researcher selects a plan,  they not only observe the reward of that treatment plan but also learn the effects of others similar to the selected one. The treatment effect (reward)  can be either some summary information or a relative effect, such as positive or negative.
Similar examples also appear in chemistry molecular simulations \citep{perez2020adaptivebandit}.

Specifically, this paper considers two bandit models: (i) the standard  graph feedback bandit model and (ii) the bandit problem with an increasing number of arms. The latter is a more challenging setting than the standard one. Relevant applications for this scenario encompass Q\&A platforms such as Reddit, Stack Overflow, and Quora, as well as product reviews on platforms like Amazon and Flipkart. The continuous addition of answers or product reviews on these platforms means that the number of arms increases over time. The goal is to display the best answers or product reviews at the top.
This problem has been previously studied and referred to as ``ballooning multi-armed bandit" by \citet{ghalme2021ballooning}. However, they require that  the best arm is more likely to arrive in the early rounds, while we do not make such assumption.
Our contributions are as follows: 
\begin{enumerate}
     
\item We propose a new feedback model, where an edge is formed if the means of two arms are less than a constant $\epsilon$. We  first analyze the underlying graph $G$ of this feedback model and establish that  the \emph{domination  number} $\gamma(G)$ is equal to the \emph{independent domination number} $i(G)$, while the \emph{independence number} $\alpha(G)$ is no greater than twice the domination number $\gamma(G)$, i.e., $\gamma(G)=i(G) \geq \alpha(G)/2$. This result helps us obtain regret bounds  based on  the domination number without explicitly using the feedback graph to guide exploration. This cannot be achieved  in the standard graph feedback bandit problem \citep{lykouris2020feedback}.
\textcolor{blue}{We also establish a lower bound of order $\Omega(\sqrt{T})$ that is independent of the graph structure, which distinguishes our result from previous works \citep{mannor2011bandits, cohen2016online}  where the lower bound depends on the independence number. }

\item Then, we introduce two algorithms tailored to this specific feedback structure: Double-UCB, which utilizes two UCB algorithms sequentially, and Conservative-UCB, which adopts a more conservative strategy. \textcolor{blue}{ The  regret upper bounds of both algorithms are of order $O(\sqrt{T}\log T)$, matching the lower bound up to a logarithmic factor. }
Additionally, we analyze the regret bounds of  UCB-N \citep{lykouris2020feedback}  applied to our proposed setting  and prove that  its regret upper bound is of the same order as that of Double-UCB and Conservative-UCB.

\item We extend these two algorithms to the ballooning setting  where the number of arms increases over time, and refer to the extended versions as Double-UCB-BL and Conservative-UCB-BL. Our algorithms do not require the best arm to arrive early (as in  \citep{ghalme2021ballooning}) but instead assumes that the means of each arriving arm are independently sampled from some distribution, \textcolor{blue}{or that the number of changes in the optimal arm is $o(\sqrt{T})$}, both of which are more realistic. 
We provide  regret upper bounds for both algorithms, along with a simple regret lower bound for Double-UCB-BL. The lower bound of Double-UCB-BL shows that it achieves sublinear regret only when the means are drawn from a normal-like distribution, while it suffers linear regret when the means are drawn from a uniform distribution. In contrast, Conservative-UCB-BL  achieves a sublinear regret upper bound regardless  of the distribution of the means. 
\item Furthermore, these two algorithms require knowledge of the graph structure in ballooning setting, which is often difficult to obtain in practice. Therefore, we design refined versions of these two algorithms that do not rely on  graph  information and also establish regret upper bounds.
\end{enumerate} 
We list the regret upper bounds of the algorithms proposed in this paper in Table~\ref{tab:regret}.

\begin{table}[htbp]
  \centering
  \caption{\textcolor{blue}{Regret of different algorithms under two settings}}
  \label{tab:regret}
  \begin{tabular}{@{\hspace{2pt}}
                  >{\raggedright\arraybackslash}m{1.8cm}|%
                  >{\centering \arraybackslash}m{4.0cm}|%
                  c|c@{}}
    \toprule
    \multicolumn{1}{c|}{Setting} &
    \multicolumn{1}{c|}{Algorithm} &
    \multicolumn{2}{c}{Regret} \\ \midrule
    \multirow{3}{*}{ Stationary}
      & Double-UCB (without graph) & \multicolumn{2}{c}{$ O(\sqrt{T}\log T)$ } \\ \cmidrule{2-4}
      & Conservative-UCB (without graph) & \multicolumn{2}{c}{$ O(\log T/(\Delta_{2\epsilon})^2)$  } \\ \cmidrule{2-4}
      & Conservative-UCB (with graph)    & \multicolumn{2}{c}{ $O(\sqrt{T}\log T)$ } \\ \midrule
    \multirow{5}{*}{ Ballooning}
      &                       & Assumption~\ref{assumption0}: $\mathcal{P}$ is $\mathcal{N}(0,1)$  & Assumption~\ref{assumption0}: $\mathcal{P}$ is $U(0,1)$  \\ \cmidrule{2-4}
      & Double-UCB-BL (with graph)    & $O(\log T\sqrt{Te^{2\epsilon\sqrt{2\log T}}})$ & $\Omega(T)$ \\ \cmidrule{2-4}
      & Conservative-UCB-BL (with graph)    &$O(\sqrt{T\log T}\log T)$  &  $O(\sqrt{T\log T}\log T)$  \\ \cmidrule{2-4}
      & Double-UCB-BL (without graph) & $O(\log T\sqrt{Te^{2\epsilon\sqrt{2\log T}}})$ & $\Omega(T)$ \\ \cmidrule{2-4}
      & Conservative-UCB-BL (without graph) &    \multicolumn{2}{c}{Without Assumption~\ref{assumption0}: $O\left (\sqrt{\alpha(G_T)TH} + H\log(T)/(\Delta_{\min}^T)^2\right)$ } \\ \bottomrule
  \end{tabular}
\end{table}

\section{Related Works}

Bandits with side observations were first introduced by \citet{mannor2011bandits} for adversarial settings. They proposed two algorithms: ExpBan, a hybrid algorithm combining expert and bandit algorithms based on clique decomposition of the side observations graph; and ELP, an extension of the well-known EXP3 algorithm \citep{auer2002nonstochastic}. The works 
\citep{caron2012leveraging,hu2020problem}  considered stochastic bandits with side observations, proposing the UCB-N, UCB-NE, and TS-N algorithms, respectively. The regret upper bounds they obtain are of the form $ \sum_{c \in \mathcal{C}}\frac{\max_{i \in c}\Delta_i \ln(T)}{(\min_{i \in c}\Delta_i)^2} $, where $\mathcal{C}$ is the clique covering of the side observation graph, $\Delta_i$ is the gap between the expectation
of the optimal arm and the $i$-th arm, $T$ is the time horizon. 
 
There have been some works that employ techniques beyond clique partition. The works \citep{buccapatnam2014stochastic,buccapatnam2018reward} proposed an algorithm named UCB-LP, which combines a version of eliminating arms \citep{even2006action} suggested by
\citet{auer2010ucb} with linear programming to incorporate the graph structure. This algorithm has a regret guarantee of $\sum_{i\in D}\frac{\ln(T)}{\Delta_i}+K^2$, where $D$ is a particularly selected dominating set, $K$ is the number of arms. The work \citep{cohen2016online} used a method based on  elimination and provides  regret upper bound   $\tilde{O}(\sqrt{\alpha T}) $, where $\alpha$ is the independence number of the underlying graph. Another work \citep{lykouris2020feedback} utilized a hierarchical approach inspired by elimination to analyze the feedback graph, demonstrating that UCB-N and TS-N have regret bounds of order $\tilde{O}(\sum_{i \in I}\frac{1}{\Delta_i})$, where $I$ is an independent set of the graph.  There are also some works that consider the case where the feedback graph is a random graph \citep{alon2017nonstochastic,ghari2022online,esposito2022learning}.

Up to now, there is limited research considering scenarios where the number of arms can change. Mortal bandits \citep{chakrabarti2008mortal} was the first to explore this dynamic setting. Their model assumes that each arm has a lifetime budget, after which it automatically  disappear and is replaced  by a new arm. Since their algorithm needs to continuously explore newly available arms in this setting, they only provided an upper bound of the mean regret per time step. \citet{tekin2012online,karthik2022best,karthik2024optimal} studied the restless bandit problem, where the state of the arms changes but their number remains constant.
\citet{ghalme2021ballooning} considered the ``ballooning multi-armed bandits" where the number of arms keeps increasing. They show that  the regret grows linearly without any distributional assumptions on the arrival
of the arms' means. With the assumption that the optimal arm arrives early with high probability, their proposed algorithm BL-MOSS can achieve sublinear regret. 
In this paper, we also  consider the ``ballooning" setting without making assumptions about the optimal arm's arrival pattern but instead use the feedback graph model mentioned above.

Clustering bandits \citep{gentile2014online,li2016art,yang2024optimal,wang2024online} are also relevant to our work. Typically, these models assume that a set of arms (or items) can be clustered into several unknown groups. Within each group, the observations associated to  each of the arms follow the same distribution with the same mean. However, we do not employ a defined concept of clustering groups. Instead, we establish connections between arms by forming an edge only when their means are less than a threshold $\epsilon$, thereby creating a graph feedback structure. Correlated bandits \citep{gupta2021multi} is also relevant to our work, where the rewards between different arms are correlated. Choosing one arm allows observing the upper bound of another arm's reward. In contrast,  we assume that the underlying feedback follows a graph structure and consider the ballooning setting.

\section{Problem Formulation} 
\subsection{Graph Feedback with Similar Arms} \label{sec:feedback_graph}
We consider a stochastic $K$-armed bandit problem with an undirected feedback graph and time horizon $T$ (where $K \leq T$). At each round $t$, the learner selects an arm $i_t$, obtains a reward $X_t(i_t)$.  Without losing generality, we assume that the rewards are bounded in $[0,1]$ or $\frac{1}{2}$-subGaussian\footnote{This is simply to provide a unified description of both bounded rewards and subGaussian rewards. Our results can be easily extended to other subGaussian distributions.}.  The expectation of $X_t(i)$ is denoted as $ \mu(i):=\mathbb{E}[X_t(i)] $. 
Graph $G:=(V,E)$ denotes the underlying graph that captures all the feedback relationships over the arms set $V$. An edge $i \leftrightarrow j$ in $E$ means that $i$ and $j$ are $\epsilon$-similarty, i.e., 
 \[|\mu(i)-\mu(j)| < \epsilon,\] 
where $\epsilon$ is some constant greater than $0$. The learner can get a side observation of arm $i$ when pulling arm $j$, and vice versa. Let $N_i$ be a subset of $\{1,..,K\}$ and  denote the observation set of arm $i$ consisting of $i$ and its neighbors in $G$. Let $T_i(t)$ denote the total number of times arm $i$ is seletced up to time $t$ and  $O_t(i)$ denote the number of  observations of arm $i$ up to time $t$.  We assume that each node in graph $G$ contains a self-loop, i.e., the learner can observe the reward of the pulled arm.  

Let $i^{*}$ denote the expected reward of the optimal arm, i.e., $\mu(i^{*})=\max_{i \in \{1,...,K\}}\mu(i)$. The gap between the expectations of the optimal arm $i^*$ and arm $i$ is denoted as $\Delta_i:=\mu(i^{*})-\mu(i).$  The minimum and maximum gaps are denoted by  $\Delta_{\min}:=\mu(i^{*})-\max_{j \neq i^{*}}\mu(j)$ and $\Delta_{\max}:=\mu(i^{*})-\min_{j}\mu(j)$, respectively. 
A policy, denoted as $\pi$,  selects arm $i_t$ to play at time step $t$ based on the history plays and rewards.
The performance of the policy $\pi$ is measured by the cumulative regret
\begin{equation}
    R_T(\pi):= \mathbb{E} \bigg[ \sum_{t=1}^{T}\mu(i^{*})-\mu(i_t) \bigg]. 
\end{equation}

\subsection{The ballooning bandit setting}

This setting is the same as the graph feedback with similar arms described above, except that the number of arms  increases over time. Let $K(t)$ denote the set of  available arms at round $t$. We consider a tricky case where only one arm $a_t$  arrives at each round $t$. For each $t$, the total number of arms satisfies $|K(t)|=t$.  Let $G_t$ denote the underlying graph at round $t$.
The newly arrived arm may be connected to  previous arms, depending on whether their means satisfy the $\epsilon$-similarity. In this setting, the optimal arm may vary over time. Let $i_t^{*}$ denote  the optimal arm at round $t$, i.e., $\mu(i_t^{*})=\max_{i \in K(t)}\mu(i) $. The regret over $T$ rounds  is then given by 
\begin{equation}
    R_{T}(\pi):= \mathbb{E} \bigg[ \sum_{t=1}^{T}\mu(i_t^{*})-\mu(i_t) \bigg]. 
\end{equation}

\subsection{Notations}
We list the notations used in this paper in Table~\ref{tab:notation}.

\begin{table}[htbp]
\centering
\caption{\textcolor{blue}{Notation used in this paper}}
\label{tab:notation}
\renewcommand{\arraystretch}{1.2} 
\begin{tabular}{@{}cc@{}} 
\toprule
\textbf{Symbol} & \textbf{Description}  \\ 
\midrule
$i_t$ & index of the arm played at time $t$   \\
$i^{*},i_t^{*}$ & index of the optimal arm in stationary and ballooning settings    \\
$\bar{\mu}_t(i),\mu(i)$ &  empirical mean and expected  reward of arm $i$  \\
$T_i(t)$ & number of times arm $i$ is pulled up to time $t$  \\
$N_i$ & set of arm $i$ and its neighbors   \\
$O_t(i)$ & number of observations of arm $i$ till time $t$  \\
$G$ & underlying graph in stationary setting  \\
$G_T (G_T^{\mathcal{P}})$& underlying graph in ballooning setting at time $T$ (with Assumption~\ref{assumption0}) \\
$\gamma(G)$ & domination number of $G$ (Definition~\ref{def:domination})  \\
$\alpha(G) $ & independence number of graph $G$ (Definition~\ref{def:ind})  \\
$\Delta_i$ & reward gap between optimal arm and arm $i$  \\
$\Delta_{\max}$ & maximum gap between optimal arm and suboptimal arms  \\
$\Delta_{\max}^T$ & maximum gap over all pairs of arms in ballooning settings  \\
$\Delta_{\min}$ & maximum gap between optimal arm and suboptimal arms \\
$\Delta_{\min}^T$ & minimum gap over all pairs of arms  in ballooning settings  \\

$G_{2\epsilon}$ & subgraph with arms $\{i \in V: \mu(i^{*})-\mu(i) < 2\epsilon \}$ \\
$\Delta_{2\epsilon}$ & minimum gap among arms in $G_{2\epsilon}$  \\
$\mathcal{I}(N_i)$ & set of all independent dominating sets in $N_i$  \\
$\mathcal{I}(i^{*})$ & $\mathcal{I}(i^{*}) = \bigcup_{i \in N_{i^{*}}} \mathcal{I}(N_i)$  \\
$M$ & number of arms close to the optimal arms  \\
$H$ & number of changes in the optimal arm \\

\bottomrule
\end{tabular}
\end{table}

\section{Stationary Environments}
In this section, we consider the problem of graph feedback with similar arms in stationary environments, i.e., the number of arms remains constant. We first analyze the structure of the feedback graph. Then, we provide a problem-dependent regret lower bound. Following that, we introduce the Double-UCB and Conservative-UCB algorithms and provide their regret upper bounds respectively.

\begin{definition}[dominating set and domination number]
\label{def:domination}
    A dominating set \(S\) in a graph \(G\) is a set of vertices such that every vertex not in \(S\) is adjacent to a vertex in \(S\). The domination number of \(G\), denoted as \(\gamma(G)\), is the smallest size of a dominating set in $G$.
\end{definition}

\begin{definition}[independent set and independence number]
\label{def:ind}
    An independent  set contains vertices that are not adjacent to each other. The independence number of \(G\), denoted as \(\alpha(G)\), is the largest size of an independent set in \(G\).
\end{definition}

\begin{definition}[independent dominating set and independent domination number]
An independent dominating set in \(G\) is a set that is both dominating and  independent. The independent domination number of \(G\), denoted as \(i(G)\), is the smallest size of such a set. 
    
\end{definition}
    
For a general graph $G$, the following holds immediately:
\[ \gamma(G) \leq i(G) \leq \alpha(G) .\]


For a feedback graph that satisfies the construction rule in Section~\ref{sec:feedback_graph}, we have the following proposition. 
\begin{proposition}
    \label{pro1}
Let $G$ denote the feedback graph in the similar arms setting. We have $$\gamma(G)=i(G) \geq \frac{\alpha(G)}{2}.  $$ 
\end{proposition}

\begin{proof}[Proof Sketch]
    The first equation can be obtained by proving that $G$ is a claw-free graph and using the fact that $\gamma(G)=i(G)$ for any claw-free graph (see Lemma \ref{claw-free} in Appendix \ref{appendix_a}). The second inequality can be obtained by a double counting argument. The details are provided in  Appendix \ref{appendix_b}. 
\end{proof}

Proposition \ref{pro1} shows that $\gamma(G) \leq \alpha(G) \leq 2\gamma(G) $.  Once we obtain the  regret bounds related to the independence number, we  simultaneously obtain the regret bounds related to the domination number. Therefore, 
we can obtain regret bounds  based on the minimum dominating set without using the feedback graph to explicitly target exploration. This cannot be guaranteed in the standard graph feedback bandit problem \citep{lykouris2020feedback}.

\subsection{\textcolor{blue}{Lower Bounds}}
Before presenting our algorithms, we first investigate the regret lower bound of this problem. 
Previous works \citep{buccapatnam2014stochastic, li2020stochastic}  have provided instance-dependent lower bounds for the standard graph feedback problem. The proof idea is to construct two bandit instances: in the second instance, a specific suboptimal arm from the first problem is upgraded to have a slightly higher mean than the optimal arm, while the means of all other arms remain unchanged. Then, using a method similar to that in \citep{lai1985asymptotically}, an instance-dependent lower bound is derived.
However, this approach does not help us establish an instance-dependent lower bound. When the  means of arms in the second problem instance change, the feedback graph structure of the problem also changes. The two instances no longer share the same information structure, making KL-divergence decomposition infeasible. Therefore, we turn to proving a minimax lower bound instead.

When the feedback graph structures of two problem instances are the same, we have the following divergence decomposition. The proof is given in Appendix~\ref{appendix_b}. 
\begin{lemma}
\label{kldecom}
    Let $ v=(P_1,...,P_K)$ be the reward distributions associated with one $K$-armed bandit with feedback graph $G$, and let $v'=(P_1',...,P_K')$ be the reward distributions associated with another  $K$-armed bandit with feedback graph $G'$. Assume that $G=G'$.  Let $\mathbb{P}_{v,\pi}$ and  $\mathbb{P}_{v',\pi}$ be the probability measures induced by the $n$-round interconnection of $\pi$ and $v$ (respectively, $\pi$ and $v'$). The expectations under $\mathbb{P}_{v,\pi}$ is denoted as $\mathbb{E}_v$. Then, 
    \begin{equation}
        \mathrm{KL}(\mathbb{P}_{v,\pi},\mathbb{P}_{v',\pi}) = \sum_{i=1}^K \sum_{j \in N_i } \mathbb{E}_v[T_i(T)]\mathrm{KL}(P_j,P_j').
    \end{equation}
\end{lemma}

\begin{theorem}
\label{lowerbound}
    Let $T > \frac{1}{4\epsilon^2}$. For any policy $\pi$, there exists a problem instance  such that
    \begin{equation}
        R_T(\pi) \geq \frac{1}{80}\sqrt{T}.
    \end{equation}
\end{theorem}
\begin{proof}
    Let $\eta \in (0,\frac{\epsilon}{3})$ be a constant to be chosen later. We construct two problem instances $v$ and $v'$ that are difficult to distinguish. The reward distributions of the arms of two bandit instances are both  $\frac{1}{2}$-subGaussian. The reward mean vectors are represented by $\mu$ and $\mu'$, respectively. For the first instance, 
    the first arm is the optimal arm with a mean of $2\eta$. The means of the other arms are between $0$ and $\eta$. The mean vectors can be denoted as 
    \[ \mu = (2\eta, \mu_2,...,\mu_k,...,\mu_K ) ,\]
    where $\mu_i \in (0,\eta), \forall i \neq 1$.
    Since $ 2\eta < \epsilon $,  the feedback graph for this problem instance is a complete graph, where all arms can observe each other. The  distribution induced by the environment and  $\pi$  is $\mathbb{P}_{v,\pi}$. The expectations under $\mathbb{P}_{v,\pi}$ is denoted by $\mathbb{E}_{v}$.

   In the second bandit environment, the mean of the $k$-th ($k \neq 1$) arm is $3\eta$, and the means of the other arms are the same as in the first environment, i.e.,
   \[ \mu'= (2\eta,\mu_2,...,3\eta,...,\mu_K ),\]
   Clearly, the feedback graph of $v'$ is also a complete graph. The optimal arm of $v'$ is the $k$-th arm, while the optimal arm of $v$ is the first arm. We have
   \begin{equation}
       R_T(\pi,v) \geq \mathbb{P}_{v,\pi}(T_1(T) \leq T/2) \frac{T\eta}{2} \text{\quad and \quad}  R_T(\pi,v') \geq \mathbb{P}_{v',\pi}(T_1(T) > T/2) \frac{T\eta}{2}
   \end{equation}
   Then use the Bretagnolle-Huber inequality,
   \begin{equation}
    \label{pinsker}
       \begin{aligned}
         R_T(\pi,v)+ R_T(\pi,v') &\geq   \frac{T\eta}{2} \big(\mathbb{P}_{v,\pi}(T_1(T) \leq T/2) + \mathbb{P}_{v',\pi}(T_1(T) > T/2) \big)\\
         &\geq \frac{T\eta}{4} \exp\big(-\mathrm{KL}(\mathbb{P}_{v,\pi},\mathbb{P}_{v',\pi} ) \big).
       \end{aligned}
   \end{equation}

   Now we bound the term $\mathrm{KL}(\mathbb{P}_{v,\pi},\mathbb{P}_{v',\pi} )$. 
   Note that $\mu_k \in (0,\eta)$, thus
   \[  \mathrm{KL}\big(\mathcal{N}(\mu_k,\frac{1}{4}),\mathcal{N}(3\eta,\frac{1}{4})\big) =2 (3\eta-\mu_k)^2 \leq 18\eta^2. \]
   Since the  two problem instances differ only in the reward distribution of the $k$-th arm, for any $i \in [K]$, we have 
   \[ \sum_{j \in N_i}  \mathrm{KL}(P_j,P_j') =  \mathrm{KL}\big(\mathcal{N}(\mu_k,\frac{1}{4}),\mathcal{N}(3\eta,\frac{1}{4})\big). \]
   Since the feedback graph of two bandit instances is complete graph, meeting the requirements of  Lemma~\ref{kldecom}, 
   we have 
   \begin{equation}
   \begin{aligned}
       \mathrm{KL}(\mathbb{P}_{v,\pi},\mathbb{P}_{v',\pi} ) &= \sum_{i=1}^K \sum_{j \in N_i} \mathbb{E}_v[T_i(T)] \mathrm{KL}(P_j,P_j')\\ 
       &=\sum_{i=1}^K \mathbb{E}_v[T_i(T)] \mathrm{KL}\big(\mathcal{N}(\mu_k,\frac{1}{4}),\mathcal{N}(3\eta,\frac{1}{4})\big) \\
       &\leq 18\sum_{i =1}^K \mathbb{E}_v[T_i(T)] \eta^2 \\
       &=18T\eta^2.
       \end{aligned}
   \end{equation}
Substituting into \cref{pinsker}, we obtain
\begin{equation}
     R_T(\pi,v)+ R_T(\pi,v') \geq \frac{T\eta}{4} \exp(-18T\eta^2).
\end{equation}

Now by choosing $\eta= \frac{1}{6\sqrt{T}}$ (since $T > \frac{1}{4\epsilon^2}$, this choice satisfies $\eta < \frac{\epsilon}{3}$), we have
\begin{equation}
2\max( R_T(\pi,v), R_T(\pi,v')) \geq  R_T(\pi,v)+ R_T(\pi,v')  \geq 
    \frac{T}{4}\frac{1}{6\sqrt{T}} \exp(-1/2) > \frac{1}{40}\sqrt{T}.
\end{equation}

\end{proof}

This lower bound does not include terms related to the graph structure. 
In our problem setting, this result is intuitive. Our proof considers only the case of a complete graph, where all arms are tightly connected to the optimal arm. These arms are hard to distinguish and therefore dominate the regret. In contrast, arms  that are not connected to the optimal arm have means at least $\epsilon$ smaller than that of the optimal arm. This constant gap makes them easily distinguishable and thus contributes little to the overall regret.

\begin{remark}
    Previous works \citep{cohen2016online, mannor2011bandits} have  derived the regret lower bound for graph bandit problems as $O(\sqrt{\alpha(G)T})$. 
\citet{cohen2016online} directly cites the results from \citep{mannor2011bandits}, where the setting is  adversarial bandits.   
Our proof  considers only the case of a complete graph, and one might think that the lower bounds in these earlier papers could be restricted  to complete graphs to recover the lower bound in our settings. However, the presence of similarity structures makes it difficult to directly apply the adversarial bandit results of \citep{mannor2011bandits}.  In fact, such an approach is not valid.
In their proof, they constructed arms with sufficiently close means that are not connected, which is not feasible in our similarity setting. Therefore, we provide a new proof.
\end{remark}


\subsection{Double-UCB}
\label{sec:DUCB}

This particular feedback structure inspires us to distinguish arms within the independent set first. This is a straightforward task because the  distance between the mean of each arm  in the independent set is greater than $\epsilon$. Subsequently,  we learn from the arm with the maximum upper confidence bound in the independent set  and its neighborhood, which may include the optimal arm. Our algorithm alternates between the two processes simultaneously.

\begin{algorithm}[tb]
    \caption{Double-UCB}
    \label{alg:D_UCB}
 \begin{algorithmic}[1]
    \STATE {\bfseries Input:}  Horizon $T,$ $\delta \in (0,1)$
    \STATE Initialize $\mathcal{I}=\emptyset, t=0 ,O_t(i)=0$ for all $i$
    \WHILE{$t \leq T$}
    \REPEAT
    \STATE Select an arm $i_t$ that has not been observed.
    \STATE $\mathcal{I}=\mathcal{I} \bigcup \{i_t\}$
    \STATE $\forall i \in N_{i_t},$ update $O_t(i),\bar{\mu}_t(i)$
    \STATE $t=t+1$
    \UNTIL{All arms have been observed at least once}
    \STATE  $j_t= \arg\max_{j \in \mathcal{I}} \bar{\mu}_t(j)+ \textcolor{blue}{\sqrt{\frac{\log(\sqrt{2}/\delta)}{O_t(j)}} }$
    \STATE Select arm $i_t= \arg \max_{i \in N_{j_t}} \bar{\mu}_t(i)+\textcolor{blue}{\sqrt{\frac{\log(\sqrt{2}/\delta)}{O_t(i)}}}$
    \STATE $\forall i \in N_{i_t}$, update $O_t(i),\bar{\mu}_t(i)$
    \STATE $t=t+1$

    \ENDWHILE
    
 \end{algorithmic}
 \end{algorithm}

\cref{alg:D_UCB}  shows the pseudocode of our method, where we define $\bar{\mu}_t(i)$ as the empirical mean of arm $i$ at round $t$.
Steps 4-9 identify an independent set $\mathcal{I}$ in \(G\) and play each arm in the independent set once. This process does not require knowledge of the  graph structure and runs at most $\alpha(G)$ rounds. 
Step 10 calculates the arm $j_t$ with the maximum upper confidence bound in the independent set.  After a finite number of rounds, the optimal arm is likely to fall within $N_{j_t}$.
Step~11 uses the UCB algorithm again to select arm in $N_{j_t}$. This step does not require knowing the graph structure either, because the observed neighborhood can be saved when searching for an independent set in Steps 4-9.
This algorithm employs  UCB rules twice for arm selection, hence it is named Double-UCB.



\textbf{Regret Analysis of Double-UCB.}
We use $\mathcal{I}(N_{i})$ to denote the set of all  independent dominating sets of graph formed by $N_i$. 
Let 
\[\mathcal{I}(i^{*}):= \bigcup_{i\in N_{i^{*}}} \mathcal{I}(N_i). \] 
Note that the elements in $\mathcal{I}(i^{*})$ are independent sets rather than individual nodes. For every  $ I \in \mathcal{I}(i^{*})$, we have $|I| \leq 2 $. 

 \textcolor{blue}{
\begin{theorem}
    \label{bound_ducb}
    Assume that the reward distribution is $\frac{1}{2}$-subGaussian or bounded in $[0,1]$. By setting $\delta = \frac{1}{T},$ the regret of Double-UCB after $T$ rounds can be bounded as
    \begin{equation}
        \label{tmp00}
        \begin{aligned}
        R_T(\pi_{\text{Double}}) \leq  32\log^2(\sqrt{2}T)\max_{I \in \mathcal{I}(i^{*} )}\sum_{i \in I \backslash \{i^{*}\}}\frac{1}{\Delta_i} +C_1\log(\sqrt{2}T)
         + \Delta_{\max} +4\epsilon,
        \end{aligned}
    \end{equation}
    where\begin{equation}
    C_1 = 
    \begin{cases}
        \frac{4(\log(2\gamma(G))+1)}{\max\{\epsilon,\Delta_{\text{min}}\}} + \frac{4\pi^2}{3\max\{\epsilon,\Delta_{\text{min}}\}}, & \Delta_{\min} < \epsilon \text{ or } \epsilon \ll \Delta_{\min}, \\
        \frac{4(\log(2\gamma(G))+1)}{\epsilon} + \frac{4\pi^2}{3\max\{\epsilon,\Delta_{\text{min}}\}}, &  \text{otherwise.}
    \end{cases}
\end{equation}
\end{theorem} 
}

\begin{proof}[Proof Sketch]
The regret can be decomposed into two parts. The first part of  regret arises  from selecting a neighborhood  $N_{j_t}$ that does not contain the optimal arm. The algorithm can easily distinguish the suboptimal arms in $N_{j_t}$ from the optimal arm. The regret caused by this part is of order $O(\log T)$.
The second part of  regret comes from selecting a suboptimal arm in the set $N_{j_t}$ that includes the optimal arm. This part can be viewed as applying UCB rule on a graph with an independence number less than 2. The detailed proof is provided in Appendix \ref{appendix_b}.
\end{proof}


From \cref{bound_ducb}, we have the following \emph{gap-free} upper bound, which is looser than Theorem~2 but does not depend on the suboptimality gaps $\{\Delta_i \}_{i\in V }$. This gap-free upper bound matches the lower bound proved in Theorem~\ref{lowerbound} up to a logarithmic factor.
\begin{corollary}
\label{gapfree_ducb}
    The regret of Double-UCB is bounded by $$ R_T(\pi_{\text{Double}}) \leq 16\sqrt{T}\log(\sqrt{2}T)+C_1 \log(\sqrt{2}T)+ \Delta_{\max}+ 4\epsilon.$$
\end{corollary}
\begin{proof}
    See Appendix~\ref{appendix_b}.
\end{proof}



\subsection{Conservative-UCB}
\label{sec:CUCB}
Double-UCB is a very natural algorithm for the similar arms setting, but the upper bound suffers logarithm term $O(\log^2T)$. To design an algorithm that achieives $O(\log T)$ bound, we propose the Conservative-UCB, which simply modifies Step 11 of Double-UCB (\cref{alg:D_UCB}) to 
\[i_t= \arg \max_{i \in N_{j_t}} \bar{\mu}_t(i)-\sqrt{\frac{\log(\sqrt{2}/\delta)}{O_t(i)}}.\] \cref{alg:C_UCB} shows the pseudocode of Conservative-UCB. This new index function is conservative when exploring arms in $N_{j_t}$. It does not immediately try each arm but selects those that have been observed a sufficient number of times. Intuitively, the similarity structure guarantees that the optimal arm is observed enough times ($\Omega(\log T)$), enabling the UCB rule to effectively disregard suboptimal arms.

\begin{algorithm}[tb]
    \caption{Conservative-UCB}
    \label{alg:C_UCB}
 \begin{algorithmic}[1]
    \STATE {\bfseries Input:}  Horizon $T,$ $\delta \in (0,1)$
    \STATE Initialize $\mathcal{I}=\emptyset, t=0,\forall i, O_t(i)=0$
    \WHILE{$t \leq T$}
    \STATE Steps 4-10 in Double-UCB
    \STATE Select arm $i_t= \arg \max_{i \in N_{j_t}} \bar{\mu}_t(i)-\textcolor{blue}{\sqrt{\frac{\log(\sqrt{2}/\delta)}{O_t(i)}} }$
    \STATE $\forall i \in N_{i_t},$ update $O_t(i),\bar{\mu}_t(i)$
    \STATE $t =t+1$
    \ENDWHILE

 \end{algorithmic}
 \end{algorithm}

\textbf{Regret Analysis of Conservative-UCB}
Let $G_{2\epsilon}$ denote the subgraph with arms $\{ i \in V: \mu(i^{*})-\mu(i) < 2\epsilon \}$ and $\Delta_{2\epsilon}$ denote the  minimum gap among the arms in  $G_{2\epsilon}$, i.e.,  $\Delta_{2\epsilon}:= \min_{i,j \in G_{2\epsilon}}|\mu(i)-\mu(j)|$. Let $\mathcal{I}:=\{\alpha_1,\alpha_2,...,\alpha^{*},...,\alpha_{|\mathcal{I}|}  \}$ denote the independent set obtained by the algorithm (Lines 4-9 in Double-UCB), where $\alpha^*$ denotes the arm whose neighborhood  includes the optimal arm, i.e., $i^* \in N_{\alpha^{*}}$. 

We divide the regret into two parts. The first part is the regret caused by choosing a neighborhood $N_{j_t}$ that does not contain the optimal arm $i^{*}$, and the analysis for this part follows the same approach as for  Double-UCB. 

The second part is the regret of choosing suboptimal arms in the set $N_{\alpha^{*}}$ that contains the optimal arm.  It can be proven that if the optimal arm  is observed more than $\frac{4\log(\sqrt{2}/\delta)}{(\Delta_{2\epsilon})^2}$ times,  the algorithm will select the optimal arm with high probability. For any suboptimal arm $i  \in N_{\alpha^{*}}$ and round $t$, the following events hold simultaneously with high probability ( Lemma~\ref{lemma1}):
\begin{equation}
    \label{tmp0}
    \bar{\mu}_t(i^{*})- \sqrt{\frac{\log(\sqrt{2}/\delta)}{O_t(i^{*})}} >   \mu(i^{*})-2\sqrt{\frac{\log(\sqrt{2}/\delta)}{O_t(i^{*})}}\geq \mu(i^{*})  - \Delta(i)=\mu(i),
\end{equation}
and 
\begin{equation}
    \label{tmp1}
    \bar{\mu}_t(i) - \sqrt{\frac{\log(\sqrt{2}/\delta)}{O_t(i)}} < \mu(i).
\end{equation}
Since the optimal arm satisfies \cref{tmp0} and  the suboptimal arms satisfy \cref{tmp1} with high probability, based on the selection rule of Algorithm~\ref{alg:C_UCB},  the suboptimal arms are unlikely to be selected with high probability. 

\textcolor{blue}{
The key to the  analysis lies in ensuring that the optimal arm can be observed more than $\frac{4\log(\sqrt{2}/\delta)}{(\Delta_{2\epsilon})^2}$ times. 
 We divide the arms in $N_{\alpha^{*}}$ into two parts:
\[ E_1 = \{ i \in N_{\alpha^{*}}: \mu(i) < \mu(\alpha^{*})  \},  \]
\[ E_2 = \{ i \in N_{\alpha^{*}}: \mu(i) \geq \mu(\alpha^{*})  \}. \]
We have $i^{*} \in E_2$. 
Because every arm in the graph induced by  $N_{\alpha^{*}}$ is connected to $\alpha^{*}$,  a sufficiently large number of time steps guarantees that $\alpha^{*}$ is observed  more than $\frac{4\log(\sqrt{2}T/\delta)}{(\Delta_{2\epsilon})^2}$ times. Consequently, the arms in $E_1$ are ignored (see  \cref{tmp0,tmp1}).
Since $E_2$ is a complete graph,  pulling any arm in $E_2$  always reveals the reward of the optimal arm, thus the optimal arm can be observed $\frac{4\log(\sqrt{2}T/\delta)}{(\Delta_{2\epsilon})^2}$ times. 
}
Hence, we have the following theorem:
\begin{theorem}
    \label{bound_cucb}
    Under the same conditions as \cref{bound_ducb}, the regret  of Conservative-UCB  is  bounded by
    \begin{equation}
        \begin{aligned}
        R_T(\pi_{\text{Cons}}) \leq \frac{16\epsilon \log(\sqrt{2}T)}{(\Delta_{2\epsilon})^2} + C_1\log(\sqrt{2}T)
           + \Delta_{\max} + 4\epsilon.
        \end{aligned}
    \end{equation}
\end{theorem}

The proof of Theorem~\ref{bound_cucb} is deferred to Appendix~\ref{appendix_b}. 
Compared to Double-UCB, the regret upper bound of Conservative-UCB decreases by a logarithmic factor, but includes an additional  problem-dependent  term $\Delta_{2\epsilon}$. 

\subsection{\textcolor{blue}{Conservative-UCB With Known Graph}}

The regret upper bound of the Conservative-UCB algorithm depends on the parameter $\Delta_{2\epsilon}$. Similar gap-dependent results have been reported in previous studies, such as \citep{caron2012leveraging,buccapatnam2014stochastic,cortes2020online}. Due to the conservative exploration mechanism of Conservative-UCB, it is challenging to derive a gap-free  regret bound. In this subsection, we propose a more refined algorithm that exploits the graph structure to overcome this limitation, yielding an upper bound that is gap-free.

First, we revisit the discussion in the Conservative-UCB section. The key idea is to ensure that the optimal arm is observed  often enough. Suppose the optimal arm lies in $N_{j_t}$. If $N_{j_t}$ is not a clique (i.e., a complete graph), selecting an arm from $N_{j_t}$ does not guarantee that the optimal arm will be observed. In the original analysis of Conservative-UCB, each selection of an arm from $N_{j_t}$ allows the observation of arm $j_t$. Therefore, arm $j_t$ can be used as a bridge. We temporarily ignore arms whose means  are lower than that of arm $j_t$. The remaining arms, whose means lie between those of arm $j_t$ and the optimal arm, form a clique. Pulling any arm from this clique always reveals the optimal arm.

From the above discussion, it follows that if $N_{j_t}$ forms a clique, the optimal arm can be guaranteed to be sufficiently observed, thereby making it possible to improve the regret upper bound. Building on this insight, we propose a new algorithm, with its pseudocode presented in Algorithm~\ref{alg:C_UCB_Graph}.

\begin{algorithm}[tb]
    \caption{Conservative-UCB With Known Graph}
    \label{alg:C_UCB_Graph}
 \begin{algorithmic}[1]
    \STATE {\bfseries Input:}  Horizon $T,$ $\delta \in (0,1)$
    \STATE Initialize $\mathcal{I}=\emptyset, t=0 ,O_t(i)=0$ for all $i$
    \WHILE{$t \leq T$}
    \REPEAT
    \STATE Select an arm $i_t$ that has not been observed.
    \STATE $\mathcal{I}=\mathcal{I} \bigcup \{i_t\}$
    \STATE $\forall i \in N_{i_t},$ update $O_t(i),\bar{\mu}_t(i)$
    \STATE $t=t+1$
    \UNTIL{All arms have been observed at least once}
    \STATE  $j_t= \arg\max_{j \in \mathcal{I}} \bar{\mu}_t(j)+\sqrt{\frac{\log(\sqrt{2}/\delta)}{O_t(j)}}$
 
    \IF{$N_{j_t}$ is not a complete graph}
    \STATE $\mathcal{I}_{j_t} = \text{ arbitrarily select a set from }  S_{j_t} $
    \STATE $j'_t= \arg\max_{j \in \mathcal{I}_{j_t}} \bar{\mu}_t(j)+\sqrt{\frac{\log(\sqrt{2}/\delta)}{O_t(j)}}$
    \STATE Select arm $i_t= \arg \max_{i \in N_{j'_t}\bigcap N_{j_t}} \bar{\mu}_t(i)-\sqrt{\frac{\log(\sqrt{2}/\delta)}{O_t(i)}}$
    \ELSE 
    \STATE Select arm $i_t= \arg \max_{i \in N_{j_t}} \bar{\mu}_t(i)-\sqrt{\frac{\log(\sqrt{2}/\delta)}{O_t(i)}}$
    \ENDIF
    \STATE $\forall i \in N_{i_t}$, update $O_t(i),\bar{\mu}_t(i)$
    \STATE $t=t+1$

    \ENDWHILE
    
 \end{algorithmic}
 \end{algorithm}
If $N_{j_t}$ is a complete graph, we directly select the arm as $i_t= \arg \max_{i \in N_{j_t}} \bar{\mu}_t(i)-\sqrt{\frac{\log(\sqrt{2}/\delta)}{O_t(i)}}$. 
If $N_{j_t}$ is not a complete graph, to ensure that arm selection takes place within a clique, we first identify a maximal independent set within $N_{j_t}$.
Recall that we use $\mathcal{I}(N_{i})$ to denote the set of all  independent dominating sets of the graph formed by $N_i$. The following set is a collection of independent sets in $N_{j_t}$, and each intersection with $N_{j_t}$ is a complete graph.
\begin{equation}
\label{sjt}
    S_{j_t} =  \{I \in \mathcal{I}(N_{j_t}): N_i \bigcap N_{j_t} \text{ is a complete graph } \forall i \in I\}.
\end{equation}
We show that $S_{j_t}$ is non-empty by  explicitly constructing a set that belongs to it. Let $(a_1,a_2)=\arg\max_{(i,j) \in N_{j_t} \times N_{j_t}} |\mu(i)-\mu(j)|$. Obviously, both $N_{a_1} \bigcap N_{j_t}$ and $N_{a_2} \bigcap N_{j_t}$ are  complete graphs.  Since $N_{j_t}$ is not a complete graph, $|\mu(a_1)-\mu(a_2)| \geq \epsilon.$ Thus, $a_1$ and $a_2$ form an independent set. We have $\{a_1,a_2\} \in S_{j_t}$. 

Then we  then apply the UCB-style criterion  to identify a clique $N_{j'_t}$ that contains the optimal arm, and then select an arm within this clique (Line 12-14).

\begin{figure}[htbp]
  \centering
  \begin{tikzpicture}[>=stealth,  
                      every node/.style={font=\small}]
    \draw (1.5,0) -- (8.5,0);
    \draw (2, 0.12) -- (2,-0.12) node[below=3pt] {$\alpha_1^{*}$};
    \draw (4, 0.12) -- (4,-0.12) node[below=3pt] {$\alpha^{*}$};
    \draw (6, 0.12) -- (6,-0.12) node[below=3pt] {$\alpha_2^{*}$};
    \draw (7, 0.12) -- (7,-0.12) node[below=3pt] {$i^{*}$};
    \draw[<->] (2, 0.5) -- (6, 0.5)
      node[midway,above] {$>\varepsilon$};
    \draw[<->] (4, 0.9) -- (7, 0.9)
      node[midway,above] {$<\varepsilon$};
  \end{tikzpicture}
  \caption{A schematic diagram of the regret analysis, which illustrates the case where $N_{\alpha^{*}}$ is not a complete graph.}
\end{figure}

\textbf{Regret Analysis of Algorithm~\ref{alg:C_UCB_Graph}. } Recall that $\mathcal{I}=\{\alpha_1,\alpha_2,...,\alpha^{*},...,\alpha_{|\mathcal{I}|}  \}$ denotes the independent sets obtained by Lines 4-9, where $\alpha^*$ denotes the arm that includes the optimal arm, i.e., $i^* \in N_{\alpha^{*}}$. 
The regret can be divided into two parts: 
\begin{equation}
\begin{aligned}
\sum_{t=1}^{T}\sum_{i \in V} \Delta_i\mathbbm{1}\{i_t=i\}= \sum_{t=1}^{T}\sum_{i\notin N_{\alpha^{*}}}\Delta_i\mathbbm{1}\{i_t=i\} + \sum_{t=1}^{T}\sum_{i\in N_{\alpha^{*}} }\Delta_i\mathbbm{1}\{i_t=i\}.
\end{aligned}
\end{equation}
Similar to Double-UCB, the first part of the regret arises from selecting a neighborhood $N_{j_t}$ that does not contain the optimal arm, which can be bounded by $O(\log T)$. For the second part, if $ N_{\alpha^{*}}$ is  a complete graph, we can bound the regret by Lemma~\ref{lemma_layer}. If $ N_{\alpha^{*}}$ is not complete graph, the algorithm  proceeds to Steps 12-14. Let $\alpha^{*}_1$ and $\alpha^{*}_2$ denote the elements in $\mathcal{I}_{j_t}$. 
Assume that the neighborhood of $\alpha_2^{*}$  contains the optimal arm. 
We have the following bound,

\begin{equation}
\label{three_part}
 \sum_{t=1}^{T}\sum_{i\in N_{\alpha^{*}}}\Delta_i\mathbbm{1}\{i_t=i\} \leq 2\epsilon \sum_{t=1}^{T}\mathbbm{1}\{j_t' = \alpha_1^{*}\} + \sum_{t=1}^{T}\sum_{i\in N_{\alpha^{*}_2} \bigcap N_{\alpha^{*}} }\Delta_i\mathbbm{1}\{i_t=i\}. 
\end{equation}

Note that if $N_{\alpha^{*}}$ is not a complete graph, then $|\mathcal{I}_{\alpha^{*}}|=2$. The mean gap between the two arms in $\mathcal{I}_{\alpha^{*}}$ is greater than  $\epsilon$. Therefore, the first term in \cref{three_part} can be bounded by $O(\log T)$. 
As for the second part in \cref{three_part}, 
we use the lemma below.
\begin{lemma}
\label{lemma_layer}
 Let $V^{*}$ denote an arbitrary complete subgraph that contains the optimal arm. Let $\delta=\frac{1}{T}$. Then 
   \begin{equation}
        \sum_{t=1}^T\sum_{i \in V^{*} }\mathbb{P}(i_t=i)\Delta_i \leq  8\sqrt{T}\log(\sqrt{2}T) + \epsilon. 
    \end{equation}
\end{lemma}
The $V^{*}$ in Lemma~\ref{lemma_layer} is guaranteed to exist.  If $N_{\alpha^{*}}$ is a complete graph, then $V^{*} = N_{\alpha^{*}}$; otherwise, $V^{*}= N_{\alpha^{*}} \bigcap N_{\alpha_2^{*}}$. 




In summary, the following regret upper bound holds.
\begin{theorem}
\label{bound_cucb_new}
    Under the same conditions as \cref{bound_ducb}, the regret  of Algorithm~\ref{alg:C_UCB_Graph} is  bounded by
    \begin{equation}
        R_T(\pi_{\text{Cons-Graph}}) \leq  8\sqrt{T}\log(\sqrt{2}T)+ 
  (C_1 + \frac{8}{\epsilon})\log(\sqrt{2}T) + 3\epsilon +  \Delta_{\mathrm{max}}.
    \end{equation}
\end{theorem}
The proofs of Lemma~\ref{lemma_layer} and Theorem~\ref{bound_cucb_new} are deferred to Appendix B. 


\subsection{UCB-N}
UCB-N are first proposed in \citep{caron2012leveraging} for   standard graph feedback model and has been thoroughly analyzed in subsequent works \citep{hu2020problem,lykouris2020feedback}.
The pseudocode of UCB-N is provided in Algorithm \ref{alg:UCB_N}. One may wonder whether UCB-N can achieve similar regret upper bounds. In fact, if  UCB-N uses the same upper confidence function as ours,
it has a similar regret upper bound to Double-UCB. We have the following theorem:

\begin{algorithm}[tb]
    \caption{UCB-N}
    \label{alg:UCB_N}
 \begin{algorithmic}[1]
    \STATE {\bfseries Input:}  Horizon $T,$ $\delta \in (0,1)$
    \STATE Initialize $\mathcal{I}=\emptyset, t=0,\forall i,O_t(i)=0$
    \WHILE{$t \leq T$}
    \STATE Select arm $i_t= \arg \max_{i \in V} \bar{\mu}_t(i)+\sqrt{\frac{\log(\sqrt{2}/\delta)}{O_t(i)}}$
    \STATE $\forall i \in N_{i_t},$ update $O_t(i),\bar{\mu}_t(i)$
    \STATE $t =t+1$
    \ENDWHILE

 \end{algorithmic}
 \end{algorithm}

\begin{theorem}
    \label{bound_ucbn}
    Under the same conditions as \cref{bound_ducb}, the regret  of UCB-N  satisfies
    \begin{equation}
        \begin{aligned}
      R_T(\pi_{\text{UCB-N}}) \leq  16\sqrt{T}\log(\sqrt{2}T)+C_1 \log(\sqrt{2}T)+ \Delta_{\max} +2\epsilon.
        \end{aligned}
    \end{equation}
  
\end{theorem}



\begin{remark}
    The gap-free upper bound of UCB-N is also $O(\sqrt{T}\log T)$. Due to the similarity assumption, the regret upper bound of UCB-N is improved compared to \citep{lykouris2020feedback}, where their regret upper bound is of order $O(\sqrt{\alpha(G)T}\log T)$. 
\end{remark}
\begin{remark}
     Double-UCB and Conservative-UCB are specifically designed for similarity feedback structure and may fail in the case of standard graph feedback settings. This is because the optimal arm may be connected to an arm with a very small mean, so the neighborhood $N_{j_t}$ selected in Step 10 (Algorithm \ref{alg:D_UCB}) may not  include the optimal arm. However, under the ballooning setting, UCB-N cannot achieve sublinear regret, while Double-UCB and Conservative-UCB can be naturally applied in this setting and achieve sublinear regret under certain conditions.
\end{remark}

\section{Ballooning Environments}
This section considers the setting where the number of arms increased over time. This problem presents significant challenges, as prior research has relied on strong assumptions  to achieve sublinear regret. The similarity structure we propose helps solve the bandit problem in the ballooning setting. 
Intuitively, if a newly arrived arm has a mean value very close to arms that have already been distinguished, the algorithm does not need to distinguish it further. This may lead to a significantly smaller number of truly effective arrived arms than $T$, making it easier to obtain a sublinear regret bound.

\subsection{Double-UCB-BL for Ballooning Settings}

\begin{algorithm}[tb]
    \caption{Double-UCB-BL for Ballooning Settings}
    \label{alg:D_UCB_BL}
 \begin{algorithmic}[1]
    \STATE {\bfseries Input:}  Horizon $T,$ $\delta \in(0,1)$
    \STATE Initialize $\mathcal{I}=\emptyset, t=0, O_t(i)=0$ for all $i$
    \FOR{$t=1$ {\bfseries to} $T$}
    \STATE Arm $a_t$ arrives
    \STATE Feedback graph  $G_t$ is updated
    \IF{$a_t \notin N_{\mathcal{I}}$}
    \STATE $\mathcal{I}=\mathcal{I} \bigcup \{a_t\}$
    \ENDIF
    \STATE 
    $j_t= \arg\max_{j \in \mathcal{I}} \bar{\mu}_t(j)+ \sqrt{\frac{\log(\sqrt{2}/\delta)}{O_t(j)}}$
    \STATE Pulls arm $i_t= \arg \max_{i \in N_{j_t}} \bar{\mu}_t(i)+\sqrt{\frac{\log(\sqrt{2}/\delta)}{O_t(i)}}$
    
    \STATE $\forall i \in N_{i_t},$ update $O_t(i),\bar{\mu}_t(i)$

    \ENDFOR
    
 \end{algorithmic}
 \end{algorithm}

\cref{alg:D_UCB_BL} shows the pseudocode for our method Double-UCB-BL, where `BL' stands for `Ballooning'. 
For any set $\mathcal{S}$, let $N_{\mathcal{S}}$ denote  the set of arms linked to $\mathcal{S}$, i.e., $N_{\mathcal{S}}:=\bigcup_{i \in \mathcal{S}}N_i $. Upon the arrival of each arm, we first check whether it belongs to $N_{\mathcal{I}}$. If not, the arm is added  to $\mathcal{I}$ to form a new independent set.
The  independent set $\mathcal{I}$ is constructed in an online manner as new arms arrive, while the other parts of the algorithm remain identical to Double-UCB. 

\subsubsection{Regret Upper Bounds}

First, we make the following assumption for Algorithm \ref{alg:D_UCB_BL}.
\begin{assumption}
    \label{assumption0}
    The means of each arrived arms are independently sampled from a distribution $\mathcal{P}$, i.e., 
 \[\mu(a_1),\mu(a_2),...,\mu(a_T) \stackrel{i.i.d.}{\sim} \mathcal{P}.\]
\end{assumption}
This assumption merely facilitates a more explicit upper bound. Remark~\ref{remark_without} discusses the regret upper bounds of our algorithms in the absence of Assumption~\ref{assumption0}. 

Let $ v =(\mu(a_1),...,\mu(a_T)) \sim \mathcal{P}$ denote a bandit instance. Under Assumption \ref{assumption0}, we redefine the regret  as
\begin{equation}
    R_T(\pi):=\mathbb{E}_{v \sim \mathcal{P}}\bigg[\mathbb{E} \bigg[ \sum_{t=1}^{T}\mu(i_t^{*})-\mu(i_t) \Big| v \bigg] \bigg].
\end{equation}
Let $\mathcal{I}_t$ denote the independent set at round $t$ and $\alpha_t^{*}\in \mathcal{I}_t$ denote the (unique) arm whose neighborhood includes the optimal arm $i_t^{*}$. Let
\begin{equation}
    \mathcal{A}:=\{a_t:t \in [T], a_t \in N_{\alpha_t^{*}}\} 
\end{equation}
denote the set of arms linked to $\alpha_t^{*}$.
The first challenge  in ballooning settings is the potential presence of numerous arms whose means are very close to that of the optimal arm. In other words, the set  $\mathcal{A}$ may be very large. To address this challenge, 
we first define a quantity that is easy to analyze as the  upper bound for all arms falling into $N_{\alpha_t^{*}}$.
Define 
\begin{equation}
        M := \sum_{t=1}^{T} \mathbbm{1}\{|\mu(a_t)-\mu(i_t^{*})| < 2\epsilon\}.
\end{equation}
It is easy to verify that for any bandit instance $v$,
\[\{a_t:t\in [T], a_t \in N_{\alpha_t^{*}} \} \subseteq \{a_t:t\in [T],|\mu(a_t)-\mu(i_t^{*})| < 2\epsilon\}.\]
Thus, we have $$\mathbb{E}\big[|\mathcal{A}|\big] \leq \mathbb{E}[M] = \sum_{t=1}^{T}\mathbb{P}(|\mu(a_t)-\mu(i_t^{*})| < 2\epsilon).$$ 

Since the upper bound of $\mathbb{E}[M]$ can be derived more easily under a specific distribution (please refer to the proof of Corollary~\ref{gapfree_ducb_bl} in Appendix~\ref{appendix_b} ), we transform the problem of analyzing the intractable term $|\mathcal{A}|$ into analyzing $\mathbb{E}[M]$.

\textcolor{blue}{
The second challenge lies in the fact that our regret is likely influenced by the independence number, while under the ballooning setting, the graph's independence number is a random variable. Let $G_T^{\mathcal{P}}$ denote the underlying graph under ballooning setting, then $G_T^{\mathcal{P}}$ is a random graph. It is important to note that $G_T^{\mathcal{P}}$ is not a standard Erd\H{o}s--R\'enyi random graph, but a random geometric graph, due to the dependencies among its edges. A random geometric graph is an undirected graph in which nodes are randomly placed in a metric space, and edges join pairs of nodes whose mutual distance does not exceed a prescribed threshold. In our setting, the $T$ vertices (arms) are randomly placed on the real line according to a distribution $\mathcal{P}$, and the distance threshold is $\epsilon$. 
To tackle this issue, we  derive a high-probability upper bound for the independence number (see Lemma \ref{ind_number_new} in Appendix~\ref{appendix_a}).}
Now, we can give the following upper bound of Double-UCB-BL. The proof of Theorem~\ref{bound_ducb_bl} is deferred to Appendix~\ref{appendix_b}.
\textcolor{blue}{
\begin{theorem}
    \label{bound_ducb_bl}
    Under the same conditions as \cref{bound_ducb} along with Assumption \ref{assumption0}, by setting $\delta = \frac{1}{T},$ the regret of Double-UCB-BL after $T$ rounds can be bounded as
    \begin{equation}
        \begin{aligned} 
         R_T(\pi_{\text{Double-BL}}) \leq   \sqrt{ \mathbb{E}\big[(\alpha(G_T^{\mathcal{P}}))^2\big]\mathbb{E}\big[ (\Delta_{\max}^T)^2\big]}  (\frac{4\log(\sqrt{2}T)}{\epsilon^2} + 1)   + 4\sqrt{2T\mathbb{E}[M]\log(\sqrt{2}T)}+ 2\epsilon .
        \end{aligned}
    \end{equation}
    where $\Delta_{\max}^T :=\max_{i,j \in [T]}|\mu(a_i)-\mu(a_j)|$.
\end{theorem}
}
If $\mathcal{P}$ is a Gaussian distribution, we have the following corollary.
\begin{corollary}
    \label{gapfree_ducb_bl}
    If $\mathcal{P}$ is the Gaussian distribution $\mathcal{N}(0,1)$, 
    we have $\mathbb{E}[M]=O(\log(T)e^{2\epsilon\sqrt{2\log(T)}})$ and $\mathbb{E}\big[(\alpha(G_T^{\mathcal{P}}))^2\big] = O(\frac{\log T}{\epsilon^2}),\mathbb{E}\big[ (\Delta_{\max}^T)^2\big] =O(\log T)$. 
    The asymptotic regret upper bound is of order \[O\Big(\log(T)\sqrt{Te^{2\epsilon\sqrt{2\log(T)}}}\Big).\] 
    The order of $e^{\sqrt{2\log(T)}} $ is smaller than any power of $T$. For example, if $T>e^n$ where $n$ is a positive integer, we have $$e^{\sqrt{2\log(T)}} \leq T^{\sqrt{2/n}}. $$ 
\end{corollary}

\subsubsection{Regret Lower Bounds of Double-UCB-BL} Define $\mathcal{B}:= \{a_t: t \in [T], \frac{\epsilon}{2}< \mu(i_t^{*})-\mu(a_t) < \epsilon\}$. Then we have $\mathcal{B} \subset \mathcal{A}$.
Let 
\[ B := \mathbb{E}\Big[ \sum_{t=1}^{T} \mathbbm{1}\{ \frac{\epsilon}{2}< \mu(i_t^{*})-\mu(a_t) < \epsilon\}  \Big] .\]
We have $\mathbb{E}[|\mathcal{B}|] = B $. 
Assume that arm $a_t$ arrives and falls into set $\mathcal{B}$ at round $t$. If the algorithm selects $j_t = \alpha_t^{*} $ in Step 9, the upper confidence bound of arm $a_t$ is highest ($+\infty$). The algorithm will select arm $a_t$ and lead to a regret larger than $\frac{\epsilon}{2}$. If the algorithm selects $j_t \neq \alpha_t^{*} $ in Step 9, the resulting regret is also larger than $\frac{\epsilon}{2}$. Therefore,
if we estimate the size of $|\mathcal{B}|$,  we get a simple regret lower bound $\frac{|\mathcal{B}|\epsilon}{2}$. 
\begin{theorem}
    \label{lowerbound_ducb_bl}
    Under Assumption \ref{assumption0}, the regret lower bound of Double-UCB-BL must satisfy 
    $R_T(\pi) \geq \frac{B\epsilon}{2}.$ 
    Specifically, 
    \begin{enumerate}
        \item If $\mathcal{P}$ is $\mathcal{N}(0,1)$, we have $B=\Omega(\log (T) e^{\frac{3\epsilon}{4}\sqrt{\log T}} )$. 
        \item If $\mathcal{P}$ is the uniform distribution $U(0,1)$ and $0 < \epsilon <1$, we can calculate that $B \geq \frac{(1-\epsilon)\epsilon}{2}(T-1)$.
        \item If $\mathcal{P}$ is the half-triangle distribution with probability density function as 
    $ f(x)=2(1-x) \mathbbm{1}\{0 < x < 1\}$ and $0<\epsilon<1$, 
    we can also calculate that $B \geq \frac{3\epsilon^2(1-\epsilon)^2}{4}(T-1) $.
    \end{enumerate}
\end{theorem}
This lower bound is far from optimal but is sufficient to show if $\mathcal{P}$ is a uniform distribution or half-triangle distribution, the regret must be of linear order.


\subsection{\textcolor{blue}{Conservative-UCB-BL for Ballooning settings}}

\begin{algorithm}[tb]
    \caption{Conservative-UCB-BL for Ballooning Settings}
    \label{alg:CUCB_BL_new}
 \begin{algorithmic}[1]
    \STATE {\bfseries Input:}  Horizon $T, $ $\delta \in (0,1)$
    \STATE Initialize $\mathcal{I}=\emptyset,t=0, O_t(i)=0$ for all $i$
    \FOR{$t=1$ {\bfseries to} $T$}
     \STATE  Steps 4-9 in Double-UCB-BL
    \IF{$N_{j_t}$ is not a complete graph}
    \STATE $\mathcal{I}_{j_t} = \text{ arbitrarily select a set from }  S_{j_t} $
    \STATE $j'_t= \arg\max_{j \in \mathcal{I}_{j_t}} \bar{\mu}_t(j)+\sqrt{\frac{\log(\sqrt{2}/\delta)}{O_t(j)}}$
    \STATE Select arm $i_t= \arg \max_{i \in N_{j'_t}\bigcap N_{j_t}} \bar{\mu}_t(i)-\sqrt{\frac{\log(\sqrt{2}/\delta)}{O_t(i)}}$
    \ELSE 
    \STATE Select arm $i_t= \arg \max_{i \in N_{j_t}} \bar{\mu}_t(i)-\sqrt{\frac{\log(\sqrt{2}/\delta)}{O_t(i)}}$
    \ENDIF
    \STATE $\forall i \in N_{i_t},$ update $O_t(i),\bar{\mu}_t(i)$
    \ENDFOR
 \end{algorithmic}
 \end{algorithm}
The poor performance under the uniform distribution limits the applicability of Double-UCB-BL. The fundamental cause lies in the aggressive exploration strategy of the UCB algorithm, which attempts to pull every arm that enters the set $\mathcal{A}$.
In this section, we adapt the Conservative-UCB approach to the ballooning setting.
\cref{alg:CUCB_BL_new} presents the pseudocode for our Conservative-UCB-BL algorithm. The independent-set search component is identical to that of Double-UCB-BL, whereas the arm-selection component matches that in  Algorithm~\ref{alg:C_UCB_Graph}. This modification avoids exploring every arm that enters $N_{j_t}$; instead, it selects only the arm that has been observed a sufficient number of times and has the highest lower confidence bound.


Based on the proofs of Theorem 1 and Theorem 2, we can obtain the regret upper bound for Algorithm~\ref{alg:CUCB_BL_new}. The detailed proof is deferred to Appendix B.
\begin{theorem}
    \label{bound_cucb_bl}
    Under the same conditions as \cref{bound_ducb} along with Assumption \ref{assumption0}, by setting $\delta = \frac{1}{T},$ the regret of Conservative-UCB-BL is  bounded by
    \begin{equation}
        \begin{aligned}
       R_T(\pi_{\text{Cons-BL}}) &\leq \sqrt{\mathbb{E}\big[(\alpha(G_T^{\mathcal{P}}))^2\big] \mathbb{E}\big[ (\Delta_{\max}^T)^2\big]} ( \frac{4\log(\sqrt{2}T)}{\epsilon^2} + 1 ) + 8 \sqrt{2T\log T}\log(\sqrt{2}T) \\
       &\quad + \frac{16\log^2(\sqrt{2}T)}{\epsilon} + 6\epsilon \log T .
        \end{aligned}
    \end{equation}
\end{theorem}

Note that the regret upper bound of Conservative-UCB-BL   does not involve $M$, i.e., the potential many arms with means very close to the optimal arm. The distribution $\mathcal{P}$ only affects the independence number $\alpha(G_T^{\mathcal{P}})$ and the maximum mean gap $\Delta^T_{\mathrm{max}}$ in the upper bound. 
\begin{enumerate}
    \item If $\mathcal{P}$ is $\mathcal{N}(0,1)$, from Corollary~\ref{gapfree_ducb_bl}, we have $\mathbb{E}\big[(\alpha(G_T^{\mathcal{P}}))^2\big] = O(\frac{\log T}{\epsilon^2})$ and $\mathbb{E}\big[ (\Delta_{\max}^T)^2\big] =O(\log T)$. Hence, 
    \[ R_T(\pi_{\text{Cons-BL}}) =O( \sqrt{T\log T}\log T).\]
    \item If  $\mathcal{P}$ is uniform distribution $U(0,1)$ or half-triangle distribution the same as Theorem~\ref{bound_cucb_bl}, we have $\mathbb{E}\big[(\alpha(G_T^{\mathcal{P}}))^2\big] \leq \frac{1}{\epsilon^2}$ and $\mathbb{E}\big[ (\Delta_{\max}^T)^2\big] = 1$. The regret upper bound is also $ R_T(\pi_{\text{Cons-BL}}) =O( \sqrt{T\log T}\log T).$
\end{enumerate}

\begin{remark}
\label{remark_without}
    Recall that $M= \sum_{t=1}^{T} \mathbbm{1}\{|\mu(a_t)-\mu(i_t^{*})| < 2\epsilon\}$ denotes the upper bound of $|\mathcal{A}|$.
    We impose Assumption~\ref{assumption0} in order to derive a more explicit upper bound.
    Since the magnitude of $M$ is not intuitive and $M$ plays a dominant role in the regret bound of Doubel-UCB-BL. Under Assumption~\ref{assumption0}, the regret bound of Double-UCB-BL can be stated in a more transparent and interpretable manner.
    Without Assumption 1, let $H := \sum_{t=1}^T  \mathbbm{1}\{\mu(a_t)=\mu(i_t^{*})\} $ denote the number of changes in the optimal arm and $\alpha(G_T)$ denote the independence number of the graph at time step $T$, the regret upper bounds of Theorem~\ref{bound_ducb_bl} becomes 
    \[ \alpha(G_T)\Delta^T_{\mathrm{max}}( \frac{4\log(\sqrt{2}T)}{\epsilon^2} + 1 ) + 4\sqrt{2TM\log(\sqrt{2}T)} +2\epsilon,\]
    and  the regret upper bound of Theorem~\ref{bound_cucb_bl} is
    \[  \alpha(G_T)\Delta^T_{\mathrm{max}}( \frac{4\log(\sqrt{2}T)}{\epsilon^2} + 1 ) + 8\sqrt{2TH}\log(\sqrt{2}T) + \frac{8H\log(\sqrt{2}T)}{\epsilon} + H\epsilon + 3\epsilon. \]

    If $H=o(\sqrt{T})$, the regret bound of  Conservative-UCB-BL is sublinear.

\end{remark}

\subsection{\textcolor{blue}{Gap-dependent regret bound }}
Algorithm~\ref{alg:CUCB_BL_new} selects arms in $N_{j_t}$ or  $N_{j_t'} \bigcap N_{j_t}$ based on whether $N_{j_t}$ is a complete graph. Since $N_{j_t'} \bigcap N_{j_t} \subset N_{j_t}$, our arm selection can be regarded as being performed within $ N_{j_t}$. We can replace steps 5-11 with 
\[ i_t= \arg \max_{i \in N_{j_t}} \bar{\mu}_t(i)-\sqrt{\frac{\log(\sqrt{2}/\delta)}{O_t(i)}}.  \]
Then the regret upper bound of this new algorithm can also serve as the regret upper bound for Algorithm~\ref{alg:CUCB_BL_new}. Using an analysis method similar to Theorem~\ref{bound_cucb}, we can obtain the following gap-dependent regret upper bound without Assumption~\ref{assumption0}. The detailed proof is deferred to Appendix B.

\begin{theorem}
\label{gap_cucb_bl}
Let $\Delta_{\min}^T := \min_{i \neq j}  | \mu(i)-\mu(j) |>0 $. 
    Under the same conditions as \cref{bound_ducb}, by setting $\delta = \frac{1}{T},$ the regret of Conservative-UCB-BL is  bounded by
    \begin{equation}
        \begin{aligned}
        R_T({\pi_{\text{Cons-BL}}})  \leq \alpha(G_T)\Delta_{\max}^T \left( \frac{4\log(\sqrt{2}T)}{\epsilon^2} + 1 \right)+  \frac{16H\epsilon\log(\sqrt{2}T)}{(\Delta_{\min}^T)^2} + 4\epsilon,
        \end{aligned}
    \end{equation}
    where $H= \sum_{t=1}^{T} \mathbbm{1}\{\mu(a_t)=\mu(i_t^{*})\}$.
\end{theorem}

\section{Unknown Graph Structure in Ballooning Environments}
Algorithms \ref{alg:D_UCB_BL} and \ref{alg:CUCB_BL_new} both require knowledge of the graph structure, particularly in the process of updating the independent set in Algorithm \ref{alg:D_UCB_BL} (Steps 4-8). However, in practical applications, it is difficult to know the connection information of each arm.  In this section, we  design an algorithm built on Double-UCB-BL (or Conservative-UCB-BL) that does not rely on graph information. 

We first introduce the idea of algorithm design.
Without loss of generality, we assume that $T$ is an integer multiple of $\tau$. 
The time horizon can be divided into uniform intervals: 
\[ [1,\tau],[\tau+1,2\tau],...,[T-\tau+1,T], \]
where we use the interval to denote finite discrete time steps. In each interval, we 
find an independent set in the current feedback graph. As long as the order of $\tau$ is  larger than $\alpha(G_T)$, the time steps spent for finding an independent set can be ignored. We can identify the independent set by checking if all arms have been observed at least once (Steps 5-8 in Algorithm \ref{alg:D_UCB}), and this process does not require knowledge of the graph structure. Once we obtain the independent set, we can use Double-UCB-BL (Steps 9-10 in \cref{alg:D_UCB_BL}) or Conservative-UCB-BL to select arms in the current interval.

\cref{alg:U_UCB_BL} shows the pseudocode of our method. This algorithm has a tuning parameter $\tau$. If $t \leq \tau$, the algorithm pulls arm $a_t$ directly. 
At every $\tau$ steps, an independent set is selected from the first $t-1$ arms (Steps 8-14). This process does not require knowledge of the graph structure and spends  at most $ \alpha(G_T) $ time steps. 
Based on the execution method of Step 17, we obtain the U-Double-UCB and U-Conservative-UCB algorithms respectively.

\begin{algorithm}[tb]
    \caption{U-Double-UCB (or U-Conservative-UCB) for Ballooning Setting}
    \label{alg:U_UCB_BL}
 \begin{algorithmic}[1]
    \STATE {\bfseries Input:}  Horizon $T,$ $\delta \in(0,1)$, positive integer $\tau$. 
    \STATE Initialize $\mathcal{I}=\emptyset, t=0, \forall i, O_t(i)=0$
    \FOR{$t=1$ {\bfseries to} $T$}
    \STATE Arm $a_t$ arrives
    \IF {$t \leq \tau$ }
    \STATE Pulls arm $i_t=a_t$
    \ELSIF {$t-1 \% \tau ==0$}
    \STATE Pulls each arm $i_t$ in $\mathcal{I}$, update $t$ and  $O_t(i),\bar{\mu}_t(i), \forall i \in N_{i_t}$
    \REPEAT
    \STATE Select an arm $i_t < t$ that has not been observed.
    \STATE $\mathcal{I}=\mathcal{I} \bigcup \{i_t\}$
    \STATE $\forall i \in N_{i_t},$ update $O_t(i),\bar{\mu}_t(i)$
    \STATE $t=t+1$
    \UNTIL{the previous $t-1$ arms  have been observed at least once}
    \ELSE 
    \STATE 
    $j_t= \arg\max_{j \in \mathcal{I}} \bar{\mu}_t(j)+\sqrt{\frac{\log(\sqrt{2}/\delta)}{O_t(j)}}$
    \STATE Pulls arm $i_t= \arg \max_{i \in N_{j_t}} \bar{\mu}_t(i)+\sqrt{\frac{\log(\sqrt{2}/\delta)}{O_t(i)}}$ \bigg(or $i_t= \arg \max_{i \in N_{j_t}} \bar{\mu}_t(i)-\sqrt{\frac{\log(\sqrt{2}/\delta)}{O_t(i)}}\bigg)$
    
    \STATE $\forall i \in N_{i_t},$ update $O_t(i),\bar{\mu}_t(i)$
    \ENDIF

    \ENDFOR
    
 \end{algorithmic}
 \end{algorithm}

 \begin{figure}[!htbp] 
    \centering 
    \centerline{	\includegraphics[width=12cm]{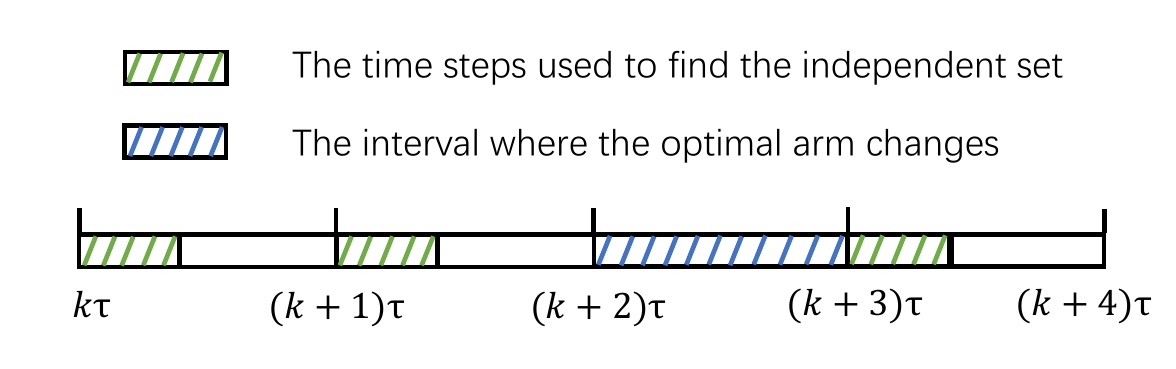}}
    \caption{Illustration of regret analysis. Green color denotes the time steps used to find independent sets, and the blue interval indicates where the optimal arm changed within that interval.  } 
    \label{fig:0}
\end{figure}

\subsection{Regret Analysis}

We use U-Double-UCB as an example to explain the  idea of  regret analysis. The analysis for U-Conservative-UCB is similar.
The regret analysis of Algorithm \ref{alg:U_UCB_BL} differs from the previous algorithms in two aspects. First, selecting the independent set causes additional regrets.
Since selecting an independent set within each interval takes at most $O(\frac{\sqrt{\log T}}{\epsilon})$ time steps with high probability, the regret for this part can be bounded by $O(\frac{T}{\tau}\sqrt{\log T})$.

Second, the independent set may not be connected to the optimal arms. 
Consider the interval $[k\tau+1,(k+1)\tau]$ (where $k$ is some positive integer). Our independent  set (denoted as $\mathcal{I}$) only considers the first $k\tau$ arms. Thus $\mathcal{I}$ may not be connected to the new optimal arm, i.e., $\exists t \in [k\tau+1,(k+1)\tau],  i_t^{*} \notin N_{\mathcal{I}} $.  If we select an arm from $N_{\mathcal{I}}$, we may miss the optimal arm.   Note that the event $\{ \exists t \in [k\tau,(k+1)\tau],  i_t^{*} \notin N_{\mathcal{I}}\} $ implies that the optimal arm has changed in $ [k\tau+1,(k+1)\tau]$.
Since the mean of each arm is randomly and independently sampled from  distribution $\mathcal{P}$, we have 
$ P(\mu(a_t)=\mu(i_t^{*}))=\frac{1}{t} $. Then the expectation of  times the optimal arm changes can be bounded by $\log T+1$. Therefore, the regret for the second part can be bounded by $O(\tau \log T)$. Figure \ref{fig:0}  illustrates the above analysis.

The other parts of the algorithm can be viewed as using Double-UCB-BL or Conservative-UCB-BL, thus the analysis is  similar to the previous algorithms. 

\begin{theorem}
    \label{u_ucb}
    Assume that the reward distribution is $\frac{1}{2}$-subGaussian or bounded in $[0,1]$. Let $\delta = \frac{1}{T}$. Under Assumption \ref{assumption0}, the regret of U-Double-UCB  can be bounded by
    \begin{equation}
        \begin{aligned} 
        R_T(\pi_{\text{U-Double-BL}})  \leq \frac{T}{\tau}\sqrt{\mathbb{E}\big[(\alpha(G_T^{\mathcal{P}}))^2\big] \mathbb{E}\big[ (\Delta_{\max}^T)^2\big]}+ 11\tau \log T \sqrt{ \mathbb{E}\big[ (\Delta_{\max}^T)^2\big]} + R_T(\pi_{\text{Double-BL}}).
        \end{aligned}
    \end{equation}
    Without Assumption \ref{assumption0}, the regret of  U-Conservative-UCB can be bounded by
    \begin{equation}
        \begin{aligned} 
        R_T(\pi_\text{U-Cons-BL})  \leq \frac{T}{\tau} \alpha(G_T) \Delta_{\max}^T + \tau H \Delta_{\max}^T+ R_T(\pi_\text{Cons-BL}).
        \end{aligned}
    \end{equation}
  
\end{theorem}
If $\mathcal{P}$ is $U(0,1)$ or half-triangle distribution, the regret of U-Double-UCB is also linear with $T$. 
If $\mathcal{P}$ is $\mathcal{N}(0,1)$, we have $\mathbb{E}\big[(\alpha(G_T^{\mathcal{P}}))^2\big] = O(\frac{\log T}{\epsilon^2})$ and  $ \mathbb{E}\big[ (\Delta_{\max}^T)^2\big] =O(\log T)$. Seting $\tau= O(\sqrt{T})$, the regret is dominanted by $R_T(\pi_\text{Double-BL})$. Thus,  
\[ R_T(\pi_\text{U-Double-BL})= O\Big(\log(T)\sqrt{Te^{2\epsilon\sqrt{2\log(T)}}}\Big).\] For U-Conservative-UCB, setting $\tau=\sqrt{\frac{\alpha(G_T)T}{H}}$, the regret is of order 
\[ R_T(\pi_\text{U-Cons-BL}) = O\left (\sqrt{\alpha(G_T)TH} + \frac{H\log(T)}{(\Delta_{\min}^T)^2}\right) .\]

\begin{remark}
   When the graph structure is unknown, we do not rely on Assumption~\ref{assumption0} in analyzing the regret of U-CUCB. This is because, without the graph structure, the arm-selection rule of U-CUCB is
   $i_t= \arg \max_{i \in N_{j_t}} \bar{\mu}_t(i)-\sqrt{\frac{\log(\sqrt{2}T/\delta)}{O_t(i)}}. $
   The regret bound involves  $ \frac{1}{(\Delta^T_{\mathrm{min}})^2}$, which is a problem-dependent bound, and Assumption~\ref{assumption0} is no longer applicable.

\end{remark}

\begin{figure*}[!htbp] 
    \centering 
       \subfloat[Bernoulli rewards]{
        \includegraphics[width=8cm]{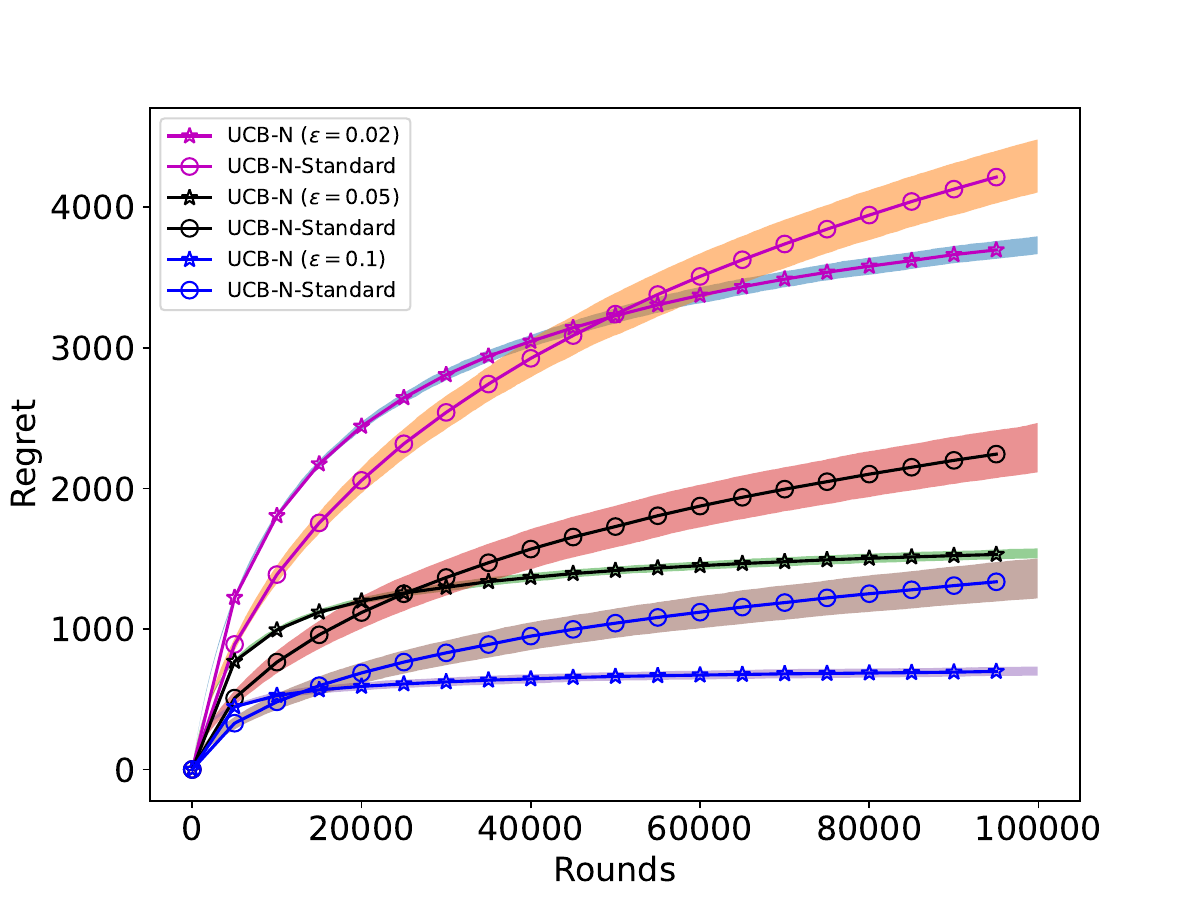}
           }
        \subfloat[Gaussian rewards]{
           \includegraphics[width=8cm]{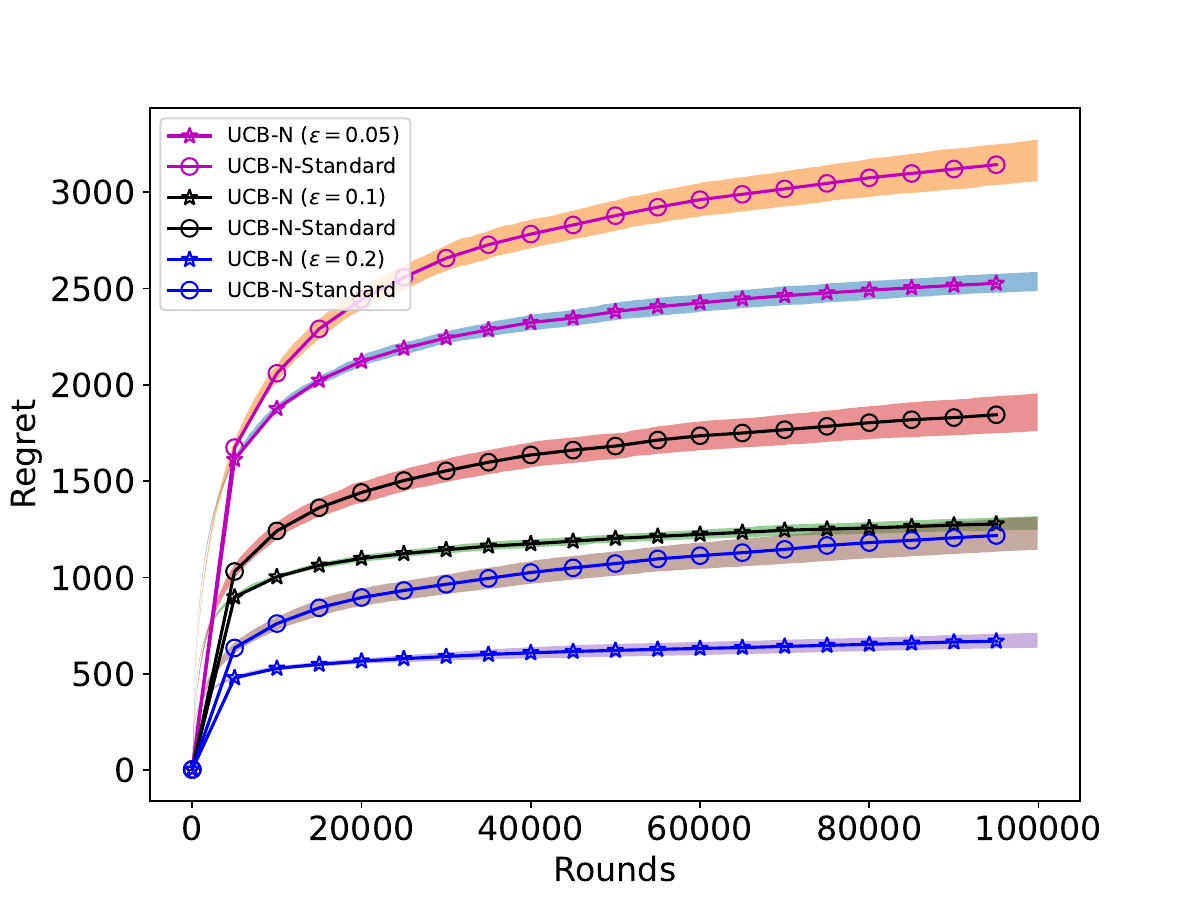} 
            }

    \caption{We run UCB-N on two different settings with either Bernoulli rewards or Gaussian rewards. ``UCB-N'' indicates that we run UCB-N on the graph with similarity structure, while ``UCB-N-Standard" means the graph feedback does not have a similarity structure. To ensure fairness in comparison,  the graphs used in both settings have  the same structure. The same color indicates that the same graph structure is used, while different colors correspond to different values of $\epsilon$. We set $T=10^5$ and $K=10^4$ in the experiment.} 
    \label{fig:1}
\end{figure*}
\section{Experiments}
\label{experiment}
\subsection{Stationary Settings}

We first compare the performance of UCB-N under standard graph feedback and graph feedback with similar arms.\footnote{Our code is available at  \url{https://github.com/qh1874/GraphBandits_SimilarArms}} The purpose of the first experiment is to show that the similarity structure improves the performance of the  UCB-N algorithm. 
To ensure fairness, we keep the graph structure of the problem instances the same in both cases. In the first instance, the arms satisfy the similarity assumption, whereas in the second instance, this assumption does not hold. We run UCB-N on both instances and refer to them as ``UCB-N" and ``UCB-N-Standard", respectively.
For each value of $\epsilon$, we generate $50$ different problem instances. The  regret is averaged on the $50$ instances. The 95\% confidence interval  is shown as a semi-transparent region in the figure.
\cref{fig:1} shows the performance of UCB-N under Gaussian and Bernoulli rewards. 
It can be observed that the regret of UCB-N in our settings is smaller than that in the standard graph feedback setting, thanks to the similarity structure. Additionally, the regret decreases as  $\epsilon$ increases, which is consistent  with theoretical results. 

\begin{figure}[!htbp] 
    \centering 
    
        \subfloat[Bernoulli rewards]{
         \includegraphics[width=8cm]{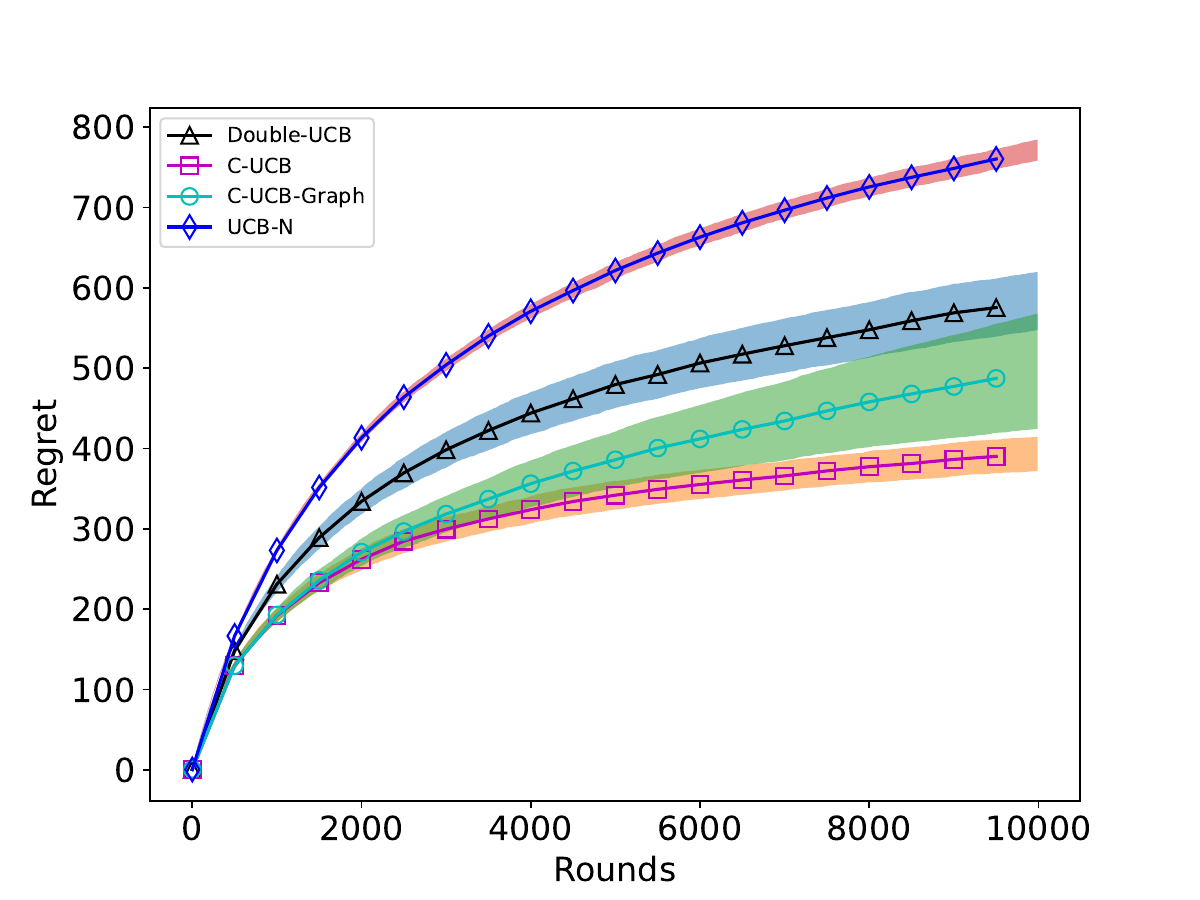}}
         \subfloat[Gaussian rewards]{
            	\includegraphics[width=8cm]{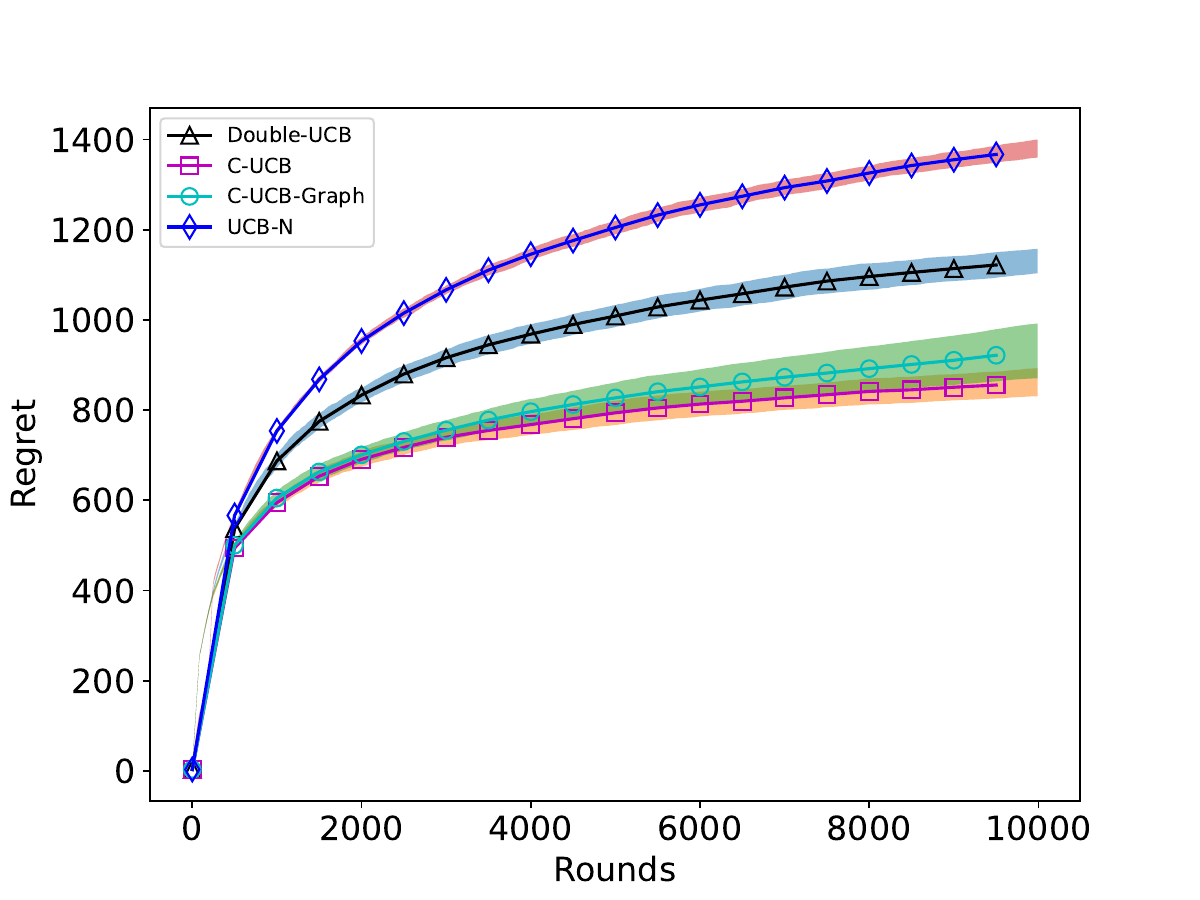} }
      
    \caption{Experimental results of UCB-N, Double-UCB, and Conservative-UCB  with $T=10^5,K=10^4,\epsilon=0.01$.   } 
    \label{fig:2}
\end{figure}

We then compare the performance of UCB-N, Double-UCB, and Conservative-UCB algorithms.  \cref{fig:2} shows the performance of the three algorithms with Gaussian and Bernoulli rewards. Although Double-UCB and UCB-N have similar regret bounds, the experimental performance of Double-UCB and Conservative-UCB is better than UCB-N. This may be because Double-UCB and Conservative-UCB directly learn on an independent set, effectively leveraging the graph structure features of similar arms.

\subsection{Ballooning Settings}
\begin{figure*}[!htbp] 
    \centering 
     \subfloat[]{
        \includegraphics[width=5.5cm]{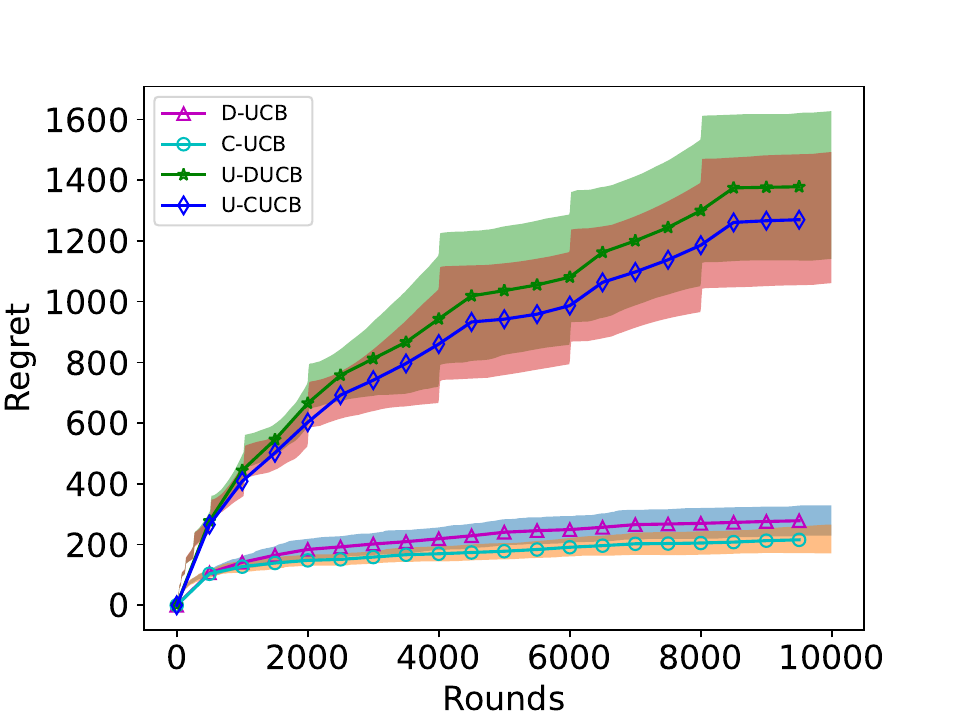}}
    \subfloat[]{
        \includegraphics[width=5.5cm]{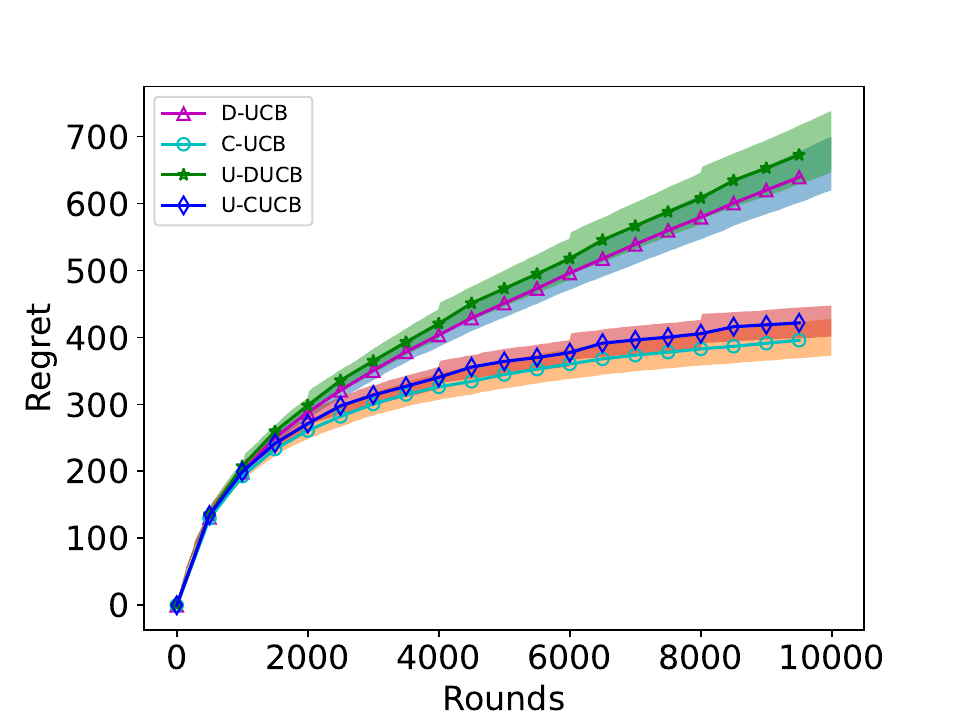} }
     \subfloat[]{   
        \includegraphics[width=5.5cm]{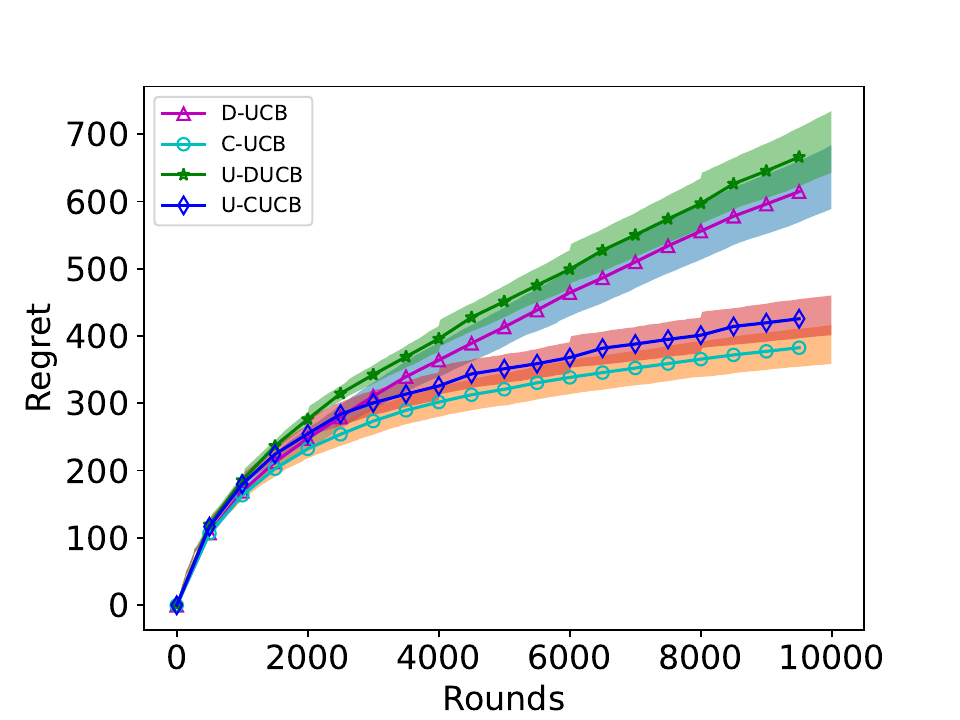} }
       
    \caption{Experimental results of Double-UCB, Conservative-UCB, U-Double-UCB and U-Conservative-UCB on ballooning Settings. Fig.~4(a) considers Gaussian rewards with $\mathcal{P}= \mathcal{N}(0,1)$ and $\epsilon=0.3$. Fig.~4(b) considers Bernoulli rewards with $\mathcal{P}=U(0,1)$ and $\epsilon=0.05$. Fig.~4(c) considers Bernoulli rewards with $\mathcal{P}$ being the half-triangle distribution and $\epsilon=0.05$. }
    \label{fig:3}
\end{figure*}
UCB-N is not suitable for ballooning settings since it selects each arrived arm at least once. The BL-Moss  algorithm \citep{ghalme2021ballooning} is specifically designed for the ballooning setting. However, this algorithm assumes that the optimal arm is more likely to appear in the early rounds and requires prior knowledge of the parameter $\lambda$ to characterize this likelihood, which is not consistent with our setting. Thus, we only compare  our proposed four algorithms: Double-UCB, Conservative-UCB, U-Double-UCB and U-Conservative-UCB.
The problem instances are generated under Assumption \ref{assumption0}
with different distributions $\mathcal{P}$. 
For each  $\mathcal{P}$ and $\epsilon$, we also generate $50$ different problem instances. The 95\% confidence interval is obtained by performing $50$
independent runs and is shown as a semi-transparent region in the figure.

\cref{fig:3}  shows the experimental results 
of ballooning settings.
When $\mathcal{P}$ follows from a standard normal distribution, Double-UCB and Conservative-UCB exhibit similar performance. However, when $\mathcal{P}$ is a uniform distribution $U(0,1)$ or half-triangle distribution with distribution function as $1-(1-x)^2$, Double-UCB fails to achieve sublinear regret, while Conservative-UCB still performs well.  

\section{Conclusion}
In this paper, we have introduced a new graph feedback bandit model with similar arms. For this model, we proposed two different UCB-based algorithms (Double-UCB, Conservative-UCB) and provided corresponding regret upper bounds. We then extended these two algorithms to the ballooning setting, in which the application of Conservative-UCB is more extensive than Double-UCB. Double-UCB can only achieve sublinear regret when the mean distribution is Gaussian, while Conservative-UCB can achieve sublinear regret regardless of the mean distribution. Furthermore, we proposed U-Double-UCB and U-Conservative-UCB algorithms, which do not require knowledge of the graph information. Under the new graph feedback setting, we can obtain regret bounds based on the minimum dominating set without using the feedback graph to explicitly exploration.
More importantly, this similar structure helps us investigate the bandit problem with ballooning settings, which was difficult to explore in previous studies.

\appendix

\section{Facts and Lemmas}
\label{appendix_a}

\begin{lemma}
    \label{lemma1}
    Assume that $\{X_i\}_{i=1}^n$ are independent random variables. Let $\mu=\mathbb{E}[X_i]$ and $\bar{\mu}=\frac{1}{n}\sum_{i=1}^{n}X_i$. If $\{X_i\}_{i=1}^n$ are all bounded in $[0,1]$ or $\frac{1}{2}$-subGaussian,  for any $\delta \in (0,1)$, with probability at least $1-\delta^2$, 
    \[
      |\bar{\mu}-\mu|  \leq \sqrt{\frac{\log(\sqrt{2}/\delta)}{n}}.
    \]
\end{lemma}

\begin{lemma}\citep{allan1978domination}
    \label{claw-free}
    If $G$ is a claw-free graph, then $\gamma(G)=i(G)$.
\end{lemma}

\begin{lemma}
\label{ind_number_new}
\textcolor{blue}{
Assume that $\mu(a_t) \stackrel{i.i.d.}{\sim} \mathcal{P}$.
Let $G_T^{\mathcal{P}}$ denote the graph constructed by $\mu(a_1),...,\mu(a_T)$, and $\alpha(G_T^{\mathcal{P}})$ is the independence number of $G_T^{\mathcal{P}}$. 
Let $\mathcal{P}$ denote the Gaussian distribution $\mathcal{N}(0,1)$, we have
\begin{equation}
    \mathbb{P}( \alpha(G_T^{\mathcal{P}}) > \frac{2\sqrt{6\log T}}{\epsilon} +1   ) \leq \frac{2}{T^2}
\end{equation}
}

\end{lemma}
\begin{proof}
   For any instance $\mu(a_1),...,\mu(a_T) \stackrel{i.i.d.}{\sim} \mathcal{N}(0,1)$, denote $M_T,m_T,L_T$ as 
   \[ M_T:=\max_t \mu(a_t), m_T:=\min_t \mu(a_t), L_T:=M_T-m_T. \] 
   Since for any independent set $S$, it holds that $ |\mu(a_i)-\mu(a_j)| > \epsilon$ for all  $a_i, a_j \in S$ such that $a_i \ne a_j$. Then we can bound $|S|$ by 
   \[ 1 + \frac{L_T}{\epsilon}. \]
Next, we derive a high probability bound for $L_T$. By the union bound,
\[ \mathbb{P}(M_T > u) = \mathbb{P}(\bigcup_{i=1}^T\{\mu(a_i) > u\} ) \leq \sum_{i=1}^T \mathbb{P}(\mu(a_i)>u)=T(1-\Phi(u)). \]
Using Lemma \ref{Gaussian}, we have
\[ \mathbb{P}(M_T > u)  \leq \frac{T}{u+\sqrt{u^2+4}}e^{-\frac{u^2}{2}}.     \]

Now let $u = \sqrt{6\log T}$. If $T> 1$, we have
\[ \mathbb{P}(M_T > u) \leq \frac{1}{2\sqrt{6\log T}}\frac{1}{T^2} \leq \frac{1}{T^2}. \]
By symmetry, 
\[ \mathbb{P}(m_T < -u )=\mathbb{P}(\max_i(-\mu(a_i)) > u) \leq \frac{1}{T^2}. \]
Hence,
\begin{equation}
    \begin{aligned}
        \mathbb{P}(L_T \leq 2u) &\geq \mathbb{P}( \{M_T <u \} \bigcap \{ m_T >-u\})\\
        &= 1- \mathbb{P}(\{M_T >u\} \bigcup \{m_T <-u\})\\
        &\geq 1-\mathbb{P}(M_T>u)-\mathbb{P}(m_T<-u)\\
        &\geq 1-\frac{2}{T^2}.
    \end{aligned}
\end{equation}
Therefore,
\[ \mathbb{P}(|S| \geq 1 + \frac{2u}{\epsilon}) \leq \mathbb{P}(L_T \geq 2u) \leq \frac{2}{T^2}. \]
Since $S$ was chosen arbitrarily, the proof is complete.

\end{proof}

\begin{lemma}[Chernoff Bounds]
    \label{chernoff}
Suppose $X_1,...,X_n$ are independent Bernoulli random variable. 
Let $X$ denote their sum and let $\mu=\mathbb{E}[X]$ denote the sum's expected value. 
\[
    \mathbb{P}(X \geq (1+\delta)\mu) \leq e^{-\frac{\delta^2 \mu}{2+\delta}}, \delta \geq 0,
\]
\[
    \mathbb{P}(X \leq (1-\delta)\mu) \leq e^{-\frac{\delta^2 \mu}{2}}, 0 < \delta <1.
\]
\end{lemma}

\begin{lemma}\citep{abramowitz1988handbook}
    \label{Gaussian}
    For a Gaussian  random variable $X$ with mean $\mu$ and variance $\sigma^2$, for any $a >0$,
    \[
    \frac{1}{\sqrt{2\pi}}\frac{a}{1+a^2}e^{-\frac{a^2}{2}} \leq \mathbb{P}(X-\mu > a\sigma) \leq \frac{1}{a+\sqrt{a^2+4}}e^{-\frac{a^2}{2}}.    
    \]
\end{lemma}

\begin{lemma}
    \label{var}
Let $X_1,...,X_n \sim \mathcal{N}(0,1)$ and $Y:=\max_{1\leq i \leq n}X_i - \min_{1\leq i \leq n}X_i $. For any $n \geq 3$, we have 
\begin{equation}
    \mathbb{E}\big[ Y^2] \leq 8\log n + \frac{32}{\log 2}\frac{1}{\log 2n -\log(1+4\log\log2n)} .
\end{equation} 

\end{lemma}
\begin{proof}
    First, we bound the expectation of random variable $\max_{1\leq i \leq n}X_i $. 
    Since $X_i \sim \mathcal{N}(0,1)$, for any $\lambda>0$, we have
    \[ \mathbb{E}\big[\exp(\lambda X_i )\big] \leq \exp(\frac{\lambda^2}{2}). \]
    Since the  exponential function is convex, we have
    \[ \exp(\lambda \mathbb{E}\big[ \max_{1 \leq i \leq n}X_i \big]) \leq \mathbb{E}\big[\exp(\lambda  \max_{1 \leq i \leq n}X_i) \big]= \mathbb{E}\big[ \max_{1 \leq i \leq n}\exp(\lambda X_i) \big] \leq \mathbb{E}\big[ \sum_{i=1}^{n}\exp(\lambda X_i) \big] \leq n\exp(\frac{\lambda^2}{2}). \]
    Then,
    \[ \mathbb{E}\big[ \max_{1 \leq i \leq n}X_i \big] \leq \frac{\log n}{\lambda}+ \frac{\lambda }{2} .\]
    By setting $\lambda=\sqrt{2\log n}$, we have 
    \[\mathbb{E}[\max_{1\leq i \leq n}X_i] \leq  \sqrt{2\log n}. \]
   Thus, we have 
   \[ \mathbb{E}[Y] \leq 2\sqrt{2\log n}. \]
   Now we focus on bounding $\mathrm{Var}(Y)$.
   From Proposition 4.7 in \citep{boucheron2012concentration}, we know
   \[ \mathrm{Var}(\max_{1\leq i \leq n}X_i) \leq \frac{8}{\log 2} \cdot \frac{1}{\log 2n -\log(1+4\log\log2n)}, \ \ \forall n \geq 3. \]
   Then 
   \[
   \begin{aligned}  
    \mathrm{Var}(Y) &= \mathrm{Var}(\max_{1\leq i \leq n}X_i)+ \mathrm{Var}(\min_{1\leq i \leq n}X_i)-2\mathrm{Cov}(\max_{1\leq i \leq n}X_i,\min_{1\leq i \leq n}X_i)\\
    & \leq\mathrm{Var}(\max_{1\leq i \leq n}X_i)+ \mathrm{Var}(\min_{1\leq i \leq n}X_i)+2\sqrt{\mathrm{Var}(\max_{1\leq i \leq n}X_i)\mathrm{Var}(\min_{1\leq i \leq n}X_i)}\\
    & \leq \frac{32}{\log 2} \cdot \frac{1}{\log 2n -\log(1+4\log\log2n)}.
   \end{aligned} 
   \]
   Therefore,
    \[ \mathbb{E}[Y^2]=(\mathbb{E}[Y])^2 + \mathrm{Var}(Y) \leq 8\log n + \frac{32}{\log 2} \cdot \frac{1}{\log 2n -\log(1+4\log\log2n)}  . \]

\end{proof}

\section{Detailed Proofs}
\label{appendix_b}

\subsection{Proofs of Proposition \ref{pro1}}
\label{proof_pro1}

We first introduce the concept of $K_{1,3}$ and claw-free graph.  $K_{1,3}$ is a specific  bipartite graph consisting of two disjoint sets of vertices. One set containing a single "central" vertex and the other set containing three "peripheral" vertices. The central vertex is connected to each of the peripheral vertices by an edge, but there are no edges among the peripheral vertices themselves.
This type of graph is also called a ``claw graph" because its shape resembles a central vertex connected to three ``claws'' extending to the other three vertices.
A claw-free graph is a graph that does not have a claw as an induced subgraph or contains no induced subgraph isomorphic to $K_{1,3}$ .

(1) We first prove $\gamma(G)=i(G)$.
From \cref{claw-free}, we just need to prove $G$ is claw-free. 

Assuming $G$ has a claw, meaning that there exist nodes $a, b, c, d$, such that $a$ is connected to $b, c, d$, while $b, c, d$ are mutually unconnected. The mean values of $b$, $c$, and $d$ can be divided into two categories: greater than the mean value of $a$ and less than the mean value of $a$.  By the pigeonhole principle, at least two nodes among $(b,c,d)$ must belong to the same category. Without loss of generality, we  assume $b$ and $c$ are in the same category. Since the absolute difference between their means and the mean of $a$ is less than $\epsilon$, the absolute difference between the means of $b$ and $c$ is also less than $\epsilon$. Therefore, $b$ and $c$ are connected. This is a contradiction. Thus, $G$ is claw-free.


(2) We then prove $\alpha(G) \leq 2i(G) $.
Let $I^{*}$ be a maximum independent set and $I$ be a minimum independent dominating set. Then, we have $\alpha(G)=|I^{*}|$ and $ i(G)=|I| $. Since $G$ is claw-free,  each vertex of $I$ is adjacent (including the vertex itself in the neighborhood) to at most two vertices in $I^{*}$. Note that each vertex of $I^{*}$ is adjacent to at least one vertex of $I$.  So by a double counting argument, when counting once the vertices of $I^{*}$,   we can choose one adjacent vertex in $I$, and we will have counted at most twice the vertices of $I$.  Therefore, $|I^{*}| \leq 2|I|$.

\subsection{\textcolor{blue}{Proof of Lemma~\ref{kldecom}}}
    The proof is similar to Lemma 15.1 in \citep{lattimore2020bandit} and we use consistent notations with them.  Assume that $ \mathrm{KL}(P_j,P_j')< \infty$ for all $i \in [K]$. Let the random variables of arms and rewards be denoted as $(A_1,X_1,...,A_T,X_T)$ and the outcome are defined as $(a_1,x_1,...,a_T,x_T)$. 
    The difference is that selecting arm $a_t$  allows observing the rewards of multiple arms, i.e., $\{x_{t,j} \}_{j \in N_{a_t}}$. Hence, the observation sequence  is 
    \[ (a_1,\{x_{1,j} \}_{j \in N_{a_1}},...,a_t,\{x_{t,j} \}_{j \in N_{a_t}}...). \]
    Then the Radon-Nikodym derivative of $\mathbb{P}_{v,\pi}$ is 
    \[ p_{v,\pi}(a_1,\{x_{1,j} \}_{j \in N_{a_1}},...,a_T,\{x_{T,j} \}_{j \in N_{a_T}}) = \prod_{t=1}^T \pi_t( a_t | a_1,\{x_{1,j} \}_{j \in N_{a_1}},...,a_{t-1},\{x_{t-1,j} \}_{j \in N_{a_{t-1}}} ) \prod_{ j \in N_{a_t}} p_j(x_{t,j}). \]
    The Radon-Nikodym of $\mathbb{P}_{v',\pi}$ is the same as $\mathbb{P}_{v,\pi}$, except $p_j(x_{t,j})$ needs to be replaced with $p'_j(x_{t,j})$. Then 
    \[  \log \frac{\mathrm{d}\mathbb{P}_{v,\pi} }{\mathrm{d} \mathbb{P}_{v',\pi}}( a_1,\{x_{1,j} \}_{j \in N_{a_1}},...,a_T,\{x_{T,j} \}_{j \in N_{a_T}})  = \sum_{t=1}^T \sum_{j \in N_{a_t}} \log \frac{p_j(x_{t,j})}{p'_j(x_{t,j})}. \]
    Taking expectations we have
    \[ \mathbb{E}_v \left[ \log \frac{\mathrm{d}\mathbb{P}_{v,\pi} }{\mathrm{d} \mathbb{P}_{v',\pi}}( A_1,\{X_{1,j} \}_{j \in N_{A_1}},...,A_T,\{X_{T,j} \}_{j \in N_{A_T}})\right]  = \sum_{t=1}^T \sum_{j \in N_{A_t}}\mathbb{E}_v \left[ \log \frac{p_j(X_{t,j})}{p'_j(X_{t,j})}\right].  \]
    Note that
    \[ \mathbb{E}_v \left[ \log \frac{\mathrm{d}\mathbb{P}_{v,\pi} }{\mathrm{d} \mathbb{P}_{v',\pi}}\right] = \mathrm{KL}(\mathbb{P}_{v,\pi},\mathbb{P}_{v',\pi} ), \]
    \[ \sum_{j \in N_{A_t}}\mathbb{E}_v \left[ \log \frac{p_j(X_{t,j})}{p'_j(X_{t,j})}\right] = \mathbb{E}_v \left[\mathbb{E}_v \left[ \sum_{j \in N_{A_t}} \log \frac{p_j(X_{t,j})}{p'_j(X_{t,j})} \Bigg| A_t \right]  \right]=\sum_{j \in N_{A_t}} \mathbb{E}_v [ \mathrm{KL}(P_j,P'_j)], \]  
    where the last equality uses the fact that under $\mathbb{P}_{v,\pi}(\cdot | A_t)$, $  \log \frac{p_j(X_{t,j})}{p'_j(X_{t,j})} = \log \frac{\mathrm{d}P_j}{\mathrm{d}P'_j}$.  
    We have
    \begin{equation}
        \begin{aligned}
     \mathrm{KL}(\mathbb{P}_{v,\pi},\mathbb{P}_{v',\pi} ) &= \sum_{t=1}^T  \sum_{j \in N_{A_t}} \mathbb{E}_v [ \mathrm{KL}(P_j,P'_j)] \\
     &= \sum_{i=1}^K \mathbb{E}_v \left[ \sum_{t=1}^T \mathbbm{1}\{ A_t =i\} \sum_{j \in N_{A_t}}  \mathrm{KL}(P_j,P'_j) \right]\\
     &=\sum_{i=1}^K \sum_{j \in N_i} \mathbb{E}_v [T_i(T)] \mathrm{KL}(P_j,P'_j).
        \end{aligned}
    \end{equation}

\subsection{Proofs of Theorem \ref{bound_ducb}}
\label{proof_bound1}
Since the independent set $\mathcal{I}$ obtained after running Step 4-9 in \cref{alg:D_UCB} may vary across runs, 
we first fix one realization of $\mathcal{I}$ for analysis.
We then take the supremum over all possible $\mathcal{I}$, yielding an upper bound independent of $\mathcal{I}$.

Recall that $\mathcal{I}=\{\alpha_1,\alpha_2,...,\alpha^{*},...,\alpha_{|\mathcal{I}|}  \}$, where $\alpha^*$ denotes the arm whose neighborhood  includes the optimal arm, i.e., $i^* \in N_{\alpha^{*}}$. 
The regret can be divided into two parts: the first  part comes from  selecting  arms $i \notin N_{\alpha^{*}}$  and the second part comes from the  selection of arms $ i \in N_{\alpha^{*}}$:
\begin{equation}
    \label{temp0}
\begin{aligned}
\sum_{t=1}^{T}\sum_{i \in V} \Delta_i\mathbbm{1}\{i_t=i\}= \sum_{t=1}^{T}\sum_{i\notin N_{\alpha^{*}}}\Delta_i\mathbbm{1}\{i_t=i\} + \sum_{t=1}^{T}\sum_{i\in N_{\alpha^{*}}}\Delta_i\mathbbm{1}\{i_t=i\}.
\end{aligned}
\end{equation}

Due to the  randomness of the algorithm, $\alpha_i$'s are also  random variables. We have
\[ \mathbb{E}\left[ \sum_{t=1}^{T}\sum_{i \in V} \Delta_i\mathbbm{1}\{i_t=i\} \right] = \mathbb{E}\left[ \mathbb{E}\left[ \sum_{t=1}^{T}\sum_{i \in V} \Delta_i\mathbbm{1}\{i_t=i\} \bigg| \mathcal{I}   \right]   \right] .\]
In regret analysis, we should first condition on $\mathcal{I}$ to derive  the inner regret bound, and then take expectation over $\mathcal{I}$ to obtain the final regret bound. As shown below, the conditional regret bound is in fact independent of $\alpha_i$'s. For convenience, in the proofs of other theorems, we directly treat $\mathcal{I}$ as fixed and do not consider its randomness.


We first focus on the expected regret incurred by the first part. Let $\Delta_{\alpha_j}':=\mu(\alpha^{*})-\mu(\alpha_j)$ and $j_t \in \mathcal{I}$ denote the arm linked to the selected arm $i_t$ (Step 11 in \cref{alg:D_UCB}). We have 
\begin{align}
    \label{temp2}
    \sum_{t=1}^{T}\sum_{i\notin N_{\alpha^{*}}}\Delta_i\mathbbm{1}\{i_t=i\} &\leq \sum_{j=1,\alpha_j \neq \alpha^{*}}^{|\mathcal{I}|} \sum_{t=1}^{T} \sum_{i \in N_{\alpha_j}}\Delta_i\mathbbm{1}\{i_t=i,i \notin N_{\alpha^{*}}\} \nonumber \\
    &  \leq \sum_{j=1,\alpha_j \neq \alpha^{*}}^{|\mathcal{I}|}  (\Delta_{\alpha_j}'+2\epsilon)\sum_{t=1}^{T}\mathbbm{1}\{ j_t=\alpha_j,\alpha_j \neq \alpha^{*}\}.
\end{align}
The last inequality uses the following two facts:
\[ \Delta_{i}=\mu(i^{*})-\mu(i)=\mu(i^{*})-\mu(\alpha^{*})+\mu(\alpha^{*})-\mu(\alpha_j)+\mu(\alpha_j)-\mu(i)\leq \Delta_{\alpha_j}'+2\epsilon, \]
and
\[ \sum_{t=1}^{T} \sum_{i \in N_{\alpha_j}}\mathbbm{1}\{i_t=i,i \notin N_{\alpha^{*}}\} = \sum_{t=1}^{T}\mathbbm{1}\{j_t=\alpha_j,\alpha_j \neq \alpha^{*}\}.  \]
Recall that $O_t(i)$ denotes the number of observations of arm $i$ till time $t$. For any $\alpha_j \neq \alpha^{*}$, if $O_t(\alpha_j) > \frac{4\log(\sqrt{2}/\delta)}{(\Delta_{\alpha_j}')^2},$ we have, with probability at least $1-\delta^2$
\begin{equation}
    \label{conf1}
     \bar{\mu}_t(\alpha_j) + \sqrt{\frac{\log(\sqrt{2}/\delta)}{O_t(\alpha_j)}} \leq \mu(\alpha_j) + 2\sqrt{\frac{\log(\sqrt{2}/\delta)}{O_t(\alpha_j)}} < \mu(\alpha_j) + \Delta_{\alpha_j}'=\mu(\alpha^{*}),
\end{equation}
and
\begin{equation}
     \label{conf2}
    \bar{\mu}(\alpha^{*}) + \sqrt{\frac{\log(\sqrt{2}/\delta)}{O_t(\alpha^{*})}}  \geq \mu(\alpha^{*}). 
\end{equation}
Thus, for any $\alpha_j \neq \alpha^{*}$, 
\[ \mathbb{P}\left(j_t =\alpha_j,O_t(\alpha_j) >  \frac{4\log(\sqrt{2}/\delta)}{(\Delta_{\alpha_j}')^2}\right) \leq \delta^2. \]
Hence, 
\[ \sum_{t=1}^{T}\mathbb{P}( j_t=\alpha_j,\alpha_j \neq \alpha^{*} ) \leq  \frac{4\log(\sqrt{2}/\delta)}{(\Delta_{\alpha_j}')^2} + \delta^2T. \]

Plugging into \cref{temp0} and \cref{temp2}, we get
\begin{equation}
    \label{temp3}
    \begin{aligned}
        \sum_{t=1}^{T}\sum_{i\notin N_{\alpha^{*}}}\Delta_i\mathbb{P}(i_t=i) 
        &\leq \sum_{j=1,\alpha_j \neq \alpha^{*}}^{|\mathcal{I}|}  (\Delta_{\alpha_j}'+2\epsilon)\sum_{t=1}^{T}\mathbb{P}( j_t=\alpha_j,\alpha_j \neq \alpha^{*})\\
        &\leq \sum_{j=1,\alpha_j \neq \alpha^{*}}^{|\mathcal{I}|} \bigg( \frac{(\Delta_{\alpha_j}'+2\epsilon)4\log(\sqrt{2}/\delta)}{(\Delta_{\alpha_j}')^2} +\Delta_{\max}\delta^2T \bigg) \\
        &= \sum_{j=1,\alpha_j \neq \alpha^{*}}^{|\mathcal{I}|} \bigg(\frac{1}{\Delta_{\alpha_j}'}+\frac{2\epsilon}{(\Delta_{\alpha_j}')^2}\bigg)4\log(\sqrt{2}/\delta) + \sum_{j=1}^{|\mathcal{I}|}\Delta_{\max}\delta^2T\\
    \end{aligned}
\end{equation}
\textcolor{blue}{
We bound the summation of the first term ($\sum_{j=1,\alpha_j \neq \alpha^{*}}^{|\mathcal{I}|} \frac{1}{\Delta_{\alpha_j}'}$). 
We discuss three cases. 
\begin{itemize}
	\item \textbf{Case 1: }$\Delta_{\text{min}}<\epsilon$. We have
	\begin{align}
		 \sum_{j=1,\alpha_j \neq \alpha^{*}}^{|\mathcal{I}|} \frac{1}{\Delta_{\alpha_j}'} \leq \sum_{k=1}^{|\mathcal{I}|-1}\frac{1}{k\epsilon} \leq \frac{\log(\alpha(G))+1}{\epsilon} 
         = \frac{\log(\alpha(G))+1}{\max\{\epsilon,\Delta_{\min}\}}.
	\end{align}
	\item \textbf{Case 2: } $\Delta_{\text{min}}\geq \epsilon$, $\Delta_{\text{min}}$ and  $\epsilon$ are of same order. Since $\epsilon \leq \Delta_{\text{min}}$, the optimal arm $i^{*}$ is isolated in the feedback graph, i.e., it has no connections to other arms. Thus, we have $i^{*} \in \mathcal{I}$ and $i^{*}=\alpha^{*}$. 
    Therefore,
	\[ \sum_{j=1,\alpha_j \neq \alpha^{*}}^{|\mathcal{I}|} \frac{1}{\Delta_{\alpha_j}'} \leq \sum_{k=0}^{|\mathcal{I}|-2}\frac{1}{\Delta_{\text{min}}+ k\epsilon}.\] 
	Using the condition that  $\Delta_{\text{min}}\geq \epsilon$, we also have the following  bound:  
    \begin{equation}
        \sum_{j=1,\alpha_j \neq \alpha^{*}}^{|\mathcal{I}|} \frac{1}{\Delta_{\alpha_j}'}\leq \sum_{k=0}^{|\mathcal{I}|-2}\frac{1}{\epsilon+ k\epsilon} \leq \frac{\log(\alpha(G))+1}{\epsilon}.
    \end{equation}
    \item \textbf{Case 3: }$ \epsilon \ll \Delta_{\text{min}}$.  We have 
    \begin{align}
        \sum_{j=1,\alpha_j \neq \alpha^{*}}^{|\mathcal{I}|} \frac{1}{\Delta_{\alpha_j}'} & \leq \sum_{k=0}^{|\mathcal{I}|-2}\frac{1}{\Delta_{\text{min}}+ k\epsilon} \nonumber \\
        &\leq \frac{1}{\Delta_{\text{min}}} +  \frac{1}{\epsilon} \log\left(1 + \frac{(\alpha(G)-2)\epsilon}{\Delta_{\text{min}}}\right) \nonumber \\
        &\leq \frac{1+o(1)}{\Delta_{\text{min}}} \nonumber \\
        &\leq  \frac{\log(\alpha(G))+1}{\Delta_{\text{min}}} = \frac{\log(\alpha(G))+1}{\max\{\epsilon,\Delta_{\text{min}}\}}. 
    \end{align}
\end{itemize}
Then we bound the summation of the second term ($\sum_{j=1,\alpha_j \neq \alpha^{*}}^{|\mathcal{I}|}\frac{2\epsilon}{(\Delta_{\alpha_j}')^2}$).
\\
If $\Delta_{\min}<\epsilon$, we have
\[ \sum_{j=1,\alpha_j \neq \alpha^{*}}^{|\mathcal{I}|}\frac{2\epsilon}{(\Delta_{\alpha_j}')^2} \leq \sum_{j=1,\alpha_j \neq \alpha^{*}}^{|\mathcal{I}|} \frac{2\epsilon}{k^2\epsilon^2} \leq \frac{\pi^2}{3\epsilon}= \frac{\pi^2}{3\max\{\epsilon,\Delta_{\min}\}}.\]
If $\Delta_{\min} \geq \epsilon$, the optimal arm $i^{*}$ is isolated in the feedback graph. Then $i^{*} \in \mathcal{I}$ and $i^{*} = \alpha^{*}$. We have
\begin{equation}
    \sum_{j=1,\alpha_j \neq \alpha^{*}}^{|\mathcal{I}|}\frac{2\epsilon}{(\Delta_{\alpha_j}')^2} \leq \sum_{k=0}^{|\mathcal{I}|-2} \frac{2\epsilon}{(\Delta_{\min}+k\epsilon)^2} \leq \sum_{k=0}^{\infty} \frac{2\epsilon}{\epsilon^2(\Delta_{\min}/\epsilon+k)^2}=\frac{2}{\epsilon}\sum_{n=\Delta_{\min}/\epsilon}^{\infty}\frac{1}{n^2}.
\end{equation}
Define 
\[g(x):=x\sum_{n=x}^{\infty}\frac{1}{n^2} \quad (x \geq 1, x \in \mathbb{N}). \]
We have $g(1)=\frac{\pi^2}{6}$. For any $x \geq 1$,
\[ g(x+1)-g(x)=\sum_{n=x+1}^{\infty}\frac{1}{n^2} -\frac{1}{x} < \sum_{n=0}^{\infty}(\frac{1}{x+n}-\frac{1}{x+n+1}) -\frac{1}{x} = 0. \]
Hence, $g(x)$ is monotonically decreasing. We have
\[ \sum_{n=\Delta_{\min}/\epsilon}^{\infty}\frac{1}{n^2} = \frac{g(\Delta_{\min}/\epsilon)}{\Delta_{\min}/\epsilon} \leq \frac{g(1)}{\Delta_{\min}/\epsilon}=\frac{\pi^2\epsilon}{6\Delta_{\min}}. \]
Therefore,
\[ \sum_{j=1,\alpha_j \neq \alpha^{*}}^{|\mathcal{I}|}\frac{2\epsilon}{(\Delta_{\alpha_j}')^2} \leq \frac{\pi^2}{3\Delta_{\min}}=\frac{\pi^2}{3\max\{\epsilon,\Delta_{\min}\}}. \]
Define 
\begin{equation}
    C_1 := 
    \begin{cases}
        \frac{4(\log(2\gamma(G))+1)}{\max\{\epsilon,\Delta_{\text{min}}\}} + \frac{4\pi^2}{3\max\{\epsilon,\Delta_{\text{min}}\}}, & \Delta_{\min} < \epsilon \text{ or } \epsilon \ll \Delta_{\min}, \\
        \frac{4(\log(2\gamma(G))+1)}{\epsilon} + \frac{4\pi^2}{3\max\{\epsilon,\Delta_{\text{min}}\}}, &  \Delta_{\text{min}}\geq \epsilon, \Delta_{\text{min}} \text{ and }  \epsilon \text{ are of same order}.
    \end{cases}
\end{equation}
Therefore, we have
\begin{equation}
     \sum_{t=1}^{T}\sum_{i\notin N_{\alpha^{*}}}\Delta_i\mathbb{P}(i_t=i)  \leq C_1\log(\sqrt{2}/\delta) + \alpha(G)\Delta_{\max}\delta^2T.
\end{equation}
}

Now we focus on the second part in \cref{temp0}. 

For any $i \in N_{\alpha^{*}}$, we have $\Delta_i \leq 2\epsilon$. This means the gap between suboptimal and optimal arms is bounded. Therefore, 
this part can be seen as using UCB-N \citep{lykouris2020feedback} on the graph formed by $N_{\alpha^{*}}$.  We can directly use their results by adjusting some constant factors. Following Theorem 6 in \citep{lykouris2020feedback},  this part has a regret upper bound as
\begin{equation}
    \label{temp4}
16 \cdot \log(\sqrt{2}/\delta)\log_2T\max_{I \in \mathcal{I}(N_{\alpha^{*}})}\sum_{i \in I\backslash\{i^{*}\} }\frac{1}{\Delta_i} +2\epsilon TK\delta^2+2\epsilon.
\end{equation}

Combining \cref{temp3} and \cref{temp4}, by settings  $\delta=\frac{1}{T}$,  we have
\begin{equation}
    \begin{aligned}
R_T({\pi_{\text{Double}}}) &\leq C_1\log(\sqrt{2}/\delta)+ 32 \log(\sqrt{2}/\delta)\log(T)\max_{I \in \mathcal{I}(N_{\alpha^{*}})}\sum_{i \in I\backslash\{i^{*}\} }\frac{1}{\Delta_i} \\
&\quad + \alpha(G)\Delta_{\max}\delta^2T + 2\epsilon TK\delta^2+2\epsilon\\
&\leq C_1\log(\sqrt{2}T)+ 32\log^2(\sqrt{2}T)\max_{I \in \mathcal{I}(i^{*} )}\sum_{i \in I \backslash \{i^{*}\}}\frac{1}{\Delta_i} +\Delta_{\max}+4\epsilon,
    \end{aligned}
\end{equation}
where the last inequality uses the fact that $\mathcal{I}(N_{\alpha^{*}}) \subset \mathcal{I}(I^{*}) $. 

\subsection{Proofs of Corollary \ref{gapfree_ducb}}
\label{proof_corollary1}
Following the proofs of \cref{bound_ducb}, we have
\begin{equation}
    \begin{aligned}
        R_T(\pi_{\text{Double}}) &= \sum_{t=1}^{T}\sum_{i \in V} \Delta_i\mathbb{P}(i_t=i)\\
       & = \sum_{t=1}^{T}\sum_{i\notin N_{\alpha^{*}}}\Delta_i\mathbb{P}(i_t=i) + \sum_{t=1}^{T}\sum_{i\in N_{\alpha^{*}}}\Delta_i\mathbb{P}(i_t=i)\\
       &\leq C_1 \log(\sqrt{2}T) + \Delta_{\max}+ \sum_{t=1}^{T}\sum_{i\in N_{\alpha^{*}},\Delta_i > \Delta}\Delta_i\mathbb{P}(i_t=i)+ T\Delta \\
       &\leq C_1 \log(\sqrt{2}T) + \Delta_{\max} + \frac{64(\log(\sqrt{2}T))^2}{\Delta} + T\Delta + 4\epsilon \\
       &\leq C_1 \log(\sqrt{2}T) + 16\sqrt{T}\log(\sqrt{2}T) + \Delta_{\max}+4\epsilon,
    \end{aligned}
\end{equation}
where the last inequality holds since $\Delta=\sqrt{\frac{64(\log(\sqrt{2}T))^2}{T}}$.

\subsection{\textcolor{blue}{Proof of Theorem~\ref{bound_cucb}}}
Recall that $\mathcal{I}=\{\alpha_1,\alpha_2,...,\alpha^{*},...,\alpha_{|\mathcal{I}|}  \}$, where $\alpha^*$ denotes the arm whose neighborhood includes the optimal arm, i.e., $i^* \in N_{\alpha^{*}}$. 
Similar to the proof of Theorem~\ref{bound_ducb}, we divide the regret into two parts:
\begin{equation}
    \label{temp0_new}
\begin{aligned}
\sum_{t=1}^{T}\sum_{i \in V} \Delta_i\mathbbm{1}\{i_t=i\}= \sum_{t=1}^{T}\sum_{i\notin N_{\alpha^{*}}}\Delta_i\mathbbm{1}\{i_t=i\} + \sum_{t=1}^{T}\sum_{i\in N_{\alpha^{*}}}\Delta_i\mathbbm{1}\{i_t=i\}.
\end{aligned}
\end{equation}
The first part can be bounded by 
\[ C_1 \log(\sqrt{2}/\delta)+ \alpha(G)\Delta_{\max}T\delta^2. \]
For the second part, we divide the arms in $N_{\alpha^{*}}$ into two parts:
\[ E_1 = \{ i \in N_{\alpha^{*}}: \mu(i) < \mu(\alpha^{*})  \},  \]
\[ E_2 = \{ i \in N_{\alpha^{*}}: \mu(i) \geq \mu(\alpha^{*})  \}. \]
Hence, we have
\begin{align}
    \sum_{t=1}^{T}\sum_{i\in N_{\alpha^{*}}}\mathbbm{1}\{i_t=i\} &= \sum_{t=1}^T \sum_{i \in E_1} \mathbbm{1}\{i_t=i\} +  \sum_{t=1}^T \sum_{i \in E_2} \mathbbm{1}\{i_t=i\}\\
\end{align}
We first bound $  \sum_{t=1}^T \sum_{i \in E_1} \mathbbm{1}\{i_t=i\}$. Recall that $\Delta_{2\epsilon}= \min_{i,j \in G_{2\epsilon}}\{|\mu(i)-\mu(j)|\}$, where  $G_{2\epsilon}$ denote the subgraph with arms $\{ i \in V: \mu(i^{*})-\mu(i) < 2\epsilon \}$. Since $ N_{\alpha^{*}} \subset G_{2\epsilon}$, we have 
\begin{equation}
    \min_{i,j \in N_{\alpha^{*}}}\{|\mu(i)-\mu(j)|\} \geq \Delta_{2\epsilon}.
\end{equation}

Since $E_1$ is a complete graph, selecting any arm in $E_1$ will observe $\alpha^{*}$.
We have
\begin{equation}
    \sum_{t=1}^T \sum_{i \in E_1} \mathbbm{1}\{i_t=i\} \leq \frac{4\log(\sqrt{2}/\delta)}{(\Delta_{2\epsilon})^2} +  \sum_{t=1}^T \sum_{i \in E_1} \mathbbm{1}\{i_t=i,O_t(\alpha^{*}) > \frac{4\log(\sqrt{2}/\delta)}{(\Delta_{2\epsilon})^2}\}.
\end{equation}
If $O_t(\alpha^{*}) > \frac{4\log(\sqrt{2}/\delta)}{(\Delta_{2\epsilon})^2}$,
using the concentration inequality (Lemma~\ref{lemma1}), with probability at least $1-\delta^2$ we have
\begin{equation}
    \bar{\mu}_t(\alpha^{*})- \sqrt{\frac{\log(\sqrt{2}/\delta)}{O_t(\alpha^{*})}} >   \mu(\alpha^{*})-2\sqrt{\frac{\log(\sqrt{2}/\delta)}{O_t(\alpha^{*})}}\geq \mu(\alpha^{*})  - \Delta_{2\epsilon} \geq \mu(i),
\end{equation}
and 
\begin{equation}
    \bar{\mu}_t(i) - \sqrt{\frac{\log(\sqrt{2}/\delta)}{O_t(i)}} < \mu(i).
\end{equation}
Hence,
\[ \mathbb{P}\left(i_t=i,O_t(\alpha^{*}) > \frac{4\log(\sqrt{2}/\delta)}{(\Delta_{2\epsilon})^2}\right) \leq \delta^2. \]
Therefore,
\begin{align}
\label{eq:cucb}
     \sum_{t=1}^T \sum_{i \in E_1} \Delta_i \mathbb{P}(i_t=i) &\leq \frac{8\epsilon \log(\sqrt{2}/\delta)}{(\Delta_{2\epsilon})^2} +  \sum_{t=1}^T \sum_{i \in E_1} \Delta_i \mathbb{P}\left(i_t=i,O_t(\alpha^{*}) > \frac{4\log(\sqrt{2}/\delta)}{(\Delta_{2\epsilon})^2}\right) \nonumber \\
     &\leq  \frac{8\epsilon \log(\sqrt{2}/\delta)}{(\Delta_{2\epsilon})^2} + 2\epsilon T^2 \delta^2.
\end{align}
Using a similar method, 
\begin{align}
     \sum_{t=1}^T \sum_{i \in E_2} \Delta_i \mathbb{P}(i_t=i) &\leq \frac{8\epsilon \log(\sqrt{2}/\delta)}{(\Delta_{2\epsilon})^2} +  \sum_{t=1}^T \sum_{i \in E_2} \Delta_i \mathbb{P}\left(i_t=i,O_t(i^{*}) \geq \frac{4\log(\sqrt{2}/\delta)}{(\Delta_{2\epsilon})^2}\right)\\
     &\leq  \frac{8\epsilon \log(\sqrt{2}/\delta)}{(\Delta_{2\epsilon})^2} + 2\epsilon T^2 \delta^2.
\end{align}
Thus, the second part in \cref{temp0_new} can be bounded as
\[ \frac{16\epsilon \log(\sqrt{2}/\delta)}{(\Delta_{2\epsilon})^2} + 4\epsilon T^2 \delta^2 \]

Let $\delta=\frac{1}{T}$, 
the total regret upper bound is
\[  C_1 \log(\sqrt{2}T) + \frac{16\epsilon \log(\sqrt{2}T)}{(\Delta_{2\epsilon})^2} + 4\epsilon + \Delta_{\mathrm{max}}.  \]

\subsection{\textcolor{blue}{Proof of Lemma~\ref{lemma_layer}}}

\begin{unnumberedlemma}
 Let $V^{*}$ denote an arbitrary complete subgraph that contains the optimal arm. Let $\delta=\frac{1}{T}$. Then 
    \begin{equation}
         \sum_{t=1}^T\sum_{i \in V^{*} }\mathbb{P}(i_t=i)\Delta_i \leq  8\sqrt{T}\log(\sqrt{2}T) + \epsilon. 
    \end{equation}
\end{unnumberedlemma}
\begin{proof}
For any arm $i \in V^{*}$, $\Delta_i \in [0,\epsilon)$. Let $\Delta:=\frac{4\log(\sqrt{2}T)}{\sqrt{T}}$. 
We divide the arms into two parts:
\[ S_1 = \{i \in V^{*}:\Delta_i \leq \Delta \}, \]
\[  S_2 = \{i \in V^{*}:\Delta_i > \Delta \}. \]
The arms in $S_1$ can only incur regret up to $T\Delta$. We focus the regret with $S_2$.  
Let $V_{\phi}$ denote the arms in $S_2$ with gap $\Delta_i \in (2^{-\phi}\epsilon,2^{-\phi+1}\epsilon)$, where $\phi = 1,...,\frac{\log(\epsilon/\Delta)}{\log 2}$.   
Define the time 
\begin{equation}
    \tau_{\phi} := \min \{t: \sum_{s=1}^t \sum_{i \in V_{\phi}} \mathbbm{1}\{i_s = i \} > \frac{4\log(\sqrt{2}/\delta)}{2^{-2\phi}\epsilon^2} \}.
\end{equation}

Since $V^{*}$ is a complete graph containing the optimal arm, selecting any arm in $V_{\phi}$ reveals the optimal arm. 
If $t \geq \tau_{\phi}$, for any $i \in V_{\phi} $, we have 
\[ O_t(i^{*}) \geq \frac{4\log(\sqrt{2}/\delta)}{2^{-2\phi}\epsilon^2}  > \frac{4\log(\sqrt{2}/\delta)}{\Delta_i^2}.  \]

Using Lemma~\ref{lemma1},  the following inequality holds with probability at least $1-\delta^2$,
\begin{equation}
    \bar{\mu}_t(i^{*})- \sqrt{\frac{\log(\sqrt{2}/\delta)}{O_t(i^{*})}} >   \mu(i^{*})-2\sqrt{\frac{\log(\sqrt{2}/\delta)}{O_t(i^{*})}}\geq \mu(i^{*})  - \Delta(i)=\mu(i),
\end{equation}
and 
\begin{equation}
    \bar{\mu}_t(i) - \sqrt{\frac{\log(\sqrt{2}/\delta)}{O_t(i)}} < \mu(i).
\end{equation}
Hence, if $t\geq \tau_{\phi}$, $\mathbb{P}(i_t =i) \leq \delta^2$ for any $i \in V_{\phi}$.
We have

\begin{align}
    \sum_{t=1}^{T}\sum_{i \in V_{\phi} }\mathbb{P}(i_t=i)\Delta_i &= \sum_{t=1}^{\tau_{\phi}-1}\sum_{i \in V_{\phi} }\mathbb{P}(i_t=i)\Delta_i + \sum_{t=\tau_{\phi}}^{T}\sum_{i \in V_{\phi} }\mathbb{P}(i_t=i)\Delta_i \nonumber \\
    &\leq   \max_{i \in V_{\phi}} \frac{4\log(\sqrt{2}/\delta) \Delta_i}{2^{-2\phi}\epsilon^2} + \epsilon |V_{\phi}|T\delta^2 \nonumber \\
    &\leq \max_{i \in V_{\phi}} \frac{16\log(\sqrt{2}/\delta)}{\Delta_i} + \epsilon |V_{\phi}|T\delta^2 \nonumber \\
    &\leq \frac{16\log(\sqrt{2}/\delta)}{\Delta} + \epsilon |V_{\phi}|T\delta^2
\end{align}
where the penultimate  inequality uses the fact that $ \frac{\Delta_i}{2^{-2\phi}\epsilon^2} \leq \frac{4}{\Delta_i} $. 

Summing over all $\phi =1,...,\frac{\log(\epsilon/\Delta)}{\log 2}$, we have
\[ \sum_{t=1}^{T} \sum_{i \in S_2} \mathbb{P} ( i_t=i) \Delta_i \leq \frac{16\frac{\log(\epsilon/\Delta)}{\log 2}\log(\sqrt{2}/\delta)}{\Delta} + \epsilon T^2\delta^2. \]

Since $\epsilon$ is constant, for large $T$, we have $ T \geq  \frac{e^{\frac{\epsilon}{4}}}{\sqrt{2}}$. Substituting $\Delta = \frac{4\log(\sqrt{2}T)}{\sqrt{T}}$,  we have 
\[ \frac{\log(\epsilon/\Delta)}{\log 2} \leq \log T.  \]
Therefore, 
\begin{align}
     \sum_{t=1}^{T} \sum_{i \in V^{*}} \mathbb{P} ( i_t=i) \Delta_i &\leq T\Delta + \frac{16\log^2(\sqrt{2}T)}{\Delta} + \epsilon \nonumber \\
     &= 8\sqrt{T}\log(\sqrt{2}T) + \epsilon.
\end{align}

\end{proof}

\subsection{\textcolor{blue}{Proof of Theorem~\ref{bound_cucb_new}}}
The regret can be divided into two parts: 
\begin{equation}
\label{qh_temp1}
\begin{aligned}
\sum_{t=1}^{T}\sum_{i \in V} \Delta_i\mathbbm{1}\{i_t=i\}= \sum_{t=1}^{T}\sum_{i\notin N_{\alpha^{*}}}\Delta_i\mathbbm{1}\{i_t=i\} + \sum_{t=1}^{T}\sum_{i\in N_{\alpha^{*}} }\Delta_i\mathbbm{1}\{i_t=i\}.
\end{aligned}
\end{equation}
Similar to Double-UCB, the first part of the regret arises from selecting a neighborhood $N_{j_t}$ that does not contain the optimal arm, which can be bounded by 
\begin{equation}
    C_1 \log(\sqrt{2}T) + \Delta_{\mathrm{max}}.
\end{equation}

For the second part in \cref{qh_temp1}, we analyze the  regret according to the following two cases. 

\begin{enumerate}
    \item \textbf{Case 1: }$\Delta_{\min} < \epsilon$. If $N_{\alpha^{*}}$ is not a complete graph,  the algorithm  proceeds to Steps 12-14. 
    We have the following bound,
    \begin{align}
    \label{qh_temp3}
     \sum_{t=1}^{T}\sum_{i\in N_{\alpha^{*}}}\Delta_i\mathbbm{1}\{i_t=i\} &\leq \sum_{t=1}^{T} \sum_{i\in N_{\alpha^{*}}} \Delta_i\mathbbm{1}\{i_t=i,j_t' = \alpha_1^{*}\} + \sum_{t=1}^{T}\sum_{i\in  N_{\alpha^{*}} }\Delta_i\mathbbm{1}\{i_t=i,j_t' = \alpha_2^{*}\}  \nonumber \\
     &\leq 2\epsilon \sum_{t=1}^{T}\mathbbm{1}\{j_t' = \alpha_1^{*}\} + \sum_{t=1}^{T}\sum_{i\in N_{\alpha^{*}_2} \bigcap N_{\alpha^{*}} }\Delta_i\mathbbm{1}\{i_t=i\} 
    \end{align}
    The regret of the first part can be bounded as 
    \begin{align}
        2\epsilon \sum_{t=1}^T \mathbb{P}(j_t'=\alpha_1^{*}) &\leq \frac{8\log(\sqrt{2}T)}{\epsilon} +  2\epsilon \sum_{t=1}^T\mathbb{P}(j_t'=\alpha_1^{*},O_t(\alpha_1^{*}) > 4\log(\sqrt{2}T)/\epsilon^2)  \nonumber\\
        &\leq \frac{8\log(\sqrt{2}T)}{\epsilon} + 2\epsilon/T .
    \end{align}
    Using Lemma~\ref{lemma_layer}, 
     the second part in \cref{qh_temp3} can be bounded by
     \begin{equation}
         \sum_{t=1}^{T}\sum_{i\in N_{\alpha^{*}_2} \bigcap N_{\alpha^{*}} }\Delta_i \mathbb{P}(i_t=i) \leq  8\sqrt{T}\log(\sqrt{2}T) + \epsilon. 
     \end{equation}
     Substitute into \cref{qh_temp3}, we have
     \begin{equation}
         \label{new_temp}
         \sum_{t=1}^{T}\sum_{i\in N_{\alpha^{*}}}\Delta_i\mathbb{P}(i_t=i) \leq  8\sqrt{T}\log(\sqrt{2}T) + \frac{8\log(\sqrt{2}T)}{\epsilon} +3\epsilon. 
     \end{equation}

 If $N_{\alpha^{*}}$ forms a complete graph, the algorithm proceeds to Step 16. The analysis is identical to the case where $N_{\alpha^{*}}$ is not complete, except that the term $8\log(\sqrt{2}T)/\epsilon$ disappears. Hence, \cref{new_temp} still holds as a worst-case upper bound.

    \item \textbf{Case 2: } $\Delta_{\min} \geq \epsilon $. The optimal arm $i^{*}$ is isolated in the feedback graph, we have $i^{*} \in \mathcal{I}$ and $i^{*}=\alpha^{*}$. Hence, $N_{\alpha^{*}}=\{i^{*}\}$ 
    \begin{equation}
         \sum_{t=1}^{T}\sum_{i\in N_{\alpha^{*}}}\Delta_i\mathbb{P}(i_t=i) =0.
     \end{equation}

\end{enumerate}
Therefore, we have
\begin{align}
    R_T(\pi_{\text{Cons}}) \leq  8\sqrt{T}\log(\sqrt{2}T)+ 
  (C_1 +\frac{8}{\epsilon})\log(\sqrt{2}T) + 3\epsilon +  \Delta_{\mathrm{max}}
\end{align}

\subsection{Proofs of Theorem \ref{bound_ucbn}}
\label{proof_bound_ucb}
We  discuss $\Delta_i$ in intervals 
\[ [0,\epsilon),[\epsilon,2\epsilon),...,[k\epsilon,(k+1)\epsilon),...
\]
Let $V_k$ denote the arms with $\Delta_i \in [k\epsilon,(k+1)\epsilon)$. We have
\begin{align}
&\sum_{t=1}^T \sum_{i \in V_k}\Delta_i\mathbb{P}(i_t=i) \nonumber \\
&=
    \sum_{t=1}^T \sum_{i \in V_k}\Delta_i \mathbb{P}(i_t= i,O_t(i) \geq \frac{4\log(\sqrt{2}/\delta)}{k^2\epsilon^2}) +  \sum_{t=1}^T \sum_{i \in V_k}\Delta_i \mathbb{P}(i_t= i,O_t(i) < \frac{4\log(\sqrt{2}/\delta)}{k^2\epsilon^2}) \nonumber \\
    &\leq  \sum_{t=1}^T \sum_{i \in V_k}\Delta_i \mathbb{P}(i_t= i,O_t(i) \geq \frac{4\log(\sqrt{2}/\delta)}{k^2\epsilon^2}) + \frac{4(k+1)\log(\sqrt{2}/\delta)}{k^2\epsilon} \nonumber \\
    &\leq \Delta_{\max}T|V_k|\delta^2 + \frac{4\log(\sqrt{2}/\delta)}{k\epsilon} + \frac{4\log(\sqrt{2}/\delta)}{k^2\epsilon}, \label{temp_ucbn}
\end{align}
where the first inequality is because $V_k$ is a clique. Selecting any arm in $V_k$ will observe all arms.  The second inequality follows from the fact that $O_t(i) \geq \frac{4\log(\sqrt{2}/\delta)}{k^2\epsilon^2} \geq \frac{4\log(\sqrt{2}/\delta)}{(\Delta_i)^2}$. Then by the UCB selection rule, the event $\{i_t=i\}$ holds with probability at most $\delta^2$. 

Since each interval contributes at most one element to the independent set, the number of intervals is bounded by $\alpha(G)$.
Similar to the proof of Theorem~\ref{bound_ducb},  the third term in \cref{temp_ucbn} can be bounded by $\frac{4\pi^2}{3\max\{\epsilon,\Delta_{\min}\}}$. For the second term,  
we discuss three cases. 
\begin{itemize}
	\item \textbf{Case 1: }$\Delta_{\text{min}}<\epsilon$. We have
	\begin{align}
		 \sum_{k=1}^{\alpha(G)-1}\frac{1}{k\epsilon} \leq \frac{\log(\alpha(G))+1}{\epsilon} \nonumber 
         = \frac{\log(\alpha(G))+1}{\max\{\epsilon,\Delta_{\min}\}}.
	\end{align}
	\item \textbf{Case 2: } $\Delta_{\text{min}}\geq \epsilon$, $\Delta_{\text{min}}$ and  $\epsilon$ are of same order. The second term in \cref{temp_ucbn} can be bounded by 
	\[ \sum_{k=0}^{\alpha(G)-2}\frac{1}{\Delta_{\text{min}}+ k\epsilon} \leq \sum_{k=0}^{\alpha(G)-2}\frac{1}{\epsilon+ k\epsilon} \leq \frac{\log(\alpha(G))+1}{\epsilon}.\] 
	
    \item \textbf{Case 3: }$ \epsilon \ll \Delta_{\text{min}}$.  The second term in \cref{temp_ucbn} can be bounded by 
    \begin{align}
        \sum_{k=0}^{\alpha(G)-2}\frac{1}{\Delta_{\text{min}}+ k\epsilon} 
        &\leq \frac{1}{\Delta_{\text{min}}} +  \frac{1}{\epsilon} \log\left(1 + \frac{(\alpha(G)-2)\epsilon}{\Delta_{\text{min}}}\right) \nonumber \\
        &\leq \frac{1+o(1)}{\Delta_{\text{min}}} \nonumber \\
        &\leq  \frac{\log(\alpha(G))+1}{\Delta_{\text{min}}} = \frac{\log(\alpha(G))+1}{\max\{\epsilon,\Delta_{\text{min}}\}}. 
    \end{align}
\end{itemize}

Therefore, let $\delta=\frac{1}{T}$
\[ \sum_{t=1}^T \sum_{i \in V_k}\Delta_i\mathbb{P}(i_t=i) \leq C_1\log(\sqrt{2}T) + \Delta_{\max}. \]

Recall that $G_{\epsilon}$ denote the subgraph with arms $\{ i \in V: \mu(i^{*})-\mu(i) < \epsilon \}$. The set of all independent dominating sets of graph $G_{\epsilon}$ is denoted as $\mathcal{I}(G_{\epsilon})$.
The regret for $\Delta_i$ in $[0,\epsilon)$ can be bounded  as 
\[ 32 (\log(\sqrt{2}T))^2 \max_{I \in \mathcal{I}(G_{\epsilon} )}\sum_{i \in I \backslash \{i^{*}\}}\frac{1}{\Delta_i} + 2\epsilon .\] 

Similar to Corollary~\ref{gapfree_ducb}, we can obtain the gap-free upper bound as
\[ R_T(\pi_{\text{UCB-N}}) \leq  16\sqrt{T}\log(\sqrt{2}T)+ C_1\log(\sqrt{2}T) + \Delta_{\max} +2\epsilon.\]


\subsection{Proofs of Theorem \ref{bound_ducb_bl}}
\label{proof_bound2}
Recall that $v$ denotes the  bandit instance generated from $\mathcal{P}$ with length $T$. And 
$\mathcal{I}_t$ denotes the independent set at time $t$ and  $\alpha_t^{*}\in \mathcal{I}_t$ denotes the arm whose neighborhood includes the optimal arm $i_t^{*}$. 

Since the optimal arm may change over time, this leads to a time-varying gap. We denote the new gap as $\Delta_t(i)$. Therefore, the analysis method in \cref{bound_ducb} is no longer applicable here. 
The regret can also be divided into two parts:
\begin{equation}
    \label{temp5}
\begin{aligned}
\mathbb{E}\Bigg[\sum_{t=1}^{T}\sum_{i \in K(t)} \Delta_t(i)\mathbbm{1}\{i_t=i\} \Bigg]= \underbrace{\mathbb{E}\Bigg[ \sum_{t=1}^{T}\sum_{i\notin N_{\alpha_t^{*}}}\Delta_t(i)\mathbbm{1}\{i_t=i\}\Bigg]}_{(A)} + \underbrace{ \mathbb{E}\Bigg[ \sum_{t=1}^{T}\sum_{i\in N_{\alpha_t^{*}}}\Delta_t(i)\mathbbm{1}\{i_t=i\}\Bigg]}_{(B)}.
\end{aligned}
\end{equation}
We focus on $(A)$ first:
\begin{equation}
    \label{temp8}
    \begin{aligned}
    (A)&= \mathbb{E}_{v \sim \mathcal{P}}\Bigg[\mathbb{E}\Bigg[ \sum_{t=1}^{T}\sum_{i\notin N_{\alpha_t^{*}}}\Delta_t(i)\mathbbm{1}\{i_t=i\} \Big| v \Bigg] \Bigg]\\
    &=\mathbb{E}_{v \sim \mathcal{P}}\Bigg[\mathbb{E}\Bigg[ \sum_{t=1}^{T}\Delta_t(i)\mathbbm{1}\{i_t=i,j_t\neq \alpha_t^{*}\} \Big| v \Bigg] \Bigg]\\
    &\leq \mathbb{E}_{v \sim \mathcal{P}}\Bigg[\mathbb{E}\Bigg[ \sum_{t=1}^{T}\Delta_{\max}^T\mathbbm{1}\{j_t\neq \alpha_t^{*}\} \Big| v \Bigg] \Bigg].  \\
    \end{aligned}
\end{equation}
Given a fixed $v$, $\mathcal{I}_T$ is deterministic. 
Since the gap between optimal  and suboptimal arms may be varying over time, we define 
\[\Delta_{\alpha_j}'':=\min_{t \in [T]}\{ \mu(\alpha_t^{*})-\mu(\alpha_j): \alpha_j \in \mathcal{I}_T \text{ and } \mu(\alpha_t^{*})-\mu(\alpha_j)>0 \}.\] 
Then, $\Delta_{\alpha_j}'' \geq \epsilon$.

Following the proofs of \cref{bound_ducb},
for any $\alpha_j \in \mathcal{I}_T \neq \alpha_t^{*}$,  the probability of the algorithm selecting it  will be less than $\delta^2$ after it has been selected $\frac{4\log(\sqrt{2}/\delta)}{\epsilon^2}$ times. Therefore, the inner expectation of \cref{temp8} is bounded as
\begin{equation}
    |\mathcal{I}_T| \Delta_{\max}^T( \frac{4\log(\sqrt{2}/\delta)}{\epsilon^2} + T\delta^2 ).
\end{equation}
Plugging into $(A)$, we get
\begin{equation}
\label{temp6}
    \begin{aligned}
    (A) 
    &\leq \mathbb{E}_{v \sim \mathcal{P}}\Bigg[   |\mathcal{I}_T| \Delta_{\max}^T( \frac{4\log(\sqrt{2}/\delta)}{\epsilon^2} + T\delta^2 )\Bigg]    \\
    &= \mathbb{E}\big[\alpha(G_T^{\mathcal{P}}) \Delta_{\max}^T\big] ( \frac{4\log(\sqrt{2}/\delta)}{\epsilon^2} + T\delta^2 )\\  
    &\leq \sqrt{\mathbb{E}\big[(\alpha(G_T^{\mathcal{P}}))^2\big] \mathbb{E}\big[ (\Delta_{\max}^T)^2\big]} ( \frac{4\log(\sqrt{2}/\delta)}{\epsilon^2} + T\delta^2 ). 
    \end{aligned}
\end{equation}

Now we consider part $(B)$:
\begin{equation}
    \label{(B)}
    \begin{aligned}
        (B)
        &=\mathbb{E}_{v \sim \mathcal{P}}\Bigg[\mathbb{E}\Bigg[ \sum_{t=1}^{T}\sum_{i\in N_{\alpha_t^{*}}}\Delta_i\mathbbm{1}\{i_t=i\}\Big|v \Bigg]\Bigg]. \\
    \end{aligned}
\end{equation}
Recall that  $\mathcal{A}=\{a_t:t \in [T],a_t \in N_{\alpha_t^{*}}\}, M= \sum_{t=1}^{T}\mathbbm{1}\{ |\mu(a_t)-\mu(i_t^{*})| \leq 2\epsilon \}$, and  $|\mathcal{A}| \leq M $.
Given a fixed instance $v$, we divide the rounds into $J+1$ parts (where $t_0=1,t_{J+1}=T$):
\[[1,t_1],(t_1,t_2],...,(t_{J},T],\]  
where we use $(t_j,t_{j+1})$ denote the set $\{t_j+1,t_j+2,...,t_{j+1} \}$. 
\textcolor{blue}{This partition satisfies that for any $t \in (t_j,t_{j+1}]$, 
$\alpha_t^{*}$ is stationary.  We can denote $\alpha_t^{*}$ as $\alpha_j$. Note that the optimal arm $i_t^{*}$ may change within interval $(t_j,t_{j+1}]$. }

We first provide an intuitive explanation for the analysis in part (B).
Let's focus on the intervals $(t_j,t_{j+1}]$, the analysis for other intervals is similar. Assume that the optimal arm remains unchanged within the interval $(t_j,t_{j+1}]$.
The best case  is that all arms in $N_{\alpha_j}$ are arrived  at the beginning (this is impossible because our setting only allows one arm to arrive per round). In this case, the regret for this part is equivalent to the regret of using the UCB algorithm on the subgraph formed by $N_{\alpha_j}$ for $t_{j+1}-t_j$ rounds. The independence number of the subgraph formed by $ N_{\alpha_j}$ is at most $2$, which leads to a regret upper bound independent of the number of arms arriving. However, we are primarily concerned with the worst case. The worst case  is that the algorithm cannot benefit from the graph feedback at all. That is, the algorithm spends $O(\frac{\log(T)}{(\Delta_1)^2})$ rounds  distinguishing the optimal arm from the first arriving suboptimal arm $a_1$. After this process, the second suboptimal arm $a_2$ arrives, and again $O(\frac{\log(T)}{(\Delta_2)^2})$ rounds are spent distinguishing the optimal arm from this arm, and so on $\ldots$ Therefore, intuitively, the regret on the interval $(t_j,t_{j+1})$ can be upper bounded by $ O(\sum_{i \in N_{\alpha_j}} \frac{\log T}{\Delta_i})$.

\textcolor{blue}{
For any arm $i \in N_{\alpha_j}$, if arm $i$ has not arrived at time $t$, we set $\mathbbm{1}\{i_t =i\}=0$. We bound the regret in interval $(t_{j},t_{j+1}]$ by two cases:
\begin{enumerate}
    \item The optimal arm $i_t^{*}$ is stationary among $(t_{j},t_{j+1}]$. Then $ \Delta_i = \mu(i_t^{*})-\mu(i)$ is also  stationary.  Using standard UCB analysis techniques, we can obtain 
    \begin{equation*}
        \sum_{t=t_j}^{t_{j+1}} \Delta_i \mathbb{P}(i_t = i) \leq \frac{4\log(\sqrt{2}/\delta)}{\Delta_i} +2\epsilon (t_{j+1}-t_j)\delta^2.
    \end{equation*}
    \item The optimal arm $i_t^{*}$  changes within the interval $(t_{j},t_{j+1}]$.  Without loss of generality, we assume that the optimal arm changes only once at time $t_{j'}$; the case with multiple changes can be handled similarly. Let $\text{Num}_i$ denote the number of times arm $i$ is observed within time steps $(t_j,t_{j'})$.  Let $\Delta'_i$ denote the suboptimal gap within time steps $(t_{j'},t_{j+1})$. We further divide the analysis into three cases. 
    \begin{itemize}
        \item $\text{Num}_i \geq \frac{4\log(\sqrt{2}/\delta)}{(\Delta_i)^2} $. Based on the UCB rule, arm $i$ can be selected with probability at most $\delta^2$.  When $t > t_{j'}$, the suboptimal gap increases to $ \Delta'_i$. The probability of arm $i$ being selected is still at most $\delta^2$. Since $\Delta'_i < 2\epsilon$, the regret caused by arm $i$ is at most 
        \[ 2\epsilon (t_{j+1}-t_j) \delta^2 + \frac{4\log(\sqrt{2}/\delta)}{\Delta_i}.\]
        \item   $\frac{4\log(\sqrt{2}/\delta)}{(\Delta'_i)^2}  \leq   \text{Num}_i < \frac{4\log(\sqrt{2}/\delta)}{(\Delta_i)^2} $. Since $ \text{Num}_i < \frac{4\log(\sqrt{2}/\delta)}{(\Delta_i)^2} $, the regret incurred in interval $(t_j,t_{j'})$ can be bounded by $\frac{4\log(\sqrt{2}/\delta)}{\Delta_i}  $. Moreover, since $\frac{4\log(\sqrt{2}/\delta)}{(\Delta'_i)^2}  \leq   \text{Num}_i$, the probability of selecting the suboptimal arm $i$ in interval $(t_{j'},t_{j+1})$ is less than $\delta^2$, resulting in at most $2\epsilon (t_{j+1}-t_j) \delta^2$ regret. Therefore, the regret in interval $(t_j,t_{j+1})$ is upper bounded by 
         \[ 2\epsilon (t_{j+1}-t_j) \delta^2 + \frac{4\log(\sqrt{2}/\delta)}{\Delta_i}.\]
         \item  $\text{Num}_i < \frac{4\log(\sqrt{2}/\delta)}{(\Delta'_i)^2}$. If the number of observations of arm $i$ within $(t_j,t_{j+1})$ is less than $\frac{4\log(\sqrt{2}/\delta)}{(\Delta'_i)^2} $, the regret can be bounded by $\frac{4\log(\sqrt{2}/\delta)}{\Delta'_i}$. If there exists some time step $t_{j''}>t_{j'}$ such that the number of observations of arm $i$ within $(t_j,t_{j''})$ is greater than $\frac{4\log(\sqrt{2}/\delta)}{(\Delta'_i)^2} $,  the probability of selecting the suboptimal arm $i$ in interval $(t_{j''},t_{j+1})$ is less than $\delta^2$. Therefore, the regret can also be bounded by 
         \[ 2\epsilon (t_{j+1}-t_j) \delta^2 + \frac{4\log(\sqrt{2}/\delta)}{\Delta_i}.\]
    \end{itemize}
\end{enumerate}
In summary, we have
    \begin{equation}
    \label{eq1}
        \sum_{t=t_j}^{t_{j+1}} \Delta_i \mathbb{P}(i_t = i) \leq \frac{4\log(\sqrt{2}/\delta)}{\Delta_i} +2\epsilon (t_{j+1}-t_j)\delta^2.
    \end{equation}
}
\textcolor{blue}{
Therefore, 
\begin{align}
    \sum_{t=t_j}^{t_{j+1}}\sum_{i \in N_{\alpha_t^{*}}}\Delta_i \mathbb{P}( i_t=i) &= \sum_{i \in N_{\alpha_j}}\sum_{t=t_j}^{t_{j+1}}\Delta_i \mathbb{P}( i_t=i) \nonumber \\
    &\leq \sum_{i \in N_{\alpha_j},\Delta_i >\Delta}\sum_{t=t_j}^{t_{j+1}}\Delta_i \mathbb{P}( i_t=i) +(t_{j+1}-t_j)\Delta \nonumber \\
    &\leq \frac{4|N_{\alpha_j}|\log(\sqrt{2}/\delta)}{\Delta} +2\epsilon |N_{\alpha_j}|(t_{j+1}-t_j)\delta^2 + (t_{j+1}-t_j)\Delta \nonumber \\
    &\leq 4\sqrt{|N_{\alpha_j}|(t_{j+1}-t_j)\log(\sqrt{2}/\delta)} + +2\epsilon |N_{\alpha_j}|(t_{j+1}-t_j)\delta^2,
\end{align}
where the last inequality comes from $\Delta = 2\sqrt{\frac{|N_{\alpha_j}|\log(\sqrt{2}/\delta)}{t_{j+1}-t_j}}$.
Summing over all $j$ and setting $\delta = \frac{1}{T}$,
the inner expectation in \cref{(B)} can be bounded as}
\textcolor{blue}{
\begin{equation}
    \label{temp16}
    \begin{aligned}
     \sum_{t=1}^{T}\sum_{i \in N_{\alpha_t^{*}}}\Delta_i \mathbb{P}( i_t=i) &=
        \sum_{j=0}^{J}\sum_{t=t_j}^{t_{j+1}}\sum_{i \in N_{\alpha_t^{*}}}\Delta_i \mathbb{P}( i_t=i)\\
        &\leq \sum_{j=0}^J 4\sqrt{|N_{\alpha_j}|(t_{j+1}-t_j)\log(\sqrt{2}/\delta)}  +2\epsilon |N_{\alpha_j}|(t_{j+1}-t_j)\delta^2 \\
        &\leq 4 \sqrt{T\sum_{j=1}^J|N_{\alpha_j}| \log(\sqrt{2}T)   } + 2\epsilon \\
        &\leq 4 \sqrt{2TM \log(\sqrt{2}T)   } + 2\epsilon,
    \end{aligned}
\end{equation}
where the penultimate  inequality uses the Cauchy-Schwarz inequality, and the last inequality uses the fact that $\sum_{j}|N_{\alpha_j}| \leq 2|\mathcal{A}| \leq 2M$.
}
Substituting into \cref{(B)}, we get
\begin{equation}
    \label{temp7}
    \begin{aligned}
        (B)
        \leq \mathbb{E}_{v \sim \mathcal{P}}\Bigg[ 4\sqrt{2TM\log(\sqrt{2}T)}+ 2\epsilon  \Bigg] 
        \leq 4\sqrt{2T\mathbb{E}[M]\log(\sqrt{2}T)}+ 2\epsilon, 
    \end{aligned}
\end{equation}
where the inequality uses the fact that $\mathbb{E}[\sqrt{X}] \leq \sqrt{\mathbb{E}[X]}$.

From \cref{temp6} and \cref{temp7}, let $\delta=\frac{1}{T}$, we get the total regret
\[ R_T(\pi_{\text{Double-BL}}) \leq   \sqrt{ \mathbb{E}\big[(\alpha(G_T^{\mathcal{P}}))^2\big]\mathbb{E}\big[ (\Delta_{\max}^T)^2\big]} ( \frac{4\log(\sqrt{2}T)}{\epsilon^2} +1) + 4\sqrt{2T\mathbb{E}[M]\log(\sqrt{2}T)}+ 2\epsilon .\]

\subsection{Proofs of Corollary \ref{gapfree_ducb_bl}}
\label{proof_corollary2}
First, we bound  $\mathbb{E}[M]$ by a constant that is independent of the distribution $\mathcal{P}$.
Given $X_1,X_2,...,X_T$ as independent random variables from $\mathcal{P}$. We have
\begin{equation*}
    \begin{aligned}
         \mathbb{E}[M] &=\sum_{t=1}^{T} \mathbb{P}(|X_t -\max_{i=1,...,t}X_i| < 2\epsilon) \\
         &\leq \sum_{t=1}^T \big(\mathbb{P}(|X_t -\max_{i=1,...,t-1}X_i| < 2\epsilon) + \mathbb{P}(X_t =\max_{i=1,...,t}X_i) \big)\\
         &\leq \sum_{t=1}^T \mathbb{P}(|X_t -\max_{i=1,...,t-1}X_i| < 2\epsilon) + \log T +1.
    \end{aligned}
\end{equation*}

Denote $F(x):=\mathbb{P}(X<x),M_{t}:= \max_{i \leq t} X_i$. Then
\textcolor{blue}{
\[
  \mathbb{P}(|X_t -M_{t-1}| < 2\epsilon|M_{t-1}=x)=F(x+2\epsilon)-F(x-2\epsilon), 
\]
}
and 
\[
    \mathbb{P}(|X_t -M_{t-1}| < 2\epsilon)=(t-1)\int_{\mathcal{D}}(F(x))^{t-2}(F(x+2\epsilon)-F(x-2\epsilon)  )F'(x) dx,
\]
where $\mathcal{D}$ is the support set of $\mathcal{P}$.
Since 
\begin{equation}
    \label{re_temp1}
    \sum_{r=1}^{R}rx^{r-1}= \frac{d}{dx}\frac{1-x^{R+1}}{1-x}=\frac{1-(R+1)x^R+Rx^{R+1}}{(1-x)^2} ,
\end{equation}

we have 
\begin{equation}
    \label{temp10}
    \sum_{t=1}^T \mathbb{P}(|X_t -\max_{i=1,...,t-1}X_i| < 2\epsilon)= \int_{\mathcal{D}}\frac{1-T(F(x))^{T-1}+(T-1)(F(x))^{T}}{(1-F(x))^2} (F(x+2\epsilon)-F(x-2\epsilon)  )F'(x) dx.
\end{equation}

For Gaussian distribution $\mathcal{N}(0,1)$, we have $F(x)=\Phi(x),F'(x)=\phi(x).$ 
We first prove $\mathbb{E}[M]=O(\log(T)e^{2\epsilon \sqrt{2\log(T)}}).$

Denote the integrand function  in \cref{temp10} as $H(x)$.
Let $m= \sqrt{2\log(T)}+\epsilon$, then
\begin{equation}
    \sum_{t=1}^T \mathbb{P}(|X_t -\max_{i=1,...,t-1}X_i| < 2\epsilon)=\int_{-\infty}^{m}H(x)dx+\int_{m}^{+\infty}H(x)dx .
\end{equation}

First, we have

\[\forall x \in \mathbb{R}, \frac{1-T(F(x))^{T-1}+(T-1)(F(x))^{T}}{(1-F(x))^2}=\sum_{t=1}^{T-1}t(F(x))^{t-1} \leq \frac{T(T-1)}{2} \leq T^2,  \]
\[\Phi(x+2\epsilon)-\Phi(x-2\epsilon) \leq 4\epsilon \phi(x-2\epsilon),  \]
and
\begin{equation}
    \begin{aligned}
        (F(x+2\epsilon)-F(x-2\epsilon))F'(x) 
        \leq  4\epsilon \phi(x-2\epsilon) \phi(x)
        \leq   \frac{2\epsilon e^{-\epsilon^2}}{\pi} e^{-(x-\epsilon)^2}.
    \end{aligned}
\end{equation}
Then, we have
\begin{equation}
    \begin{aligned}
        \int_{m}^{+\infty}H(x)dx 
        &\leq T^2 \frac{2\epsilon e^{-\epsilon^2}}{\pi} \int_{m}^{\infty}e^{-(x-\epsilon)^2}dx\\
        &= T^2 \frac{2\epsilon e^{-\epsilon^2}}{\sqrt{\pi}} \Phi(\sqrt{2}(m-\epsilon))\\
        &\stackrel{(a)}{\leq} T^2 \frac{2\epsilon e^{-\epsilon^2}}{\sqrt{\pi}} \frac{1}{\sqrt{2}(m-\epsilon)+\sqrt{2(m-\epsilon)^2+4}} e^{-(m-\epsilon)^2} \\
        &\leq \frac{2\epsilon e^{-\epsilon^2}}{\sqrt{\pi}}, 
    \end{aligned}
\end{equation}
where $(a)$ uses \cref{Gaussian}.

Now we calculate the second term:
\begin{equation}
    \begin{aligned}
        \int_{-\infty}^{m}H(x)dx=\int_{-\infty}^{1}H(x)dx+\int_{1}^{m} H(x)dx \leq \frac{\Phi(1)}{(1-\Phi(1))^2} + \int_{1}^{m} H(x)dx.
    \end{aligned}
\end{equation}
We only need to bound the integral  within $(1, m)$.
\begin{equation}
    \begin{aligned}
        \int_{1}^{m} H(x)dx 
        &\leq \int_{1}^{m} \frac{(\Phi(x+2\epsilon)-\Phi(x-2\epsilon))\phi(x)}{(1-\Phi(x))^2} dx\\
        &\leq \frac{2\epsilon e^{-\epsilon^2}}{\pi} \int_{1}^{m} \frac{e^{-(x-\epsilon)^2}}{(1-\Phi(x))^2} dx\\
        &\stackrel{(b)}{\leq}  4\epsilon e^{-\epsilon^2} \int_{1}^{m} (\frac{1}{x}+x)^2 e^{x^2} e^{-(x-\epsilon)^2}dx\\
        &= 4\epsilon e^{-2\epsilon^2} \int_{1}^{m} (\frac{1}{x}+x)^2 e^{2x\epsilon}dx\\
        &\leq  4m^2e^{2m\epsilon}e^{-2\epsilon^2}\\
        &\leq 8\log(T) e^{2\epsilon \sqrt{2\log(T)}}, 
    \end{aligned}
\end{equation}
where $(b)$ uses \cref{Gaussian}.

Therefore, 
\[ \mathbb{E}[M] \leq \frac{2\epsilon e^{-\epsilon^2}}{\sqrt{\pi}}+ \frac{\Phi(1)}{(1-\Phi(1))^2}+  8\log(T)e^{2\epsilon \sqrt{2\log(T)}} + \log T +1. \]

Next, we bound $\mathbb{E}\big[(\alpha(G_T^{\mathcal{P}}))^2\big] $.

Using Lemma \ref{ind_number_new},
\begin{equation}
    \begin{aligned}
        \mathbb{E}\big[(\alpha(G_T^{\mathcal{P}}))^2\big] &= \sum_{k=1}^{T}k^2\mathbb{P}(\alpha(G_T^{\mathcal{P}})=k) \\
        &\leq \sum_{k=1}^{1+2\sqrt{6\log T}/\epsilon} (1+\frac{2\sqrt{6\log T}}{\epsilon})^2\mathbb{P}(\alpha(G_T^{\mathcal{P}})=k) + \sum_{k=2+2\sqrt{6\log T}/\epsilon}^{T}T^2\mathbb{P}(\alpha(G_T^{\mathcal{P}})=k)\\
        &\leq (1+\frac{2\sqrt{6\log T}}{\epsilon})^2 + 2\\
        &\leq 2 (1+\frac{2\sqrt{6\log T}}{\epsilon})^2.
    \end{aligned}
\end{equation}

Since
$\Delta_{\max}^T=\max_{i,j \in [T]}\mu(a_i)-\mu(a_j)= \max_{1 \leq i \leq T}\mu(a_i)- \min_{1 \leq i \leq T} \mu(a_i) $. Using Lemma~\ref{var}, we have
\[ \mathbb{E}\big[(\Delta_{\max}^T)^2 \big] \leq  8\log T + \frac{32}{\log 2}\frac{1}{\log 2T -\log(1+4\log\log2T)} .\]
Therefore, the regret is of order $O\Big(\log(T)\sqrt{Te^{2\epsilon\sqrt{2\log(T)}}}\Big).$

\subsection{Proof of Theorem \ref{lowerbound_ducb_bl}}
\label{proof_bound4}
Recall that $B= \mathbb{E}\Big[ \sum_{t=1}^{T} \mathbbm{1}\{ \frac{\epsilon}{2}< \mu(i_t^{*})-\mu(a_t) < \epsilon\}  \Big] .$

Following the analysis in the main text,
the lower bound is 
\[ R_T(\pi_{\text{Double-BL}})   \geq \mathbb{E}_{v \sim \mathcal{P}}\Bigg[ \frac{\epsilon}{2}\sum_{t=1}^{T}\mathbbm{1}\{ \frac{\epsilon}{2} < \mu(i_t^{*})-\mu(a_t) < \epsilon \}  \Bigg] = \frac{B\epsilon}{2}. \]

Note that, if $\mu(i_t^{*})=\mu(a_t)$, $\mathbbm{1}\{ \frac{\epsilon}{2}< \mu(i_t^{*})-\mu(a_t) < \epsilon\}=0.$ Hence, 
\[ B= \mathbb{E}\Big[ \sum_{t=1}^{T} \mathbbm{1}\{ \frac{\epsilon}{2}< \mu(i_{t-1}^{*})-\mu(a_t) < \epsilon\}  \Big] . \]

Following the calculate method of $\mathbb{E}\big[ M\big]$ (\cref{temp10}), we have
\begin{equation}
    \label{temp15}
    B= \int_{\mathcal{D}}\frac{1-T(F(x))^{T-1}+(T-1)(F(x))^{T}}{(1-F(x))^2} (F(x-\frac{\epsilon}{2})-F(x-\epsilon)  )F'(x) dx,
\end{equation}
where $\mathcal{D}$ is the support set of distribution $\mathcal{P}$.

(1) When $\mathcal{P}$ is $ \mathcal{N}(0,1)$. We have $F(x)=\Phi(x)$ and $F'(x)=\phi(x)$.

Since $\phi(t)$ is convex in $[1,+\infty)$, we have
\[ \phi(t) \geq \phi(\frac{a+b}{2}) + \phi'(\frac{a+b}{2})(t-\frac{a+b}{2}), \ \ \forall t \in [a,b], a \geq 1 . \]
When $x-\frac{\epsilon}{2} \geq 1$,
\[ \Phi(x-\frac{\epsilon}{2})-\Phi(x-\epsilon)=\int_{x-\epsilon}^{x-\frac{\epsilon}{2}}\phi(t)dt \geq \frac{\epsilon}{2}\phi(x-\frac{3\epsilon}{4}). \]
Hence,
\[ (\Phi(x-\frac{\epsilon}{2})-\Phi(x-\epsilon))\phi(x) \geq \frac{\epsilon}{2}\phi(x)\phi(x-\frac{3\epsilon}{4})= \frac{\epsilon}{4\pi}e^{-x^2+\frac{3}{4}x\epsilon-\frac{9}{32}\epsilon^2}  .\]
Substituting into \cref{temp15}, we have
\begin{equation}
    \label{temp13}
    \begin{aligned}
        B &\geq \frac{\epsilon}{4\pi}\int_{1+\frac{\epsilon}{2}}^{\sqrt{\log(T)}} \frac{1-T(\Phi(x))^{T-1}+(T-1)(\Phi(x))^{T}}{(1-\Phi(x))^2} e^{-x^2+\frac{3}{4}x\epsilon-\frac{9}{32}\epsilon^2} dx\\
    & \geq e^{-\frac{9}{32}\epsilon^2}\frac{\epsilon}{2\pi} \int_{1+\frac{\epsilon}{2}}^{\sqrt{\log(T)}} e^{\frac{3}{4}x\epsilon}(x^2+2)(1-T\Big(1-\frac{1}{\sqrt{2\pi}}\frac{x}{1+x^2}e^{-\frac{x^2}{2}}\Big)^{T-1}     ) dx .\\
    \end{aligned}
\end{equation}
The function $h(x)= 1-\frac{1}{\sqrt{2\pi}}\frac{x}{1+x^2}e^{-\frac{x^2}{2}} $ is an increasing function in the interval $(1, +\infty)$. We have
\begin{equation}
    \begin{aligned}
        (1-\frac{1}{\sqrt{2\pi}}\frac{x}{1+x^2}e^{-\frac{x^2}{2}}\Big)^T 
        &\leq (1-\frac{1}{\sqrt{2\pi}}\frac{\sqrt{\log(T)}}{1+\log(T)}\frac{1}{\sqrt{T}}\Big)^T \\
        &\leq e^{-\sqrt{\frac{T}{4\pi\log(T)}}} .
    \end{aligned}
\end{equation}
Observe that for large $T$ ( $T \geq e^{12}$), we have $e^{-\sqrt{\frac{T}{4\pi\log(T)}}} \leq \frac{1}{T^2}$.
Therefore, for any $T\geq e^{12}$,

\[ 
\begin{aligned}
B &\geq e^{-\frac{9}{32}\epsilon^2}\frac{\epsilon}{2\pi} \int_{1+\frac{\epsilon}{2}}^{\sqrt{\log(T)}} e^{\frac{3}{4}x\epsilon} (x^2+2) (1-\frac{T}{T^2}) dx \\
& \geq e^{-\frac{9}{32}\epsilon^2}\frac{\epsilon}{2\pi} \int_{1+\frac{\epsilon}{2}}^{\sqrt{\log(T)}} e^{\frac{3}{4}x\epsilon} (x^2+2) dx \\
& = e^{-\frac{9}{32}\epsilon^2}\frac{\epsilon}{2\pi}  e^{\frac{3\epsilon}{4}\sqrt{\log T}}(\frac{4\log T}{3\epsilon} -\frac{32\log T}{9\epsilon^2} + \frac{128}{27\epsilon^2} + \frac{8}{3\epsilon} ) + C(\epsilon),
\end{aligned}
\]
where $C(\epsilon)$ is some constant related to $\epsilon$ (unrelated to $T$).
Therefore, we have 
\[B=\Omega(\log (T) e^{\frac{3\epsilon}{4}\sqrt{\log T}} ).\]

(2) When $\mathcal{P}$ is $U(0,1)$. We have $F(x)=x\mathbbm{1}\{0<x<1\}.$ and $F'(x)=1$. Note that, from \cref{re_temp1}, $B$ can also be denoted as 
\begin{equation}
     B= \int_{0}^1\sum_{t=1}^{T-1}tF(x)^{t-1} (F(x-\frac{\epsilon}{2})-F(x-\epsilon)  )F'(x) dx.
\end{equation}
Then 
\begin{equation}
    \begin{aligned}
      B&=  \left( \int_{0}^{\frac{\epsilon}{2}} +\int_{\frac{\epsilon}{2}}^{\epsilon}+\int_{\epsilon}^1\right)\sum_{t=1}^{T-1}tF(x)^{t-1} (F(x-\frac{\epsilon}{2})-F(x-\epsilon)  )F'(x) dx\\
      &= 0 + \int_{\frac{\epsilon}{2}}^{\epsilon}\sum_{t=1}^{T-1}tx^{t-1}(x-\frac{\epsilon}{2}) \mathrm{d}x + \int_{\epsilon}^{1} \sum_{t=1}^{T-1}tx^{t-1}\frac{\epsilon}{2} \mathrm{d}x \\
      &\geq \int_{\epsilon}^{1} \sum_{t=1}^{T-1}tx^{t-1}\frac{\epsilon}{2} \mathrm{d}x\\
      &\geq \frac{\epsilon(1-\epsilon)}{2}(T-1).
    \end{aligned}
\end{equation}

(3) When $\mathcal{P}$ is half-triangle distribution with probability density function as 
$ f(x)=2(1-x) \mathbbm{1}\{0 < x < 1\}$. First, we calculate that $F(x) = (2x-x^2)\mathbbm{1}\{0<x<1\}$ and $F(x-\frac{\epsilon}{2}) - F(x-\epsilon)=-x\epsilon+\epsilon+\frac{3}{4}\epsilon^2$.
We have
\begin{equation}
    \begin{aligned}
     B &=  \int_{0}^1\sum_{t=1}^{T-1}tF(x)^{t-1} (F(x-\frac{\epsilon}{2})-F(x-\epsilon)  )F'(x) \mathrm{d}x\\
     &\geq \int_{\epsilon}^1 \sum_{t=1}^{T-1}tF(x)^{t-1}F'(x)  (-x\epsilon+\epsilon+\frac{3}{4}\epsilon^2 ) \mathrm{d}x\\
     &\geq \frac{3}{4}\epsilon^2 \int_{\epsilon}^1 \sum_{t=1}^{T-1}tF(x)^{t-1}F'(x) \mathrm{d}x \\
     &=\frac{3}{4}\epsilon^2 \sum_{t=1}^{T-1} (1-(2\epsilon-\epsilon^2)^t)  \\
     &\geq \frac{3\epsilon^2(1-\epsilon)^2}{4}(T-1).
    \end{aligned}
\end{equation}

\subsection{Proofs of  Theorem \ref{gap_cucb_bl}}
\label{proof_bound3}
Similar to the proofs of \cref{bound_ducb_bl}, the regret can also be divided into two parts:
\begin{equation}
    \label{temp9}
\begin{aligned}
\mathbb{E}\Bigg[\sum_{t=1}^{T}\sum_{i \in K(t)} \Delta_t(i)\mathbbm{1}\{i_t=i\} \Bigg]= \underbrace{\mathbb{E}\Bigg[ \sum_{t=1}^{T}\sum_{i\notin N_{\alpha_t^{*}}}\Delta_t(i)\mathbbm{1}\{i_t=i\}\Bigg]}_{(A)} + \underbrace{ \mathbb{E}\Bigg[ \sum_{t=1}^{T}\sum_{i\in N_{\alpha_t^{*}}}\Delta_i\mathbbm{1}\{i_t=i\}\Bigg]}_{(B)}.
\end{aligned}
\end{equation}
The analysis of part $(A)$ is similar to the analysis in the corresponding part of \cref{bound_ducb_bl}. 
\begin{equation}
    \begin{aligned}
    (A)&= \mathbb{E}\Bigg[ \sum_{t=1}^{T}\sum_{i\notin N_{\alpha_t^{*}}}\Delta_t(i)\mathbbm{1}\{i_t=i\} \Big| v \Bigg] \\
    &\leq \mathbb{E}\Bigg[ \sum_{t=1}^{T}\Delta_{\max}^T\mathbbm{1}\{j_t\neq \alpha_t^{*}\} \Big| v \Bigg]   \\
    &\leq \alpha(G_T)\Delta_{\max}^T( \frac{4\log(\sqrt{2}/\delta)}{\epsilon^2} + T\delta^2 ).
    \end{aligned}
\end{equation}

We now focus on part $(B)$.
Let 
\[  H= \sum_{t=1}^{T}\mathbbm{1}\{ \mu(a_t)=\mu(i_t^{*}) \} \]
denote the  number of times the optimal arms changes.

We  divide the rounds into $H+1$ parts: $(t_0=1,t_{H+1}=T)$
\[[1,t_1],(t_1,t_2],...,(t_{H},T].\]  
This partition satisfies that for any $t \in (t_j,t_{j+1}]$, 
$i_t^{*}$ is stationary.
The $\alpha_t^{*}$ is also stationary. 
We can denote $\alpha_t^{*}$ as $\alpha_j$. When $\alpha_t^{*}$ remains unchanged, $i_t^{*}$ may change. Therefore, there may be duplicate elements in sequence $\{\alpha_1,\alpha_2,...,\alpha_H\}$.

Let's focus on the intervals $(t_j,t_{j+1}]$, the analysis for other intervals is similar. 
The arms in $N_{\alpha_j}$ can be divided into two parts: 
\[ E_1^j=\{i \in N_{\alpha_j}: \mu(i) <\mu(\alpha_j) \},\] and 
\[ E_2^j=\{i \in N_{\alpha_j}: \mu(i)\geq \mu(\alpha_j)\}.\]
Based on how our algorithm constructs independent sets, it can be deduced that all arms in $ N_{\alpha_j} $ are connected to $ \alpha_t^{*} = \alpha_j$ and both $E_1$ and $E_2$ form a clique. 

Note that selecting any arm in $E_1^j$ leads to the observation of $\alpha_j$. From \cref{eq:cucb}, we can bound the regret for arms in $E_1^j$ by
\begin{equation*}
    \sum_{i \in E_1^j}\sum_{t=t_j}^{t_{j+1}} \Delta_i \mathbb{P}(i_t = i) \leq \frac{8\epsilon\log(\sqrt{2}/\delta)}{(\Delta^T_{\mathrm{min}})^2} +2\epsilon T(t_{j+1}-t_j) \delta^2.
\end{equation*}

Since selecting any arm in $E_2^j$ leads to the observation of $i_t^{*}$, we can also bound the regret for arms in $E_2^j$ by
\[ \sum_{i \in E_2^j}\sum_{t=t_j}^{t_{j+1}} \Delta_i \mathbb{P}(i_t = i) \leq \frac{8\epsilon\log(\sqrt{2}/\delta)}{(\Delta^T_{\mathrm{min}})^2} +2\epsilon T(t_{j+1}-t_j) \delta^2. \]

In summary, we have
\begin{equation}
    \sum_{i \in N_{\alpha_j}}\sum_{t=t_j}^{t_{j+1}} \Delta_i \mathbb{P}(i_t = i) \leq \frac{16\epsilon\log(\sqrt{2}/\delta)}{(\Delta^T_{\mathrm{min}})^2} +4\epsilon T(t_{j+1}-t_j) \delta^2.
\end{equation}

Summing over $j=1,...,H$ we have
\begin{align}
    \sum_{t=1}^T\sum_{i \in N_{\alpha_t^{*}}}\Delta_i \mathbb{P}(i_t=i)&=
    \sum_{j=1}^{H}\sum_{t=t_j}^{t_{j+1}}\sum_{i\in N_{\alpha_t^{*}}}\Delta_i\mathbb{P}(i_t=i)\nonumber\\
    & =    \sum_{j=1}^{H}\sum_{t=t_j}^{t_{j+1}}\sum_{i\in N_{\alpha_j}}\Delta_i\mathbb{P}(i_t=i) \nonumber \\
    &\leq \sum_{j=1}^H \left( \frac{16\epsilon\log(\sqrt{2}/\delta)}{(\Delta^T_{\mathrm{min}})^2} +4\epsilon T(t_{j+1}-t_j)\delta^2 \right) \nonumber \\
    &\leq \frac{16H\epsilon\log(\sqrt{2}/\delta)}{(\Delta^T_{\mathrm{min}})^2} +4\epsilon T^2 \delta^2.
\end{align}

Therefore, by setting $\delta=\frac{1}{T}$, we get the total regret
\[ R_T({\pi_{\text{Cons-BL}}}) \leq \alpha(G_T)\Delta_{\max}^T( \frac{4\log(\sqrt{2}T)}{\epsilon^2} + 1 )+  \frac{16H\epsilon\log(\sqrt{2}T)}{(\Delta_{\min}^T)^2} + 4\epsilon.\]

\subsection{\textcolor{blue}{Proof of Theorem~\ref{bound_cucb_bl}}}
Let 
$\mathcal{I}_t$ denote the independent set at time $t$ and  $\alpha_t^{*}\in \mathcal{I}_t$ denotes the arm whose neighborhood include the optimal arm $i_t^{*}$.  The regret can also be divided into two parts:
\begin{equation}
    \label{new_temp5}
\begin{aligned}
\mathbb{E}\Bigg[\sum_{t=1}^{T}\sum_{i \in K(t)} \Delta_t(i)\mathbbm{1}\{i_t=i\} \Bigg]= \underbrace{\mathbb{E}\Bigg[ \sum_{t=1}^{T}\sum_{i\notin N_{\alpha_t^{*}}}\Delta_t(i)\mathbbm{1}\{i_t=i\}\Bigg]}_{(A)} + \underbrace{ \mathbb{E}\Bigg[ \sum_{t=1}^{T}\sum_{i\in N_{\alpha_t^{*}}}\Delta_t(i)\mathbbm{1}\{i_t=i\}\Bigg]}_{(B)}.
\end{aligned}
\end{equation}
Similar to the proof of Theorem~\ref{bound_ducb_bl}, part (A) can be bounded as 
\begin{equation}
\label{new_temp6}
    (A) \leq \sqrt{\mathbb{E}\big[(\alpha(G_T^{\mathcal{P}}))^2\big] \mathbb{E}\big[ (\Delta_{\max}^T)^2\big]} ( \frac{4\log(\sqrt{2}/\delta)}{\epsilon^2} + T\delta^2 ).
\end{equation}

Now we consider part $(B)$:
\begin{equation}
    \label{(new_B)}
    \begin{aligned}
        (B)
        &=\mathbb{E}_{v \sim \mathcal{P}}\Bigg[\mathbb{E}\Bigg[ \sum_{t=1}^{T}\sum_{i\in N_{\alpha_t^{*}}}\Delta_i\mathbbm{1}\{i_t=i\}\Big|v \Bigg]\Bigg]. \\
    \end{aligned}
\end{equation}

Given a fixed instance $v$, we divide the rounds into $H+1$ parts (where $t_0=1,t_{H+1}=T$):
\[[1,t_1],(t_1,t_2],...,(t_{H},T].\]  
This partition ensures that for any $t \in (t_j, t_{j+1}]$,
$i_t^{*}$ is stationary. Thus, for any $t \in (t_j, t_{j+1}]$, $\alpha_t^{*}$ is also stationary, and we denote it by $\alpha_j$.
We first analyze the regret within the interval $(t_j, t_{j+1}]$, and then sum over all intervals to obtain the total regret.
The subgraph $N_{\alpha_t^{*}}$ may start as a complete graph and later become non-complete as new arms arrive. As we will show below, our analysis applies to both cases, and the regret upper bound is smaller when $N_{\alpha_t^{*}}$ is a complete graph than when it is not. Without loss of generality, we may assume that $N_{\alpha_t^{*}}$ remains non-complete throughout the entire interval $(t_j, t_{j+1}]$.

Under the ballooning setting, a result similar to Lemma~\ref{lemma_layer} still holds. The following lemma is an immediate consequence of Lemma~\ref{lemma_layer}.

\begin{lemma}
\label{new_lemma_layer}
 Let $V^{*}_j$ denote  an arbitrary complete subgraph that contains the optimal arm in $(t_j,t_{j+1}]$. Let $\delta=\frac{1}{T}$.  Then 
    \begin{equation}
        \sum_{t=t_j}^{t_{j+1}}\sum_{i \in V_j^{*} }\mathbb{P}(i_t=i)\Delta_i \leq  8\sqrt{t_{j+1}-t_j}\log(\sqrt{2}T) + \epsilon.
    \end{equation}
\end{lemma}


Since we assume that $N_{\alpha_j}$ is not a complete graph and the optimal arm $i_t^{}$ remains stationary in $(t_j, t_{j+1})$, we can denote $\mathcal{I}_{j_t}$ as ${\alpha_1^{*}, \alpha_2^{*}}$.
Both $\alpha_1^{*}$ and $\alpha_2^{*}$ depend on $j$, but we omit the explicit dependence for brevity.   If arm $i$ has not arrived at time $t$, we set $\mathbbm{1}\{i_t =i\}=0$ and $\mathbbm{1}\{j_t'=i\}=0$. 
Thus we have
\begin{align}
\label{qh_temp33}
 \sum_{t = t_j}^{t_{j+1}} \sum_{i \in N_{\alpha_j}} \Delta_i  \mathbbm{1}\{i_t = i\} 
 &=  \sum_{t = t_j}^{t_{j+1}}\sum_{i \in  N_{\alpha_j}} \Delta_i  \mathbbm{1}\{i_t = i,j_t' = \alpha_1^{*}\} 
   + \sum_{t = t_j}^{t_{j+1}} \sum_{i \in N_{\alpha_j}} \Delta_i \,\mathbbm{1}\{i_t = i, j_t'=\alpha_2^{*}\} \nonumber \\
&\leq  2\epsilon \sum_{t = t_j}^{t_{j+1}}  \mathbbm{1}\{j'_t = \alpha_1^{*}\} 
   + \sum_{t = t_j}^{t_{j+1}} \sum_{i \in N_{\alpha^{*}_2} \cap N_{\alpha_j}} \Delta_i \,\mathbbm{1}\{i_t = i\}. 
\end{align}

The regret of the first part can be bounded by 
\begin{align}
    2\epsilon \sum_{t=t_j}^{t_{j+1}} \mathbb{P}(j_t'=\alpha_1^{*}) &\leq \frac{8\log(\sqrt{2}/\delta)}{\epsilon} +  2\epsilon \sum_{t=t_j}^{t_{j+1}}\mathbb{P}(j_t'=\alpha_1^{*},O_t(\alpha_1^{*}) > 4\log(\sqrt{2}/\delta)/\epsilon^2)  \nonumber\\
    &\leq \frac{8\log(\sqrt{2}/\delta)}{\epsilon} + 2\epsilon (t_{j+1}-t_j)\delta^2.
\end{align}
Using Lemma~\ref{new_lemma_layer},  the second part in \cref{qh_temp33} can be bounded by
\begin{equation}
    \sum_{t=t_j}^{t_{j+1}}\sum_{i\in N_{\alpha^{*}_2} \bigcap N_{\alpha_j} }\Delta_i\mathbb{P}(i_t=i) \leq    8\sqrt{t_{j+1}-t_j}\log(\sqrt{2}T) + \epsilon.
\end{equation}

Substitute into \cref{qh_temp33}  we have
\begin{align}
    \sum_{t = t_j}^{t_{j+1}} \sum_{i \in N_{\alpha_j}} \Delta_i  \mathbb{P}(i_t = i)  &\leq  8\sqrt{t_{j+1}-t_j}\log(\sqrt{2}T)  + \frac{8\log(\sqrt{2}T)}{\epsilon} + 3\epsilon .
\end{align}
If $N_{\alpha_j}$ is a complete graph, the term $\frac{8\log(\sqrt{2}T)}{\epsilon}$ will vanish. 
Summing over all intervals,
\begin{align}
     \sum_{t =1}^{T} \sum_{i \in N_{\alpha_t^{*}}} \Delta_i  \mathbb{P}(i_t = i) &=  \sum_{j=0}^{H}\sum_{t = t_j}^{t_{j+1}} \sum_{i \in N_{\alpha_j}} \Delta_i  \mathbb{P}(i_t = i) \nonumber \\
     &\leq 8\sum_{j=0}^H \sqrt{t_{j+1}-t_j}\log(\sqrt{2}T) + \frac{8H\log(\sqrt{2}T)}{\epsilon} + 3H\epsilon \nonumber \\
     &\leq 8\sqrt{TH}\log(\sqrt{2}T)  + \frac{8H\log(\sqrt{2}T)}{\epsilon} + 3H\epsilon,
\end{align}
where the last inequality uses the Cauchy-Schwarz inequality. 
Thus,
\begin{equation}
    (B) \leq 8\sqrt{T\mathbb{E}[H]}\log(\sqrt{2}T) + \frac{8\mathbb{E}[H]\log(\sqrt{2}T)}{\epsilon} + 3\mathbb{E}[H]\epsilon.
\end{equation}
Based on Assumption~\ref{assumption0},
\[ \mathbb{E}[H] = \sum_{t=1}^T \mathbb{P}\big(\mu(a_t)=\mu(i_t^{*})\big) \leq \sum_{t=1}^T \frac{1}{t}\leq 2\log T. \]
Therefore, let $\delta =\frac{1}{T}$, we have
\begin{align}
    R_T(\pi_{\text{Cons-BL}}) &\leq (A) + (B) \nonumber \\
    &\leq \sqrt{\mathbb{E}\big[(\alpha(G_T^{\mathcal{P}}))^2\big] \mathbb{E}\big[ (\Delta_{\max}^T)^2\big]} ( \frac{4\log(\sqrt{2}T)}{\epsilon^2} + 1 ) + 8\sqrt{2T\log T}\log(\sqrt{2}T) \nonumber\\
    &\quad + \frac{16\log^2(\sqrt{2}T)}{\epsilon} + 6\epsilon \log T.
\end{align}

\subsection{Proof of Theorem \ref{u_ucb}}
\label{proof_bound_new}
(1) Under Assumption \ref{assumption0}, we first bound the regret of U-Double-UCB.

Without loss of generality, we assume that $T$ is an integer multiple of $\tau$.
Given a fixed instance $\mathcal{F}$, we divide the time horizon into $\frac{T}{\tau}$ parts:
\[ (1,\tau),(\tau+1,2\tau),...,(T-\tau+1,T). \]
Let $ \mathcal{T}_1$ denote the time steps used to find the independent set in each interval and $ \mathcal{T}_2 $ denote the intervals where the optimal arm changes, i.e.,   
\[\mathcal{T}_2 := \{[k\tau+1,(k+1)\tau]: k \in [\tau-1] \text{ and } \exists t \in [k\tau+1,(k+1)\tau], \mu(a_t)=\mu(i_t^{*}) \}.  \]
Then we divide the reget as follows:
\begin{equation}
    \label{temp14}
\begin{aligned}
 \mathbb{E}\Bigg[\sum_{t=1}^{T}\Delta_t(i_t) \Bigg]=  \mathbb{E}\Bigg[\sum_{t \in \mathcal{T}_1}\Delta_t(i_t) + \sum_{t \in \mathcal{T}_2}\Delta_t(i_t) +\sum_{t \notin \mathcal{T}_1 \bigcup \mathcal{T}_2}\Delta_t(i_t) \Bigg].
\end{aligned}
\end{equation}
The first part can be bounded as
\begin{equation}
    \begin{aligned}
    \mathbb{E}_{v \sim \mathcal{P}}\Bigg[ \mathbb{E}\Bigg[\sum_{t \in \mathcal{T}_1}\Delta_t(i_t)\Big| v \Bigg]\Bigg] 
    &\leq \frac{T}{\tau}\mathbb{E}\big[ \alpha(G_T^{\mathcal{P}}) \Delta_{\max}^T \big]\\
    &\leq \frac{T}{\tau}\sqrt{\mathbb{E}\big[(\alpha(G_T^{\mathcal{P}}))^2\big] \mathbb{E}\big[ (\Delta_{\max}^T)^2\big]}.
    \end{aligned}
\end{equation}
For the second part, we have 
\begin{equation}
    \begin{aligned}
        \mathbb{E}\Bigg[\sum_{t \in \mathcal{T}_2}\Delta_t(i_t)\Bigg] 
        \leq \tau \mathbb{E}\Bigg[\sum_{t=1}^{T}\mathbbm{1}\{ \mu(a_t)=\mu(i_t^{*}) \} \Delta_{\max}^T\Bigg] 
        \leq \tau \sqrt{\mathbb{E}\big[H^2\big] \mathbb{E}\big[ (\Delta_{\max}^T)^2\big]}.
    \end{aligned}
\end{equation}
Using Lemma \ref{chernoff}, we have $\mathbb{P}(H \geq 5\mathbb{E}[H]) \leq e^{-2\mathbb{E}[H]}$. We also have $ \log T \leq \mathbb{E}[H] \leq \log T+1 $. Thus,
\begin{equation}
    \begin{aligned}
    \mathbb{E}\big[H^2\big]&=\sum_{k=1}^{T}k^2\mathbb{P}(H=k)\\
    &\leq \sum_{k=1}^{5\mathbb{E}[H]}25(\mathbb{E}[H])^2\mathbb{P}(H=k) + \sum_{k=5\mathbb{E}[H]+1}^{T}T^2\mathbb{P}(H=k)\\
    &\leq 25(\mathbb{E}[H])^2+ 1.
    \end{aligned}
\end{equation}
Therefore, the second part can be bounded as 
\begin{equation}
 \mathbb{E}\Bigg[\sum_{t \in \mathcal{T}_2}\Delta_t(i_t)\Bigg]  \leq 11\tau \log T \sqrt{ \mathbb{E}\big[ (\Delta_{\max}^T)^2\big]} .
\end{equation}
Then we focus on the third part. Let $\mathcal{T}_3= [T]\setminus \mathcal{T}_1 \bigcup \mathcal{T}_2$. The regret can also be divided into two parts as \cref{temp5}: 
\[\underbrace{\mathbb{E}\Bigg[ \sum_{t \in \mathcal{T}_3}\sum_{i\notin N_{\alpha_t^{*}}}\Delta_t(i)\mathbbm{1}\{i_t=i\}\Bigg]}_{(A)} + \underbrace{ \mathbb{E}\Bigg[ \sum_{t \in \mathcal{T}_3}\sum_{i\in N_{\alpha_t^{*}}}\Delta_i\mathbbm{1}\{i_t=i\}\Bigg]}_{(B)}.\]
The  regret of the first part comes from selected $j_t \neq \alpha_t^{*}$ and can be bounded exactly the same as \cref{temp6}:
\[(A) \leq \sqrt{\mathbb{E}\big[(\alpha(G_T^{\mathcal{P}}))^2\big] \mathbb{E}\big[ (\Delta_{\max}^T)^2\big]} ( \frac{4\log(\sqrt{2}T)}{\epsilon^2} + 1). \]

If we drop $\mathcal{T}_1 $ and $\mathcal{T}_2$ (drop the blue and green parts in \cref{fig:0}),  we get a partitation of $\mathcal{T}_3$ as $ (t_1,t_2),...,(t_{\tau'},t_{\tau'+1})$, where $\tau' < \frac{T}{\tau}$. 
Since the algorithm finds an independent set at the beginning of each interval and the independent set does not change throughout the entire interval, we have $\alpha_t^{*}$ is also stationary.
Note that the optimal arm only changes in $\mathcal{T}_2$. So each interval in $\mathcal{T}_3$ satisfies  $i_t^{*}$ is stationary. By choosing $\delta=\frac{1}{T}$, we can use the same method as \cref{temp7} to bound the regret as
\[ (B) \leq 4\sqrt{2T\mathbb{E}[M]\log(\sqrt{2}T)}+ 2\epsilon. \]
Therefore, 
\begin{equation}
    \begin{aligned}
    R_T(\pi_{\text{U-Double-BL}}) &\leq \frac{T}{\tau} \sqrt{\mathbb{E}\big[(\alpha(G_T^{\mathcal{P}}))^2\big] \mathbb{E}\big[ (\Delta_{\max}^T)^2\big]}+ 11\tau \log T \sqrt{ \mathbb{E}\big[ (\Delta_{\max}^T)^2\big]} \\
    &\quad +\sqrt{\mathbb{E}\big[(\alpha(G_T^{\mathcal{P}}))^2\big] \mathbb{E}\big[ (\Delta_{\max}^T)^2\big]} ( \frac{4\log(\sqrt{2}T)}{\epsilon^2} + 1)\\
    &\quad + 4\sqrt{2T\mathbb{E}[M]\log(\sqrt{2}T)}+ 2\epsilon \\
    & = \frac{T}{\tau}\sqrt{\mathbb{E}\big[(\alpha(G_T^{\mathcal{P}}))^2\big] \mathbb{E}\big[ (\Delta_{\max}^T)^2\big]}+ 11\tau \log T \sqrt{ \mathbb{E}\big[ (\Delta_{\max}^T)^2\big]} + R_T(\pi_{\text{Double-BL}}).
    \end{aligned}
\end{equation}

(2) Without Assumption \ref{assumption0}, we bound the regret of U-Conservative-UCB. The method is similar to analyzing U-Double-UCB and we still divide the regret into three parts. We also have 
\begin{equation}
    \mathbb{E}\Bigg[\sum_{t \in \mathcal{T}_1}\Delta_t(i_t)\Bigg]
    \leq \frac{T}{\tau} \alpha(G_T) \Delta_{\max}^T ,
\end{equation}
\begin{equation}
    \mathbb{E}\Bigg[\sum_{t \in \mathcal{T}_2}\Delta_t(i_t)\Bigg] 
    \leq \tau \sum_{t=1}^{T}\mathbbm{1}\{ \mu(a_t)=\mu(i_t^{*}) \} \Delta_{\max}^T .
\end{equation}
Then 
\begin{equation}
    R_T(\pi_{\text{U-Cons-BL}}) \leq \frac{T}{\tau} \alpha(G_T) \Delta_{\max}^T + \tau H \Delta_{\max}^T+ R_T(\pi_{\text{Cons-BL}}).
\end{equation}

\bibliographystyle{plainnat}
\bibliography{main.bib}

\end{document}